\documentclass[twoside]{article}

\usepackage{xcolor}
\usepackage[preprint]{aistats2026}
\usepackage[round, authoryear]{natbib}
\usepackage{wrapfig,caption, graphicx, subfig}
\usepackage{comment}
\usepackage{subcaption}

\usepackage{amsthm, amsfonts, amssymb, enumitem, cancel}
\newtheorem{theorem}{Theorem}
\newtheorem{proposition}{Proposition}
\newtheorem{lemma}{Lemma}
\newtheorem{corollary}{Corollary}
\newtheorem{example}{Example}
\newtheorem{assumption}{Assumption}
\newtheorem{remark}{Remark}
\newtheorem{definition}{Definition}
\newcommand{\Rectflow}{\sf Rectflow}
\newcommand{\Rectify}{\sf Rectify}

\newcommand{\blue}{\color{blue}}
\newcommand{\bk}{\color{black}}
\newcommand{\vansh}[1]{\color{green}{Vansh: #1} \bk}

\usepackage{style}
\usepackage{mathrsfs}

\def\argmin{{\arg\min}}

\def\bP{\mathbf{P}}

\def\bbB{\mathbb{B}}

\def\bbE{\mathbb{E}}

\def\bbP{\mathbb{P}}

\def\bbR{\mathbb{R}}
\def\bbS{\mathbb{S}}

\def\bmu{{\boldsymbol\mu}}

\def\cC{\mathcal{C}}

\def\cN{\mathcal{N}}
\def\cO{\mathcal{O}}

\def\cS{\mathcal{S}}
\def\cT{\mathcal{T}}

\def\cZ{\mathcal{Z}}

\def\ccF{\mathscr{F}}

\usepackage{yfonts}

\def\Norm#1{\Vert#1\Vert}
\def\innerprod#1{\left\langle#1\right\rangle}

\def\pr{{\bbP}}

\def\tr{{\sf Tr}}

\def\ddt{{\frac{\rm d}{\rm dt}}}
\def\ddtau{{\frac{\rm d}{\rm d\tau}}}

\def\Unif{{\rm Unif}}

\def\evel{{\varepsilon_{\rm vl}}}
\def\sqevel{{\varepsilon^2_{\rm vl}}}

\usepackage{bbm}
\def\ind{{\mathbbm{1}}}
\mathchardef\mhyphen="2D

\newcommand{\bas}[1]{\begin{align*}#1\end{align*}}
\newcommand{\ba}[1]{\begin{align}#1\end{align}}

\def\1{\bm{1}}

\DeclareMathAlphabet{\mathsfit}{\encodingdefault}{\sfdefault}{m}{sl}
\SetMathAlphabet{\mathsfit}{bold}{\encodingdefault}{\sfdefault}{bx}{n}

\newcommand{\Law}{\text{Law}}

\hypersetup{
    colorlinks=true,
    linkcolor=red,
    citecolor=blue,
    urlcolor=blue
}    
\begin{document}

%

%
%
\runningauthor{Bansal, Roy, Rinaldo, Sarkar}

\twocolumn[

\aistatstitle{On the Convergence and Straightness of Rectified Flow}

\aistatsauthor{Vansh Bansal$^*$ \And
Saptarshi Roy$^*$ \And
Alessandro Rinaldo \And
Purnamrita Sarkar}

\aistatsaddress{Department of Statistics and Data Sciences, UT Austin\\
$^*$Equal contribution}
]
\begin{abstract}
Flow Matching has become a cornerstone of modern generative models like Stable Diffusion 3, largely due to the efficiency of its Rectified Flow (RF) variant. The success of RF hinges on iteratively learning straight trajectories, pushing generation towards fewer sampling steps. However, the theoretical link between path geometry and sampling efficiency has been under-explored. This paper fills this gap by introducing a novel \textit{Piecewise Straightness} parameter, $\gamma_{2,T}$. We establish the first Wasserstein convergence bound that explicitly links the discretization error of \textit{any} general flow-model to $\gamma_{2,T}$, proving that minimizing curvature is the key to achieving high-fidelity, one-step sampling. 

Building on this theory, we establish the first theoretical framework to analyze the straightness of RF. We begin by offering intuitive geometric arguments for simple cases before identifying sufficient conditions under which a single rectification step (1-RF) yields a perfectly straight or even a Monge optimal coupling. While whether these sufficient conditions are met depends on the problem geometry, they enable the first concrete proofs in this area. Critically, fulfilling these conditions makes the subsequent flow (2-RF) perfectly straight ($\gamma_{2,T}=0$). This eliminates the discretization error in our bound and makes flawless, single-step sampling possible. 
\end{abstract}
\section{{Introduction}}

In recent years, diffusion models have become the mainstream approach for image generation tasks \citep{image_ho2022cascaded, image_balaji2022ediff, image_rombach2022high}.
They leverage the score-based generative model (SGM) framework \citep{sohl2015deep,ho2020denoising}, where data is gradually perturbed according to a pre-defined diffusion process, and the process is reversed using Stochastic Differential Equations (SDEs) for sample generation. While powerful, the stochastic nature of the reverse SDEs makes sampling computationally expensive as it requires fine a discretization. Deterministic alternatives, such as probability-flow ordinary differential equations (ODEs) and DDIM \citep{song2020score, song2023consistency, song2020denoising, dpm-lu2022fast, dpmv3-zheng2023dpm} 
can be faster but often produce less faithful outputs with the coarse discretizations needed for rapid sampling. The primary reason for this inaccuracy is the discretization error introduced by numerical solvers when approximating highly curved, nonlinear trajectories.

The key to overcoming this limitation lies in learning generative paths that are inherently straight, as this would minimize the error from numerical solvers. This has motivated the development of Flow Matching (FM) \citep{lipman2022flow, albergo2023stochastic, albergo2022building}, a powerful framework that allows adoption of general probability paths to supervise ODE-based flow models. A prominent application of FM is Rectified Flow (RF) \citep{liu2023flow}, which is uniquely designed to learn perfectly straight trajectories from a simple noise distribution to the target data. Through an iterative ``reflow" procedure, RF progressively straightens the flow, thereby reducing the transport cost \citep{liu2022rectified, shaul2023kineticoptimalprobabilitypaths}. 

Recent empirical works \citep{liu2023instaflow, liu2023flow} have demonstrated RF's ability to generate high-quality images with just one or two discretization steps after 2-rectification procedures (2-RF). Moreover, \citep{lee2024improvingtrainingrectifiedflows} offered a heuristic explanation for why 2-RF should often produce straight flows and developed an improved training routine to achieve this directly, avoiding the potential performance degradation from excessive reflowing.
Despite these empirical breakthroughs, the underlying principle remains a heuristic. The formal connection between a flow's geometric straightness, its convergence rate, and the conditions under which RF can provably achieve such straightness remains elusive. In this paper, we establish this missing theoretical foundation by connecting a generative flow's geometry to its sampling efficiency.

First, we introduce a novel and robust Piecewise Straightness (PWS) parameter, $\gamma_{2,T}(\mathcal{Z})$, to quantify a flow's curvature. Using this metric, we establish the first 2-Wasserstein convergence bound that explicitly connects flow geometry to sampling efficiency. Our bound proves that the discretization error scales with $\mathcal{O}(\gamma_{2,T}/T^2)$, providing a rigorous theoretical justification for why straighter flows enable high-fidelity generation with fewer sampling steps. We also give ways to estimate $\gamma_{2, T}$ and show that for Gaussian-mixture target distributions it crucially depends on the maximum separation between the component means.

Building on this theory, we develop the first theoretical framework to analyze RF and identify a sufficient condition on the flow's Jacobian that makes the 2-RF flow perfectly straight $(\gamma_{2,T}=0)$. Furthermore, we identify a stronger commutativity condition under which this flow is not only straight but also the \textit{Monge optimal} transport map. We also provide the first concrete proofs that these conditions are met for key multi-dimensional problems, such as Gaussian-to-Gaussian and Gaussian-to-Gaussian-mixture flows, confirming they yield optimal and straight flows.
\paragraph{Notations.} For a matrix $A$, we define the matrix exponential as $\exp(A):= \sum_{k=}^\infty A^k/k!$. For two matrices $A$ and $B$, we define the Lie-bracket operation as $[A,B]:= AB - BA$.
\section{{Preliminaries}}\label{sec: background}
Flow-based generative models define a mapping between samples $X_0$ from the noise distribution $\rho_0$ (typically standard Gaussian) and the samples $X_1$ from the target distribution $\rho_1$ through an ODE:
\begin{equation}
    dZ_t = v(Z_t, t)\ dt, \quad  Z_0  = X_0 \sim \rho_0, \label{eq: ode-true}
\end{equation}
where $v: \bbR^d \times [0,1] \to \bbR^d$ is a time-varying drift (velocity) field that defines a probability path between $\rho_0$ and $\rho_1$. A typical recipe for constructing $v$ is to first consider a stochastic interpolation path $X_t = a_t X_0 + b_t X_1$ such that $a_0 = b_1 = 1$ and $a_1 = b_0 = 0$, and $(X_0, X_1) \sim \rho_0 \otimes \rho_1$. Then one can construct a drift field by optimizing the following:
\begin{equation}
    v := \mathop{\argmin}\limits_{f: \bbR^d \times [0,1] \to \bbR^d} \int_0^1\E{\norm{\dot X_t - f(X_t, t)}_2^2} dt,
    \label{eq: drift_obj}
\end{equation}
which has the solution $v(x, t):= v_t(x) = \bbE[\dot X_t \mid X_t = x]$. However, in practice, we parameterize $v$ by a neural network class $v_\theta$ to solve \eqref{eq: drift_obj} as the aforementioned conditional expectation is typically intractable. Let $\widehat{v}$ be an approximate solution of $\eqref{eq: drift_obj}$, and consider the ODE
\begin{equation}
    d \tilde Y_t = \widehat v_t(\tilde Y_t)\, dt , \quad  \tilde Y_0 \sim \rho_0 \label{eq: ode-approx}.
\end{equation}
As proposed in \cite{liu2023flow}, we apply the Euler discretization of the ODE to obtain our final sample estimates:
\begin{equation}
\label{eq: emp-ode-disc}
    \widehat Y_{t_i} = \widehat Y_{t_{i-1}} + \widehat v_{t_{i-1}}(\widehat Y_{t_{i-1}}) (t_i-t_{i-1}), \quad  i \in [T],
\end{equation}
where the ODE is discretized into $T$ uniformly spaced steps, with $t_i = i/T$. The final sample estimate $\widehat Y_1$ follows the distribution $\widehat\rho_1 := \Law(\widehat Y_1)$.

\paragraph{The Optimal Transport (OT) problem}:
The OT problem, first formulated by Monge (1781), seeks to find a deterministic map $\mathcal{T}$ that transports mass from an initial distribution $\rho_0$ to a target distribution $\rho_1$ with minimal cost. This is expressed as
{
\begin{equation} \label{eq: monge-problem}
\inf_{\mathcal{T}} \Ee{\rho_0}{c(\mathcal{T}(X_0) - X_0)} \ \  \text{s.t.} \ \ \text{Law}(\mathcal{T}(X_0)) = \rho_1
\end{equation}
}
An equivalent dynamic formulation recasts this as finding an optimal continuous-time path $\{X_t\}_{t \in [0,1]}$ between the distributions. For convex cost functions $c$, the optimal path is the straight-line \textit{displacement interpolant}, $X_t = tX_1 + (1-t)X_0$, which forms a geodesic in the Wasserstein space. This specific interpolant minimizes the kinetic energy of the flow, resulting in straight trajectories. RF, as we will discuss next, is uniquely designed to learn the drift function $v_t$ corresponding to this displacement interpolant, simplifying the complex OT problem into a series of tractable least-squares optimization tasks. For the remainder of this paper, we assume $c=\norm{\cdot}^2$ when referring to the OT problem \eqref{eq: monge-problem} unless stated otherwise.

\paragraph{Rectified flow (RF)}: RF \citep{liu2023flow} interpolates between the two distributions in straight line paths as $X_t = tX_1 + (1-t)X_0$ for $t \in [0,1]$, giving $v_t(x) = \E{X_1 - X_0\mid X_t=x}$. Therefore, the resulting sampling ODE \eqref{eq: ode-true} also approximates a geodesic path between $\rho_0$ and $\rho_1$ that allows faster sampling.
\paragraph{Straight couplings and flows}:
In context of RFs, the optimization step \eqref{eq: drift_obj} is known as a rectification step, and the solution path $\cZ: = \{Z_t\}_{t \in [0,1]}$ of ODE \eqref{eq: ode-true} is known as the rectified flow, denoted by $\cZ = {\Rectflow}(X_0, X_1)$. It is known that RF is marginal preserving \cite{liu2023flow}, i.e., $\Law(X_t) = \Law(Z_t)$. Hence, RF yields a new \textit{dependent} coupling $(Z_0, Z_1):= {\Rectify}(X_0, X_1)$ between $\rho_0$ and $\rho_1$. Moreover, a coupling $(X_0, X_1)$ is called a \textit{straight coupling} if $\bbE[X_1 - X_0 \mid tX_1 + (1-t)X_0] = X_1 - X_0$, a.s. with respect to the distribution of $(X_0, X_1)$ and for $t \sim \text{Unif}(0, 1)$, i.e., the drift $v$ along its linear interpolation is a constant function of time almost surely. Therefore, the flow generated by rectification of a straight coupling $(X_0, X_1)$ is defined as a \emph{straight flow}. \cite{liu2023flow} showed that successive rectifications eventually lead to a near straight flow that allows faster sampling in the generation stage. However,  quantitative bounds on the effects of straightness of a flow on its convergence rate still remain unknown.

\section{{Straightness of a flow}}

This section introduces our novel parameters for quantifying the straightness of the ODE flow in \eqref{eq: ode-true}. We will subsequently show that our more precise notions of straightness are critical for efficient numerical integration of flow-based models with fewer discretization steps. 

Intuitively, a perfectly straight path has no curvature. For a parametric curve, $\alpha(t):= (t, Z_t)$ for $t \in [0,1]$, this corresponds to zero magnitude of its acceleration, i.e., 
\begin{align} \label{eq: acceleration-zero}
    \norm{\ddot \alpha(t)}_2 = \norm{(0, \dot v_t(Z_t))}_2 = \norm{\dot v_t(Z_t)}_2 = 0
\end{align}
for all $t$. For the definition of a straight flow mentioned in Section \ref{sec: background}, assuming $\cZ = \{Z_t\}_{t \in [0,1]}$ is a twice differentiable curve defined through ODE \eqref{eq: ode-true} with $Z_0 \sim \rho_0 = N(0,I_d)$, we require \eqref{eq: acceleration-zero} to hold for almost every $t \sim \text{Unif}(0, 1).$ This motivates us to quantify the straightness of the entire flow $\cZ$ by measuring its magnitude of acceleration along the path by the following two quantities:
\begin{definition}
Let $\cZ = \{Z_t\}_{t \in [0,1]}$ be twice-differentiable flow following the ODE \eqref{eq: ode-true}. 

\begin{enumerate}
    \item The average straightness (AS) parameter of $\cZ$  is defined as
\[
\gamma_1(\cZ) := \int_0^1 \E{ \norm{\dot v_t(Z_t)}_2^2} \; dt.
\]

\item Let $0 = t_0 < t_1 <\ldots <t_T =1$ be a partition of $[0,1]$ into $T$ intervals of equal length. The piece-wise straightness (PWS) parameter of the flow $\cZ$ is defined as 
\[
\gamma_{2, T}(\cZ) := \max_{i \in [T]}\frac{1}{t_i - t_{i-1}}\int_{t_{i-1}}^{t_i} \E{ \norm{\dot v_t(Z_t)}_2^2} \; dt.
\]
\end{enumerate}

\end{definition}

The quantity $\gamma_1(\cZ)$ essentially captures the average curvature of the flow over time $t \in [0,1]$. On the other hand, $\gamma_{2,T}(\cZ)$ captures the maximum average curvature of $\cZ$ over the sub-intervals $[t_{i-1}, t_{i}]$ for all $i \in [T]$. Therefore, $\gamma_{2, T}(\cZ)$ captures a more stringent notion of straightness, and it's easy to show that $\gamma_{2,T}(\cZ)\ge \gamma_1(\cZ)$. See Lemma \ref{lemma: straightness comparison} in the appendix.

A small value for $\gamma_1(\cZ)$ or $\gamma_{2,T}(\cZ)$ indicates the flow is nearly straight. Below, we argue that these measures provide a more robust notion of straightness than the criterion 
$$S(\cZ) := \int_{0}^1 \E{\norm{Z_1 - Z_0 - v_t(Z_t)}_2^2} \; dt$$
proposed by \cite{liu2023flow}.  While, it can be shown that $\gamma_1(\cZ)=\gamma_{2,T}(\cZ)=0 \implies S(\cZ)=0$, the quantity $S(\cZ)$ can be misleadingly small even for highly curved paths. For instance, consider the wavy flow given by $Z_t = Z_0 + (t, 50 N^{-2}\sin(2 \pi N t))^\top$. A straightforward calculation shows that $S(\cZ) = O(N^{-2}) \to 0$ as $N \to \infty$, which incorrectly implies the path is becoming straighter. In reality, the flow just oscillates more rapidly. Our metrics, on the other hand, remain bounded away from zero ($\gamma_{2,T}(\cZ) \geq \gamma_1(\cZ) = 2 \times 10^4 \pi^4$), and correctly identify the increasing oscillations as a departure from straightness. 
\paragraph{Estimating the straightness parameters}: 
Since the true drift field is not typically tractable in practice, we give the following estimators for our straightness parameters which use the learnt field and do not require any extra computation:
\begin{align} 
    &\widehat{\gamma}_{2, T} = \frac{1}{(\Delta t)^2} \max_{i=1, \dots, T-1} \left( \frac{1}{N} \sum_{j=1}^{N} \left\| \widehat v_i^{(j)} - \widehat v_{i-1}^{(j)} \right\|_2^2 \right)\label{eq: gamma_2_estimate}\\
    & \text{and }\widehat{\gamma}_1 = \frac{1}{\Delta t} \frac{1}{N} \sum_{i=0}^{T-2} \sum_{j=1}^{N} \left\| \widehat v_{i+1}^{(j)} - \widehat v_i^{(j)} \right\|_2^2,\nonumber
\end{align}
where $\widehat v_i^{(j)} = \widehat v_{t_{i}}(\widehat Y^{(j)}_{t_{i}})$ denotes the estimated drift for each of the $j \in [N]$ samples flowing through the discretized ODE~\eqref{eq: emp-ode-disc}.

\section{{Wasserstein Convergence}}
\label{sec: convergence rate}
In this section, we analyze error rates for the final sampling distribution of a learnt flow model in terms of the 2-Wasserstein distance from the target distribution. To this end, we make the following assumptions on the drift function and its estimate that are necessary for establishing our error bounds:

\begin{assumption}
\label{assmp: main assumption}
Assume that 
    \begin{enumerate}[label=(\alph*)]
    \item \label{assmp: estimation error}
    There exists $\evel\ge 0$ such that 
    $
        \smash{\displaystyle \max_{0 \le i \le T}\Ee{X_{t_i} \sim \rho_{t_i}}{\norm{v_{t_i}(X_{t_i}) - \widehat v_{t_i}(X_{t_i})}_2^2} \leq \sqevel}$.
    \item \label{assmp: lipschitz cond} 
     There exists $\widehat{L}>0$ such that $\norm{\widehat{v}_t(x) - \widehat{v}_t(y)}_2 \le \widehat{L}\norm{x - y}_2$ almost surely.
     \end{enumerate}
\end{assumption}
Assumption \ref{assmp: main assumption}\ref{assmp: estimation error} requires $\widehat{v}_{t}$ to closely estimate the true drift $v_t$ across all $t \in \{t_i\}_{i \in [T]}$, a standard condition in the diffusion model literature \citep{gupta2024improvedsamplecomplexitybounds,li2023towards, li2024towards, chen2022sampling} essential for controlling error rates. We emphasize that it is notably weaker than assuming a uniform bound on the estimation error for all $t \in [0, 1].$ Assumption \ref{assmp: main assumption}\ref{assmp: lipschitz cond} imposes a Lipschitz condition (similar to \textit{one-sided} Lipschitzness) on $\widehat{v}_t$, also common for score functions in both score-based and flow-based generative models \citep{chen2022sampling, kwon2022scorebasedgenerativemodelingsecretly,li2023towards, pedrotti2024improvedconvergencescorebaseddiffusion, boffi2024flow}. Since $\widehat{v}$ is typically parameterized by neural networks with Lipschitz activations, this condition is both natural and practical. Moreover, in flow-based models, Lipschitz continuity of $v_t$ is crucial for the well-posedness of the ODE \eqref{eq: ode-true} \citep{liu2023flow, boffi2024flow}, further justifying the use of Lipschitz-constrained networks in training.
\subsection{Convergence under exact integration}
 We begin by analyzing the continuous-time flow~\eqref{eq: ode-approx} with estimated drift under exact integration. Our first main result bounds its Wasserstein distance to the target distribution in terms of the drift estimation error:

\begin{theorem}
\label{thm: pdata bound}
    Let $\rho_1$ be absolutely continuous with respect to the Lebesgue measure in $\bbR^d$. Define $\varepsilon^2(t) := \Ee{X_t\sim \rho_t}{ \norm{v_t(X_t) - \widehat{v}_t(X_t)}^2}$ for $t \in [0,1]$, and $\tilde \rho_1 := \Law(\tilde Y_1)$ for $\tilde Y_1$ obtained through the flow in~\eqref{eq: ode-approx}. Then, under Assumption \ref{assmp: main assumption}\ref{assmp: lipschitz cond}, we have that
    \[
    W_2^2(\tilde \rho_1, \rho_1) \le e^{1 + 2\widehat{L}} \int_{0}^{1} \varepsilon^2(t)\; dt \quad \text{almost surely.}
    \]
\end{theorem}
The bound presented in Theorem~\ref{thm: pdata bound} closely resembles those established in prior works \citep{kwon2022scorebasedgenerativemodelingsecretly, pedrotti2024improvedconvergencescorebaseddiffusion, boffi2024flow}, as it primarily depends on the  estimation error \(\varepsilon(t)\) for all \(t \in [0,1]\). Specifically, if there exists an \(\varepsilon > 0\) such that \(\sup_{t \in [0,1]} \varepsilon^2(t) \le \varepsilon^2\), then the squared 2-Wasserstein distance between the two distributions is of order \(\mathcal{O}(\varepsilon^2)\). 
The full proof is deferred to Appendix~\ref{sec: proof W2 ode rate}.

\begin{remark}
\label{remark: mollification}
    The absolute continuity requirement in Theorem \ref{thm: pdata bound} can be relaxed. If the density of $\rho_1$ does not exist, then one can convolve $X_1$ with an independent noise $W_\eta \sim N(0, \eta I_d)$ for a very small $\eta >0$, and consider the mollified distribution $\rho_1^\eta := \Law(X + W_\eta)$ as the target distribution. Note that $ \rho_1^\eta$ is absolutely continuous and satisfies $W_2^2(\rho_1^\eta , \rho_1) \le \eta^2 d$.
    Therefore, under the condition of Theorem \ref{thm: pdata bound}, and using triangle inequality we have 
    $
        W_2^2(\tilde \rho_1,  \rho_1) \lesssim \eta^2d +  e^{1 + 2\widehat{L}} \int_{0}^{1} \varepsilon^2(t)\; dt. 
    $
\end{remark}

\subsection{Convergence of the discretized flow}
In this section we show that the more accurate perception of the straightness of a flow captured by our AS and PWS parameters is crucial for analyzing discretization error. We present our main result below:

\begin{theorem}
    \label{thm: W2 ode disc}
    Let Assumption~\ref{assmp: main assumption} hold, and suppose the continuous-time flow \(\mathcal{Z} := \{Z_t\}_{t \in [0,1]}\) is defined by the ODE~\eqref{eq: ode-true} with a differentiable drift field \(v : \mathbb{R}^d \times [0,1] \to \mathbb{R}^d\). Then the sampling distribution $\widehat{\rho}_1$ obtained through the discretized ODE \eqref{eq: emp-ode-disc} satisfies the following almost sure inequality:
    \[
    W_2^2(\widehat \rho_1, \rho_1) \le \frac{27  e^{4 \widehat{L}}}{\max\{\widehat{L}^2,1\} } \left(\frac{\gamma_{2,T}(\cZ)}{T^2} + \sqevel \right),
    \]
\end{theorem}
The term involving the PWS parameters could be referred to as an error term due to discretization. More importantly, the above Wasserstein error bound shows that $T = \Omega\left(  \sqrt{\gamma_{2,T}(\cZ)/\epsilon}\right)$ is sufficient to achieve a discretization error of the order $O(\epsilon)$. Therefore, Theorem \ref{thm: W2 ode disc} indicates that if the flow is near-straight  (i.e., $\gamma_{2,T}(\cZ) \approx 0$), then accurate estimation of the data distribution can be achieved with a very few discretization steps. This finding indeed aligns with the prior empirical findings \citep{liu2023flow, lee2024improvingtrainingrectifiedflows, liu2023instaflow} related to the rectified flow. 
It is also consistent with the empirical behavior of Perflow \citep{yan2024perflow}, a methodology that has achieved improved performance by further straightening the rectified flow in each interval $[t_{i-1}, t_i]$ for all $i \in [T]$. The proof of the theorem can be found in Appendix \ref{sec: proof W2 ode disc}. 
Finally, the elementary inequality $\gamma_{2,T}(\cZ) \le T \gamma_{1}(\cZ)$ (see Appendix \ref{sec: proof of lemma straightness comparison}) also yields the following corollary.
\begin{corollary}
\label{cor: W2 error bound}
    Under the same conditions of Theorem \ref{thm: W2 ode disc}, we have the following almost sure inequality:
    \[
    W_2^2(\widehat{\rho}_1, \rho_1) \le \frac{27 e^{4 \widehat{L}}}{\max\{\widehat{L}^2,1\}} \left(\frac{\gamma_{1}(\cZ)}{T} + \sqevel \right).
    \]
\end{corollary}

\subsection{Convergence rates for Rectified Flow}
In this section, we focus our analysis on RF for an independent coupling between noise samples $X_0 \sim \rho_0 = N(0, I_d)$ and the target samples $X_1 \sim \rho_1$. We first obtain an expression for the acceleration $\dot v_t(x)$ and then obtains bounds on the Wasserstein error through $\gamma_{2, T}$ under certain assumptions.

By definition, we have $X_t = tX_1 + (1-t)X_0$ and $\rho_t = \Law(X_t).$ Using Tweedie's formula, we obtain:
\begin{equation}\label{eq: drift-score}
    v_t(x) := \E{X_1 - X_0 \mid X_t = x} =  \frac{x}{t} + \br{\frac{1-t}{t}} s_t(x),
\end{equation}
where $s_t(x)$ is the score, i.e. the gradient of perturbed log density w.r.t $x$ for $t\in(0, 1]$ and $v_0(x) = \mathbb{E}[X_1]-x$; see Lemma \ref{lemma: drift-score}.
\begin{figure}
    \centering
    \subfloat[\centering ]{{\includegraphics[width=0.23\textwidth]{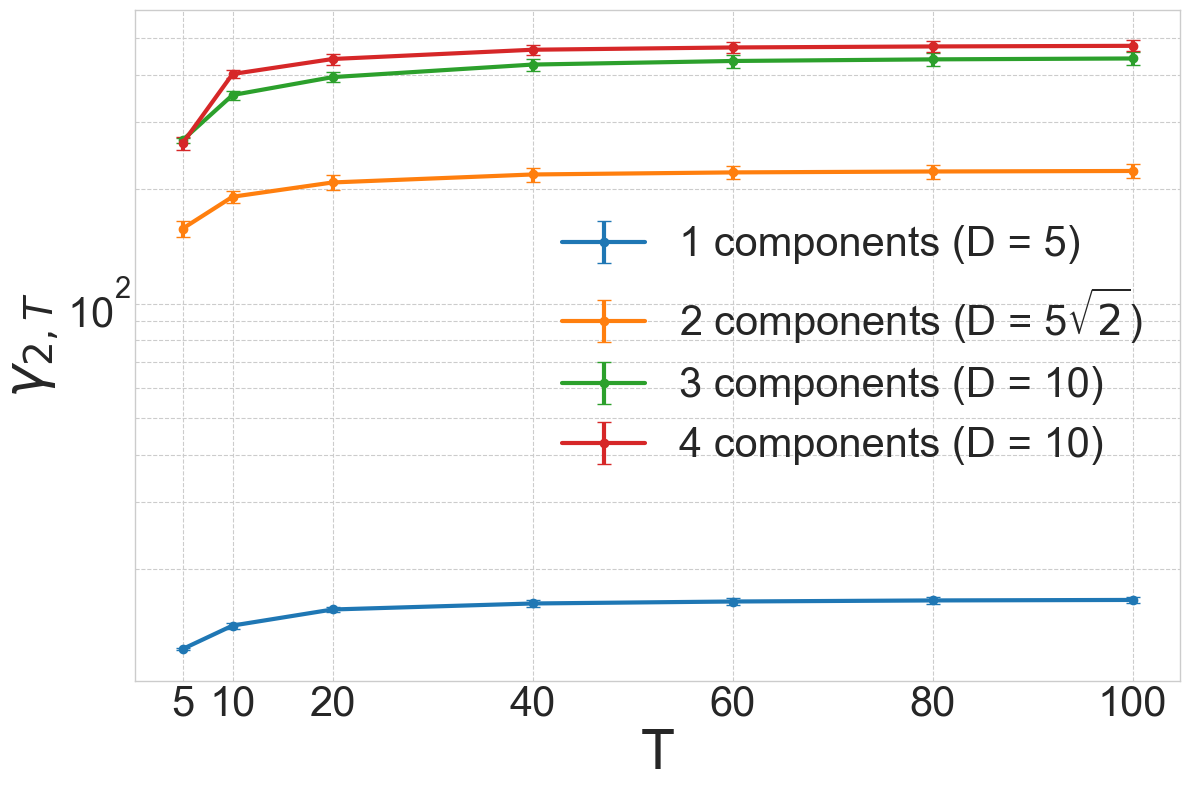} }}%
    \hfill
    \subfloat[\centering  ]{{\includegraphics[width=0.23\textwidth]{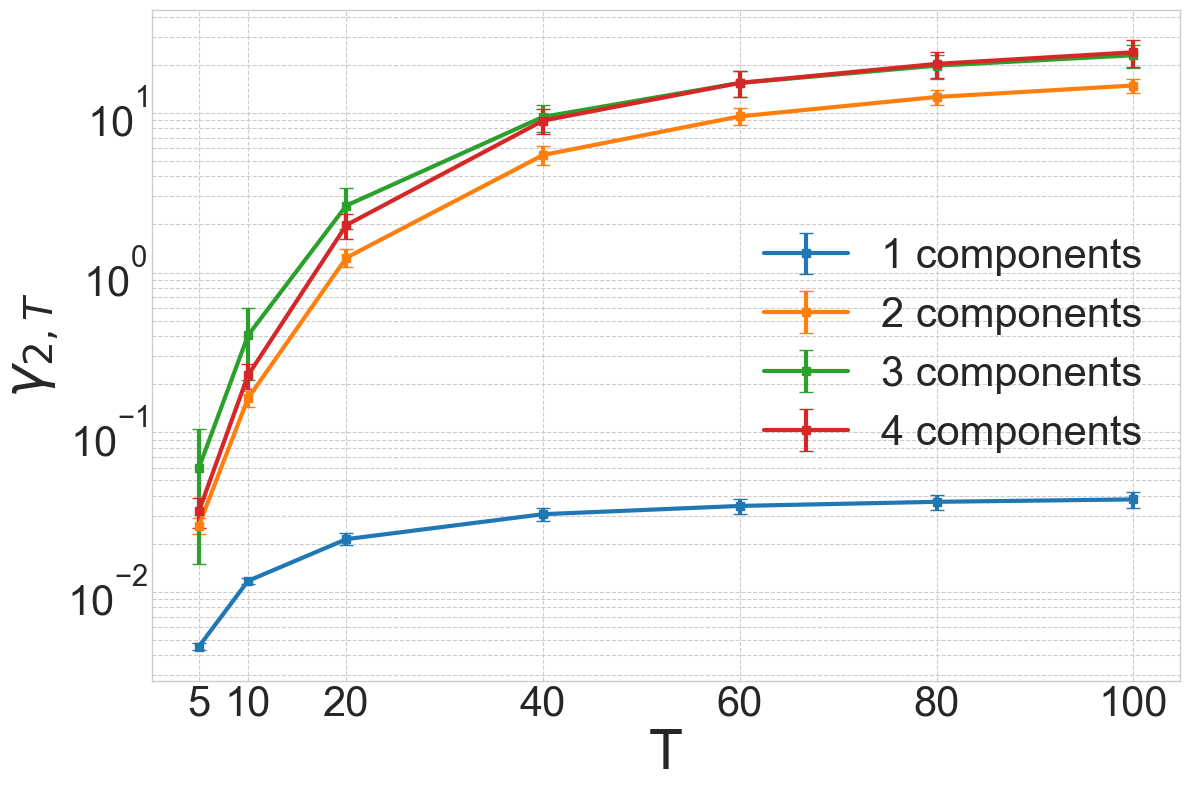} }}%
    \hfill
    \caption{For mixtures-of-Gaussians target distributions with varying components $K$ and maximum mean separation $D$ (see details in Appendix \ref{sec: exp mixture of Gaussian}), the figure shows (a) $\hat \gamma_{2, T}(\cZ)$ vs $T$ for the 1-RF flow (b) $\hat \gamma_{2, T}(\cZ)$ vs $T$ for the 2-RF flow}
    \label{fig: gmm-gamma}
\end{figure}
\begin{lemma} \label{lem: acceleration-rf}
    Let $(X_0, X_1) \sim \rho_0\otimes \rho_1$ and $v$ be the RF field defined in \eqref{eq: drift-score}. Then we have that:
    \begin{equation}
    \dot{v}_{t}(x)=-\frac{s_t(x)}{t^{2}}-\frac{(1-t)^{2}}{t^{2}}\left(H_{t}(x)s_{t}(x)+\nabla_{x}\tr(H_{t}(x))\right)
\end{equation}
where $H_t(x) := \nabla_xs_t(x)$ is the gradient of the score function $s_t(x)$.
\end{lemma}
The result of Lemma \ref{lem: acceleration-rf} is critical, as it shows that the acceleration $\dot v_t$-- and therefore the parameters $\gamma_{2, T}$ and $\gamma_1$-- is determined by the score $s_t$, Hessian $H_t$ and the gradient of the Hessian's trace. While $s_t$ is known to be a sub-Gaussian for $0\leq t<1$ \cite[Lemma F.2]{gupta2024improvedsamplecomplexitybounds}, bounding the other two higher order terms requires the knowledge of the geometry of the target distribution $\rho_1$. As an example, the following lemma considers the expected squared acceleration norm for a Gaussian-mixture target with bounded mean separation.

\begin{lemma}
\label{lemma: bound on gamma_2}
    Let $(X_0, X_1) \sim N(0, I_d) \otimes \rho_1$ where $\rho_1 = \sum_{i=1}^{K} \pi_k N(\mu_k, \sigma^2 I_d)$, such that $\mu_i \in \bbR^d$ and $D^2 :=\max_{i, j} \norm{\mu_i - \mu_j}_2^2$. For $\cZ={\Rectflow}(X_0, X_1)$ obtained by integrating \eqref{eq: ode-true} with drift field $v$ defined in \eqref{eq: drift-score}, we have
    \begin{equation*}
        \gamma_{2, T}(\cZ) = \cO\br{D^2 d^2 + D^4 d}
    \end{equation*}
\end{lemma}
The lemma reveals a surprising result: for a fixed maximum-mean separation $D$, the straightness metric $\gamma_{2,T}$ is largely independent of the number of mixture components $K$. This theoretical finding is supported by our empirical results in Figure~\ref{fig: gmm-gamma}(a) which demonstrates that the value of $\hat{\gamma}_{2, T}$ estimated as in \ref{eq: gamma_2_estimate} is a function of $D$, increasing as $D$ grows, while showing little variation with changing $K$ when $D$ is held constant.
Moreover, it also immediately yields an $\cO(d^2/T^2)$ bound on the $W_2^2$-error for a given $D$; see Lemma \ref{lemma: bound on L for GMM}, suggesting that the discretization steps should scale linearly with $d$. We defer the proof to Appendix \ref{sec: bound on v_dot_t for GMM}. 

Another attractive property of RF is its \emph{reflow} procedure, where one iteratively rectifies the coupling generated by the preceding flow. \cite{liu2023flow} show that their straightness parameter $S(\cZ)$ decreases with successive applications of the reflow procedure. In the same vein, Figure \ref{fig: gmm-gamma}(b) shows that even our straightness parameter $\gamma_{2, T}$ decreases with successive rectifications. 

Empirical studies argue that not only a single reflow (2-RF) is sufficient to generate a near-straight flow, but too many reflows also hurt the performance due to model collapse caused by error accumulation \citep{lee2024improvingtrainingrectifiedflows}. However, most prior arguments are based on heuristics and a theoretical framework to argue the straightness of 2-RF remains elusive.

\section{{(When) does 2-RF yield a straight flow?}} \label{sec: when is 2-RF straight}
In this section, we pose the following question, which has mostly been addressed empirically for general target distributions previously. 

\textit{When does 2-RF produce a provably straight flow, or equivalently, 1-RF yield a straight coupling?}

\cite{liu2023flow} prove that for one-dimensional target distributions, 2-RF does yield a straight flow since the 1-RF coupling is deterministic and monotonic, and therefore straight (Appendix~\ref{sec: RF in 1-D}). We first illustrate that for some target distributions, such as Gaussians and two-Gaussians mixtures, 1-RF indeed yields a straight coupling even in $d~(\geq 2)$ dimensions. This provides the first concrete theoretical evidence for straightness of 2-RF beyond one-dimensional targets. Motivated by these examples, we develop a general framework for proving straightness and Monge optimality in Section~\ref{sec:sufficient} and \ref{subsec:monge} respectively.

\subsection{Concrete illustrative examples}
We start present a few examples of $d$-dimensional target distributions with $d\geq 2$, where 1-RF yeilds a straight coupling using the true RF drift field in~\eqref{eq: ode-true}.

 
\begin{example}[\textbf{Gaussian to Gaussian}]\label{ex:gtog}
Let $\rho_0=N(0,I_d)$ and $\rho_1=N(\mu,\Sigma)$, where $\mu \in \bbR^d$ and $\Sigma \in \bbR^{d\times d}$ is a symmetric positive-definite matrix.
\end{example}



\begin{theorem}\label{thm:ItoSigma}
    Let $(X_0, X_1) \sim \rho_0 \times \rho_1$ be an independent coupling where $\rho_0$ and $\rho_1$ are specified in Example~\ref{ex:gtog}. 
    The coupling $(Z_0, Z_1) = {\Rectify}{(X_0, X_1)}$ is straight and is given by
 $
        Z_1 = \Sigma^{1/2}Z_0 + \mu.
  $
\end{theorem}
\textbf{Proof sketch.} In this simple case, one can check that the RF velocity is $v_t(Z_t) = \frac{x}{t} + \frac{1-t}{t} \Sigma_t^{-1}(t \mu - x)$, where $\Sigma_t = t^2 \Sigma + (1-t)^2 I_d$. Therefore, ODE \eqref{eq: ode-true} can be solved exactly, yielding the solution $Z_1 = \Sigma^{1/2}Z_0  + \mu$. Consequently, this is also the Monge coupling between $N(0, I_d)$ and $N(\mu, \Sigma)$. Additionally, it is worth mentioning that the optimality of the 1-RF coupling is also a byproduct of a certain commutativity property of $\nabla_{Z_t} v_t(Z_t)$ (see Theorem \ref{prop: 1-rf Monge}). The details of the proof is deferred to Appendix \ref{sec: proof gaussian RF}.

The rest of the section considers target distributions that are multimodal.

\begin{figure}
    \centering
    {\includegraphics[width=0.35\textwidth]{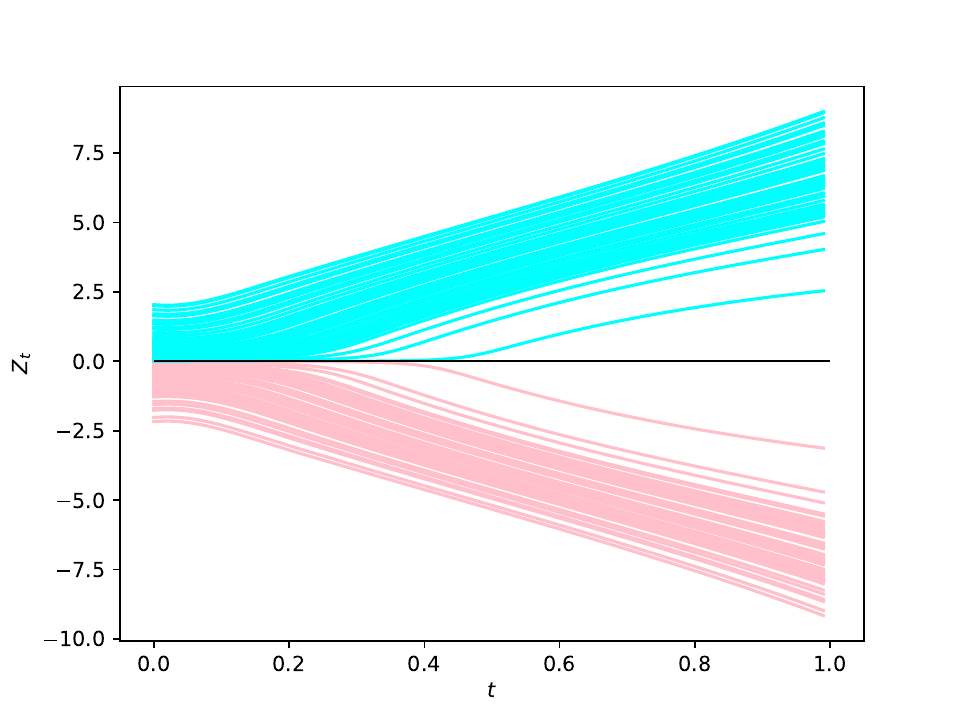} }%
    \caption{
     shows the flow of points from a standard Gaussian $\rho_0 = N(0, 1)$ to a symmetric mixture of two Gaussians $\rho_1=.5N(10,1)+.5N(-10,1)$, where the black line represents $y=0$. 
    }
    \label{fig: quantile}%
\end{figure}
\begin{example}[\textbf{Gaussian to 2-mixture of Gaussian.}]
\label{ex:1to2}
Here, $\rho_0$ is $N(0,I_d)$ and $\rho_1$ is a mixture of two Gaussians with the isotropic covariance matrix, i.e., $\rho_1 := \pi N(\mu_1, \sigma^2I_d) + (1-\pi) N(\mu_2, \sigma^2I_d)$ with $\mu_1, \mu_2 \in \bbR^d$ and $\pi \in [0,1]$.
\end{example}
\begin{theorem}\label{thm:2gauss_rd}
Let $(X_0, X_1) \sim \rho_0 \otimes \rho_1$ be an independent coupling where $\rho_0$ and $\rho_1$ are specified in Example~\ref{ex:1to2}. 
Then 1-RF yields a straight coupling.
\end{theorem}

\textbf{Intuitive proof:} First, note that $(Z_0, Z_1):= {\Rectify}(X_0, X_1)$ is an deterministic coupling. Therefore, the global invertibility of the map $H_t: Z_0 \mapsto (1-t)Z_0 + tZ_1$ is enough to ensure the straightness of $(Z_0, Z_1)$.
Now, consider a simple case of Example~\ref{ex:1to2} with $\mu_1=\mu=-\mu_2$. 
It turns out that in this case, the flow induced by $v_t$ has an interesting geometric structure (see, for example, Figure~\ref{fig: quantile}). In particular, if $z_0$ is positive (negative), then $z_t$ is also positive (negative) \textit{for all $t$}. This holds coordinate-wise. Here $z_0$ and $z_t$ are obtained from ODE \eqref{eq: ode-true}. This follows from a very fundamental fact that flow ODE decouples into $d$ one-dimensional ODEs after an orthonormal transformation.  Since, in the transformed space, the linear interpolation paths of a straight coupling $(Z_{0, i}, Z_{1, i})$ can not intersect for scalar random variables, we must have that it is monotonically increasing,
i.e., for $(z_{0, i}, z_{1, i})$ and $(z_{0, i}', z_{1, i}')$ such that $z_{0, i} < z_{0, i}'$, we must have that $z_{1, i} < z_{1, i}'$ for each co-ordinate $i \in [d]$. This ensures that the RF transport map is co-ordinate wise increasing. Therefore, $H_t(\cdot)$ is invertible and 1-Rf yields a straight coupling.
Furthermore, this also allows to conclude that the flow preserves the \textit{quantiles} co-ordinate wise; see Lemma~\ref{lem:quantile_invariance}.

Finally, we come to the Gaussian mixture to Gaussian mixture setting. 
\begin{example}[\textbf{2-mixture of Gaussians to 2-mixture of Gaussians}]\label{ex:2to2}
 Consider $\mu_{01} = (0,a)^\top, \mu_{02} = (0,-a)^\top$ and $\mu_{11} = (a,a)^\top, \mu_{12} = (a,-a)^\top$ for some $a>0$.
Let $X_0 \sim 0.5 N(\mu_{01}, I_2) + 0.5 N(\mu_{02}, I_2) \text{ and }  X_1 \sim 0.5 N(\mu_{11}, I_2) + 0.5 N(\mu_{12},I_2)$.   
\end{example}

\begin{theorem}\label{thm:mixtomix}
Let $(X_0, X_1) \sim \rho_0 \otimes \rho_1$ be an independent coupling where $\rho_0$ and $\rho_1$ are specified in Example~\ref{ex:2to2}.
Then 1-RF gives a straight coupling.
\end{theorem}
The intuitive explanation is that even in this case, the flows along each coordinate decouples and leads to a straight coupling. We defer the proofs of Theorems \ref{thm:ItoSigma}, \ref{thm:2gauss_rd} and \ref{thm:mixtomix} to Appendix~\ref{sec: proof of main straightness results}. The above proof techniques, while intuitive, do not generalize to examples like a mixture of three or more Gaussians. They also do not provide a way to establish whether the 1-RF coupling gives a Monge map. Now we present an overarching theoretical framework for proving both.
\subsection{Sufficient condition for straightness}
\label{sec:sufficient}
In this section we provide a sufficient analytical condition that makes the 2-RF a straight flow.
To this end, let us consider the ODE \eqref{eq: ode-true} with a fixed initial condition, i.e., 
\begin{equation}
\label{eq: modified ode}
    dZ_t = v_t(Z_t)\; dt, \quad Z_0 = z_0,
\end{equation}
where $v_t$ is the solution of \eqref{eq: drift_obj}.
For clarity, we denote the solution of the above ODE as $Z_t(z_0)$.
It is well known that under a locally-Lipschitz field $v_t$, ODE \eqref{eq: modified ode} admits a unique solution if it's non-explosive (see  Appendix~\ref{sec: general straightness}). These conditions are milder than those imposed in prior literature and are satisfied by a large class of target distributions such as Gaussian-mixtures. Therefore, in the remainder of this section, we always assume that \eqref{eq: modified ode} admits a unique solution. 

Under the above conditions, \citet[page 25-27]{coddington1956theory} show that the Jacobian $J_t^{z_0}:= \nabla Z_t(z)\mid_{z = z_0}$ obeys the following ODE:
 \begin{equation}
 \label{eq: jacobian ODE}
     \frac{dJ_t^{z_0}}{dt} = \nabla_{Z_t}v_t (Z_t(z_0)) J_t^{z_0}; \quad J_0^{z_0} = I_d. 
 \end{equation}
Since the initial coupling is assumed to be independent, Equation \eqref{eq: drift-score} yields that $\nabla_z v_t(z) = (1/t) I_d + t^{-1}(1-t)\nabla_x^2 \log \rho_t(z) $, which is a symmetric matrix for all $ t>0$. Next, we state the sufficient condition on the Jacobian $J_1^{z_0}$ which makes 2-RF straight.

\begin{assumption}
\label{assumption: J1 psd}
    The minimum eigenvalue of $J_1^{z_0} + J_1^{z_0\top}$ is non-negative for all $z_0 \in \bbR^d$.
\end{assumption}

Although, this assumption can be partially checked by simulating ODE \eqref{eq: jacobian ODE} for a large set of initial values (see Appendix \ref{appendix: verifying-assumption empirically}),  
verifying it in practice is challenging, and it may not hold for general target distributions (see Figure \ref{fig: non-Monge example}(b)).  However, in the later sections we theoretically show that Assumption \ref{assumption: J1 psd} is satisfied for some motivating examples.  
We now state our main straightness result below:

\begin{theorem}[1-RF is straight]
    \label{thm: straightness}
    Under Assumption \ref{assumption: J1 psd}, 1-RF yields a straight coupling.
\end{theorem}
As mentioned previous section, the main step of the proof relies on showing that the map $H_t(z_0) := (1-t)z_0 + t Z_1(z_0)$ is globally invertible almost surely in $t \in \Unif([0,1])$. \citet{plastock1974homeomorphisms} shows that a necessary and sufficient condition for a map $f \colon \bbR^d \rightarrow \mathbb{R}^d$ to be an homeomorphism is its local invertibility (namely, that the Jacobian of $f$ is everywhere non-zero) and properness, i.e. $\| f(x) \|_2 \to \infty$ whenever $\|x\|_2 \to \infty$. When specialized to our problem, by \eqref{eq: jacobian ODE} this amounts to showing that, for all $z_0 \in \mathbb{R}^d$, the determinant of $
(1-t) I_d + t J_1^{z_0}$
is non-zero, for almost all $t \in (0,1)$. While, in general, this condition is hard to verify, in Appendix \ref{appendix: 1-rf straight} we show that Assumption \ref{assumption: J1 psd} guarantees both local invertibility and properness of the map under consideration.
Thus, Theorem \ref{thm: straightness} provides an explicit condition for the straightness of the 1-RF coupling or, equivalently,  the 2-RF flow.


While Assumption~\ref{assumption: J1 psd} seems stringent, we would show in Section~\ref{subsec:monge} the surprising fact that both Examples~\ref{ex:gtog} and ~\ref{ex:1to2} satisfy it. We further consider a more general family of target distributions, a mixture of more than two Gaussians, all with the same covariance matrix. We show that as long as the maximum pairwise mean separation is suitably small, 1-RF produces a straight coupling.  
\begin{theorem}
    \label{prop: K-mixture rf}
    Let $(X_0, X_1) \sim N(0, I_d) \otimes \rho_1$ where $\rho_1 := \sum_{j=1}^K \pi_j N(\mu_j,  \sigma^2 I_d)$ with mixture proportions $\{\pi_j\}_{j\in [K]}\in (0,1)^K$. If $\max_{i \ne j} \norm{\mu_i - \mu_j}_2^2 \le 4 \sigma^2$, then 1-RF yields a straight coupling.
\end{theorem}
We note that, for our framework to guarantee straightness, the bound on the maximum pairwise distance is important. See Figure \ref{fig: non-Monge example}(b) where a large pairwise mean separation leads to a violation of Assumption \ref{assumption: J1 psd}. The detailed proof is present in Appendix \ref{appendix: proof of K-mixture rf}. A similar straightness result is also true when $\rho_1 = \sum_{j=1}^K \pi_j N(\mu_j, \Sigma)$ for some positive definite matrix $\Sigma$. The details of the result and its proof is deferred to Appendix \ref{sec: proof of general version of K-mixture}.

We note that Theorem \ref{prop: K-mixture rf} assumes the ideal RF coupling, i.e., one generated by exactly integrating the true RF field. This is analytically unavailable even though the field is exactly known (see Lemma \ref{lemma: GMM general covariance v}.) However, as we show in Lemma~\ref{lemma: bound on gamma_2}, even the discretizated ODE leads to small errors, resulting in a near straight 1-RF coupling, i.e.,  2-RF flow is straight.
\begin{figure}%
    \centering
    \subfloat[\centering]{
    {\includegraphics[width=0.22\textwidth]{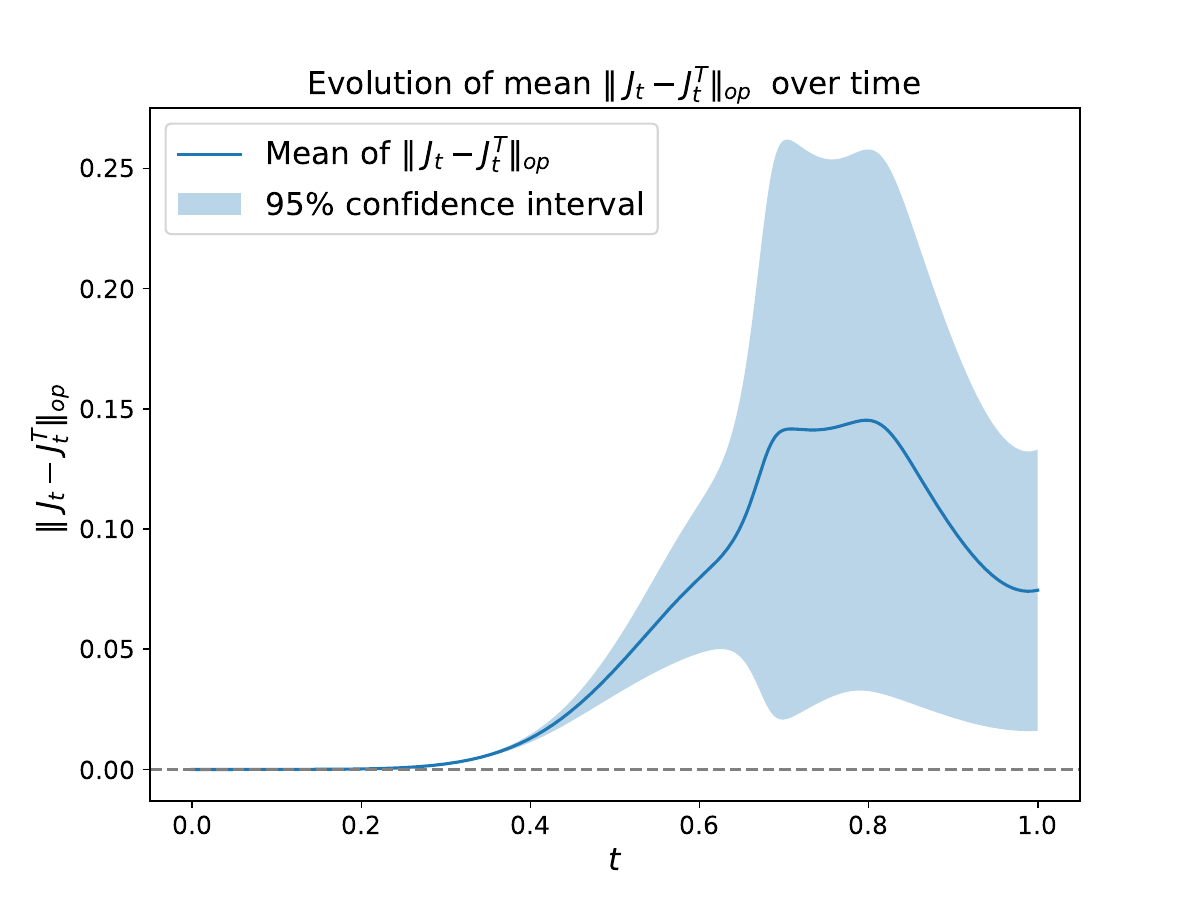} }}
    \subfloat[\centering]{
    {\includegraphics[width=0.22\textwidth]{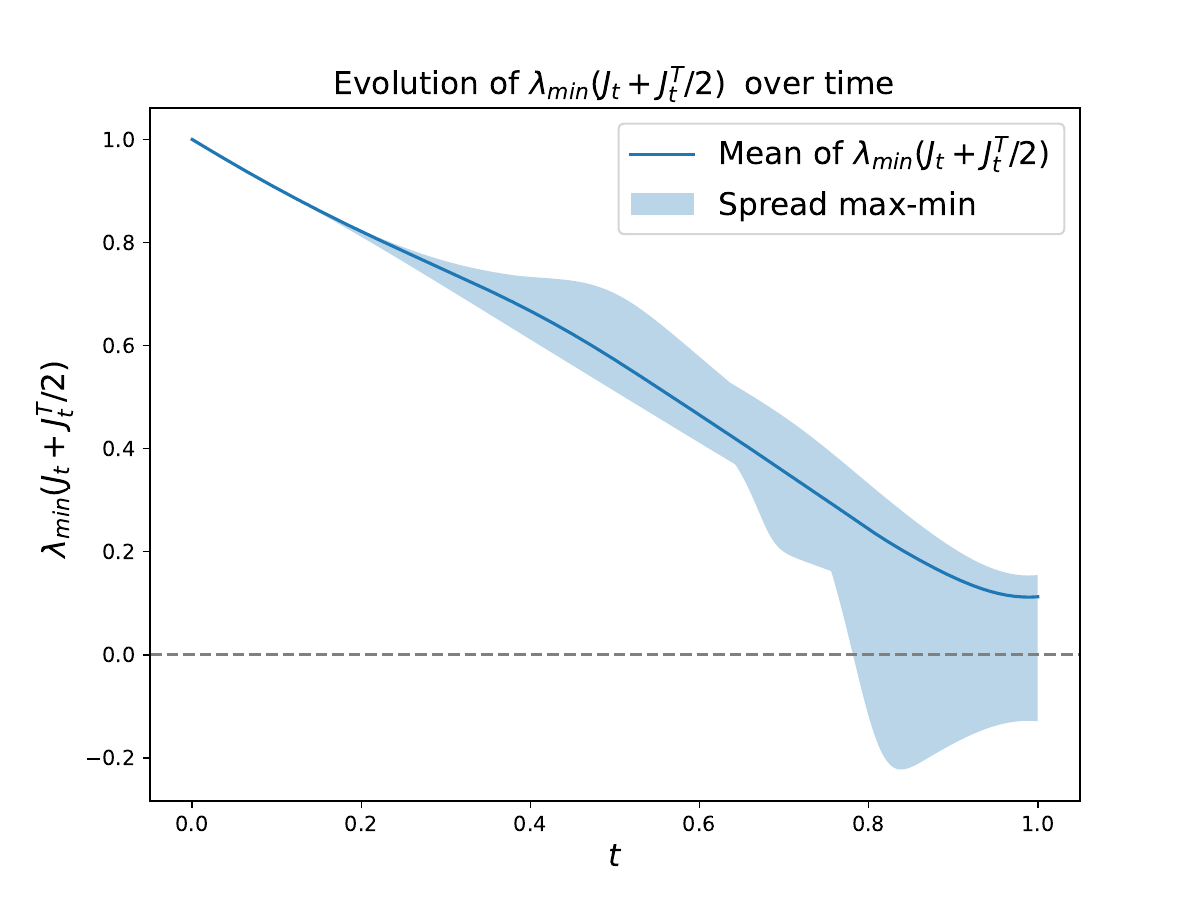} }}
    \caption{The above plots show the evolution of mean $\Norm{J_t^{z_0} - J_t^{z_0 \top}}_{op}$ and $0.5*\lambda_{min}(J_t^{z_0} + J_t^{z_0 \top})$ over 100 different random initial points $z_0\sim N(0, I_2)$ for $\rho_1 = \sum_{k = 1}^4 N(\mu_k, \sigma^2 I_2)$ with weights $\pi_k = k/10$ for $k \in \{1,2,3,4\}$, and $\mu_1 = (0,-2), \mu_2 = (-1,-2), \mu_3 = (1, 3) , \mu_4 = (2,0)$ and $\sigma = 0.1$. 
    }
    \label{fig: non-Monge example}%
\end{figure}
\subsection{Connection to Monge map}
\label{subsec:monge}
It is important to clarify the distinction between straight and optimal couplings. While an OT (Monge) map for convex costs is always straight, the converse is not true; any straight coupling is not necessarily the optimal solution to the Monge problem. In this section, we show that 1-RF does lead to the Monge map between $\rho_0$ and $\rho_1$ under an additional sufficient condition on the velocity drift $v_t$. We emphasize that 1-RF may not yield a Monge map in general, and one needs to apply $c$-Rectified Flow \citep{liu2022rectified} to obtain an approximate solution to Monge problem.

\begin{theorem}[1-RF yields Monge map]
\label{prop: 1-rf Monge}
    Assume that $[\nabla_{Z_t}v_t(Z_t(z_0)) ,\nabla_{Z_s}v_s(Z_s(z_0))] = 0$ for all $t \ne s$, and initial points $z_0$. Then 1-RF yields the optimal transport plan for the OT problem \eqref{eq: monge-problem}, and its Jacobian satisfies
    \[
    J_1^{z_0} = \exp\left(\int_0^1 \nabla_{Z_t}v_t(Z_t(z_0))\;dt\right).
    \]
\end{theorem}

The proof of the theorem is provided in Appendix~\ref{appendix: 1-rf monge}. Since \(\nabla_{Z_t} v_t(Z_t(z_0))\) is symmetric for all \(t \in [0,1]\), the matrix \(J_1^{z_0}\) is symmetric and positive-definite. Brenier's theorem~\citep{chewi2024statisticaloptimaltransport} suggests that for $z_0 \mapsto Z_1(z_0)$ to be the Monge map, $Z_1 = \nabla \varphi$ for some \textit{convex function} $\varphi: \bbR^d \to \bbR$, i.e.,  the Jacobian $J_1^{z_0} = \nabla^2 \varphi(z_0)$ has to be a \textit{symmetric and positive semi-definite} matrix. Thus, under commutativity condition above, \textit{1-RF yields the \(\ell_2\)-optimal (Monge) coupling, which is also straight}. However, we emphasize that the converse might not be true. 

In one dimension, 1-RF always recovers the Monge map, as the commutativity condition holds trivially; see Appendix~\ref{sec: RF in 1-D}. We show more examples below:

\begin{proposition}
\label{prop: monge map gaussian & 2-mix}
    Let $(X_0, X_1) \sim \rho_0 \otimes \rho_1$ be an independent coupling. When $\rho_0$ and $\rho_1$ are specified in Example~\ref{ex:gtog} or~\ref{ex:1to2}, 1-RF produces the optimal Monge map, and therefore a straight coupling. 
\end{proposition}
In these examples, one can easily verify that the conditions in Theorem 
 \ref{prop: 1-rf Monge} are satisfied. Therefore, it follows that 1-RF yields the straight and optimal coupling; see Appendix \ref{sec: proof of monge map gaussian & 2-mix}.

In contrast, for a general mixture of Gaussians, the gradient of the velocity may not commute. So, even though under a certain condition on the means, the resulting 1-RF yields a straight coupling (Theorem~\ref{prop: K-mixture rf}), it may not be the Monge map (also see Figure~\ref{fig: non-Monge example}).

\section{{Experiments}}
\label{sec: experiments}
In this section, we primarily explore the effect of the number of discretization steps $T$ and the straightness parameter $\gamma_{2,T}(\cZ)$ on the $W_2$ distance between the target distribution and the sampling distribution of the first and second RFs. We present numerical experiments for some synthetic and real datasets. Additional experiments can be found in Appendix \ref{sec: app experiments}.  We provide the code at \href{https://github.com/bansal-vansh/rectified-flow-straightness}{\texttt{{https://github.com/bansal-\\vansh/rectified-flow-straightness}}}.


\begin{figure}
    \centering
    \subfloat[\centering]{{\includegraphics[width=0.23\textwidth]{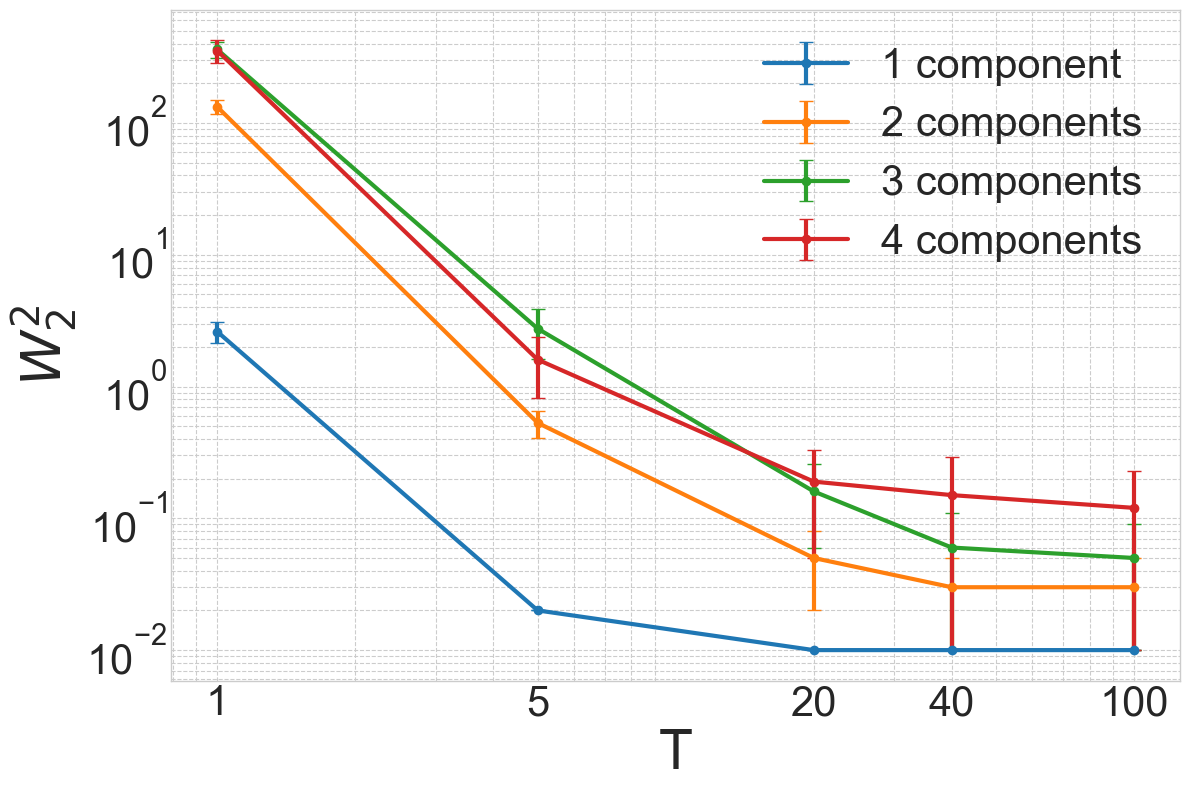} }}
    \hfill
    \subfloat[\centering]{{\includegraphics[width=0.23\textwidth]{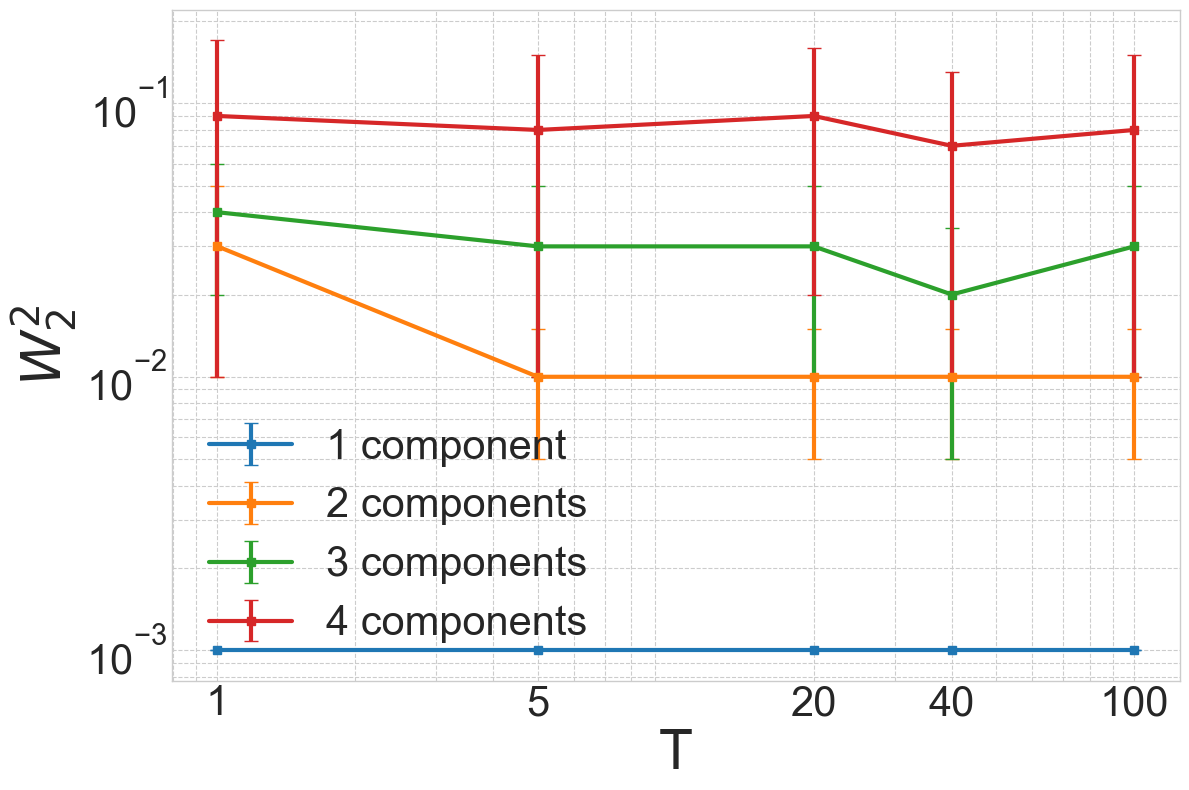} }}
    \caption{For the same mixture-of-Gaussians target distributions and estimated drift model in Figure~\ref{fig: gmm-gamma}, we show (a) $W_2^2 (\widehat{\rho}_1, \rho_1)$ vs $T$ for the first RF (b) $W_2^2 (\widehat{\rho}_1, \rho_1)$ vs $T$ for the second RF.}
    \label{fig: gmm}
\end{figure}

\paragraph{Synthetic data:} We first start with the independent 
coupling $(X_0, X_1) \sim N(0, I_d) \otimes \rho_1$, where $\rho_1$ is a mixture of Gaussians with $K\in \{1, 2, 3, 4\}$ components and varying maximum mean separation $D$ (see details in Appendix \ref{sec: exp mixture of Gaussian}). For all four cases, we estimate the RF drift using a feed-forward neural network and generate the 1-RF samples using the ODE \ref{eq: emp-ode-disc} with $T$ discretization steps. Figure \ref{fig: gmm}(a)-(b) shows that $W_2(\widehat{\rho}_1, \rho_1)$ steeply decreases with increasing number of steps $T$, 
validating our Theorem \ref{thm: W2 ode disc}. Moreover, we note that $W_2$ distance is consistently larger for the flow corresponding to a larger component mean separation $D$, owing to a larger value of the straightness parameter $\gamma_{2,T}(\cZ)$ as shown in Figure \ref{fig: gmm-gamma}(a)-(b). Lastly, in accordance with the $\gamma_{2, T}$ plots in Figure~\ref{fig: gmm-gamma}(b), Figure \ref{fig: gmm}(b), further empirically validates that the 2-RF for Gaussian mixtures produces a near-straight flow, since the Wasserstein error even with a single discretization step is close to 0.
\begin{figure}
    \centering
    \subfloat[\centering ]{{\includegraphics[width=0.23\textwidth]{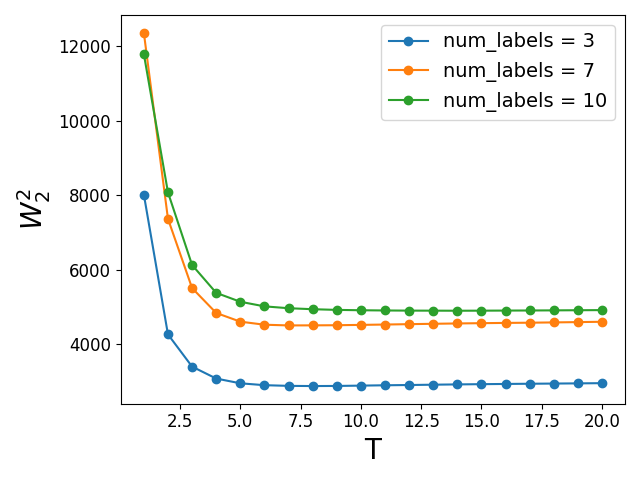} }}%
    \hfill
    \subfloat[\centering  ]{{\includegraphics[width=0.23\textwidth]{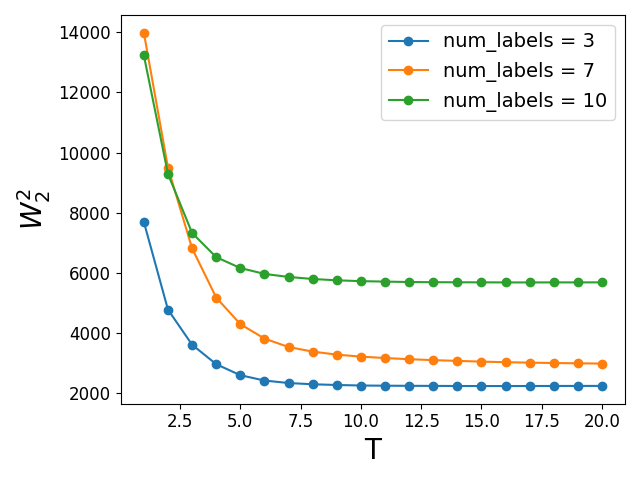} }}%
    \hfill
    \caption{The figure shows $W_2^2 (\widehat{\rho}_1, \rho_1)$ vs $T$ for the first RF for (a) MNIST dataset (b) FashionMNIST dataset with varying number of labels.}
    \label{fig: real}
\end{figure}
\paragraph{Real data:}
For the real data experiments, we consider the MNIST and FashionMNIST datasets. In both examples, we train a UNet architecture-based network on training data to estimate the drift function and then evaluate the Wasserstein distance of the generated samples from the test split of the dataset.
To emulate the behavior of having different number of modes, we consider three subsets of both the datasets consisting of the first 3 labels, the first 7 labels, and all 10 labels. We observe in Figure \ref{fig: real} that similar to the Gaussian mixture example,  the presence of a higher number of components, possibly indicating towards increasing mode separation, increases the Wasserstein error, again indicating that the flow becomes less straight with the increasing number of modes.

 \section{{Conclusion}}
This paper establishes the first formal link between flow geometry and sampling efficiency by introducing the piecewise straightness parameter, $\gamma_{2,T}$. We derive a novel Wasserstein convergence bound showing that the $W_2^2$-error scales with $\cO(\gamma_{2,T}/T^2)$, providing a rigorous foundation for why straighter flows enable high-fidelity, few-step generation.

Building on this, we propose the first theoretical framework to analyze straightness in Rectified Flow (RF). While the sufficient conditions we identify for achieving a perfectly straight flow ($\gamma_{2,T}=0$) can be conservative in some cases, they eliminate the discretization error in our bound and allow for the first concrete proofs of perfect one-step generation in key multi-dimensional settings. Ultimately, our work provides a methodology to advance the study of flow straightness from an empirical heuristic to a provable principle.

\paragraph{Open problem:} Recall that Assumption \ref{assumption: J1 psd} is only sufficient to ensure global invertibility of the map $H_t(\cdot)$ almost surely for all $t \sim \Unif([0,1])$. Even in our simulations, it appears to not be necessary in some cases, thus suggesting that $H_t(\cdot)$ remains globally invertible (almost surely in $t$) under milder conditions. In particular, we conjecture that $H_t(\cdot)$ is globally invertible almost surely in $t \sim \Unif([0,1])$ as long as $\rho_0$ is the standard Gaussian and $\rho_1$ is a general mixture of Gaussians. Proving this conjecture remains a challenging open problem. 

\paragraph{Acknowledgments.} We thank Dr. Dheeraj Nagaraj at Google DeepMind India and Dr.
Qiang Liu at UT Austin for engaging in helpful discussions which helped improve the manuscript significantly. We also thank Dr. Adam Klivans and the Institute for Foundations of Machine Learning (IFML) at UT Austin for providing the compute resources. The work was also supported by the NSF grants CCF-2019844, CCF-2505865,  2217069, NIH Award RF1NS121913, and DMS grant 2109155.

\bibliography{ref}
\bibliographystyle{plainnat}
\appendix

\thispagestyle{empty}

\onecolumn
\aistatstitle{Appendix}


\section{{Experimental details}}
\label{sec: app experiments}

\subsection{Details of synthetic data experiments with mixture of Gaussian distributions}
\label{sec: exp mixture of Gaussian}
We choose the target distribution to be mixture of Gaussians with equal cluster probability and unit variance, where the number of components $K$ varies within $\{1,2,3,4\}$. In all the cases, we show that the actual Wasserstein error is closely characterized by $\gamma_{2, T}.$ 
For $K = 1$, we set mean of the target distribution to be $\mu_1 = (5, 0)^\top$; for $K = 2 $: $\mu_1 = (5, 0)^\top, \mu_2 = (0, 5)^\top$; for $K = 3$: $\mu_1 = (5, 0)^\top, \mu_2 =  (0, 5)^\top, \mu_3=  (-5, 0)^\top$; for $K = 4$:  $\mu_1 = (5, 0)^\top, \mu_2 =  (0, 5)^\top, \mu_3=  (-5, 0)^\top, \mu_4 = (0, -5)^\top$. 

\subsection{Checker board example}
\label{sec: app checker board}
We consider the checker-board distribution with $2$, $5$ and $8$ components. We use training datasets of size 10,000 to train a feed-forward neural network in order to learn the velocity drift function and evaluate $W_2^2(\widehat{\rho}_1, \rho_1)$ using POT \citep{feydy2019interpolating} for different levels of discretization $T$ over test data of size 5000. Figure \ref{fig: checker board}(d) also shows that larger component size has a negative effect on the Wasserstein distance, which stems from the fact that a larger number of components typically pushes the flow away from straightness.

\begin{figure}[h!]
\centering
\begin{tabular}{c c c c}
     \includegraphics[width=0.2\textwidth]{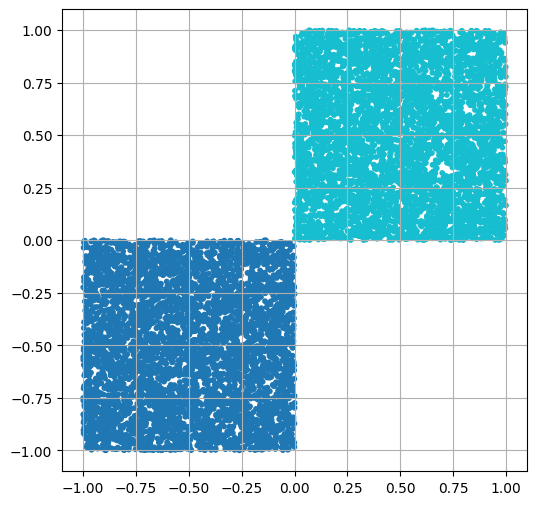}& 
     \includegraphics[width=0.2\textwidth]{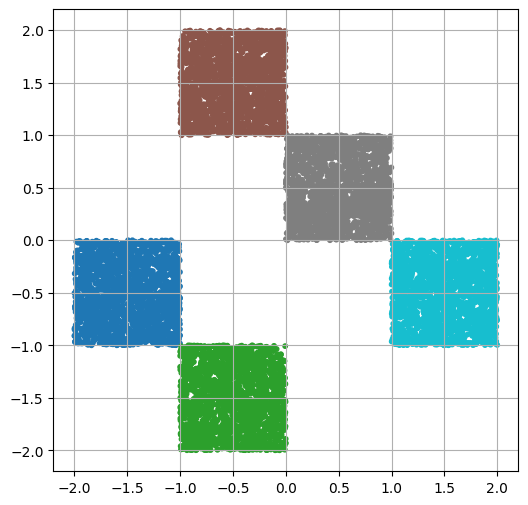}&
     \includegraphics[width=0.2\textwidth]{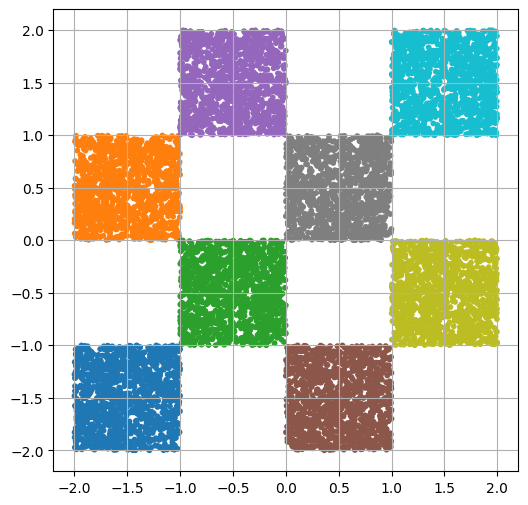}&
     \includegraphics[width=0.23\textwidth]{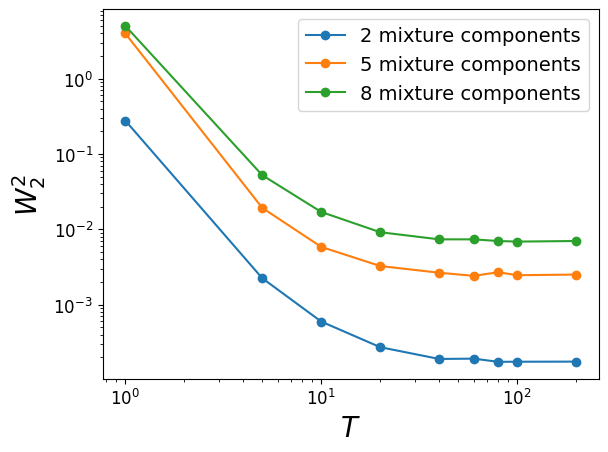}\\
     (a)&(b)&(c)&(d)
\end{tabular}
    \caption{(a) Checker-board distribution with 2 components. (b) Checker-board distribution with 5 components. (c) Checker-board distribution with 8 components. (d)  shows the $W_2^2 (\widehat{\rho}_1, \rho_1)$ vs $T$ (on log-log scale) for the Checker-board distribution with varying components.}
    \label{fig: checker board}
\end{figure}


\subsection{Additional Large-Scale Real Data Experiment on CelebA}
\label{app:celeba}

To further examine whether our theoretical findings extend beyond relatively small benchmark datasets such as MNIST and FashionMNIST, we conducted an additional experiment on the CelebA dataset, which contains roughly 200{,}000 images and is substantially larger and more diverse. This experiment was designed to probe the effect of increasing modal complexity on the discretization error of Rectified Flow, and to test whether the trends predicted by our theory continue to hold at a larger scale.

\paragraph{Experimental setup.}
We trained Rectified Flow models on three CelebA subsets with increasing diversity:
\begin{enumerate}
    \item \textbf{Blonde Female}: a relatively homogeneous subset with the fewest modes,
    \item \textbf{Blonde All}: a broader subset containing all blonde subjects,
    \item \textbf{All Classes}: the full dataset subset with the highest diversity and the largest effective number of modes.
\end{enumerate}
This ordering was chosen to approximately vary the underlying mode complexity from low to high while keeping the evaluation protocol fixed across all settings. For each trained model, we evaluated the squared 2-Wasserstein error between the generated distribution and the target data distribution under Euler discretization with varying numbers of sampling steps $T \in \{1,\dots,10\}$. The resulting values are reported in Figure~\ref{fig:celeba_w2}

\begin{figure}[t]
    \centering
    \includegraphics[width=0.72\linewidth]{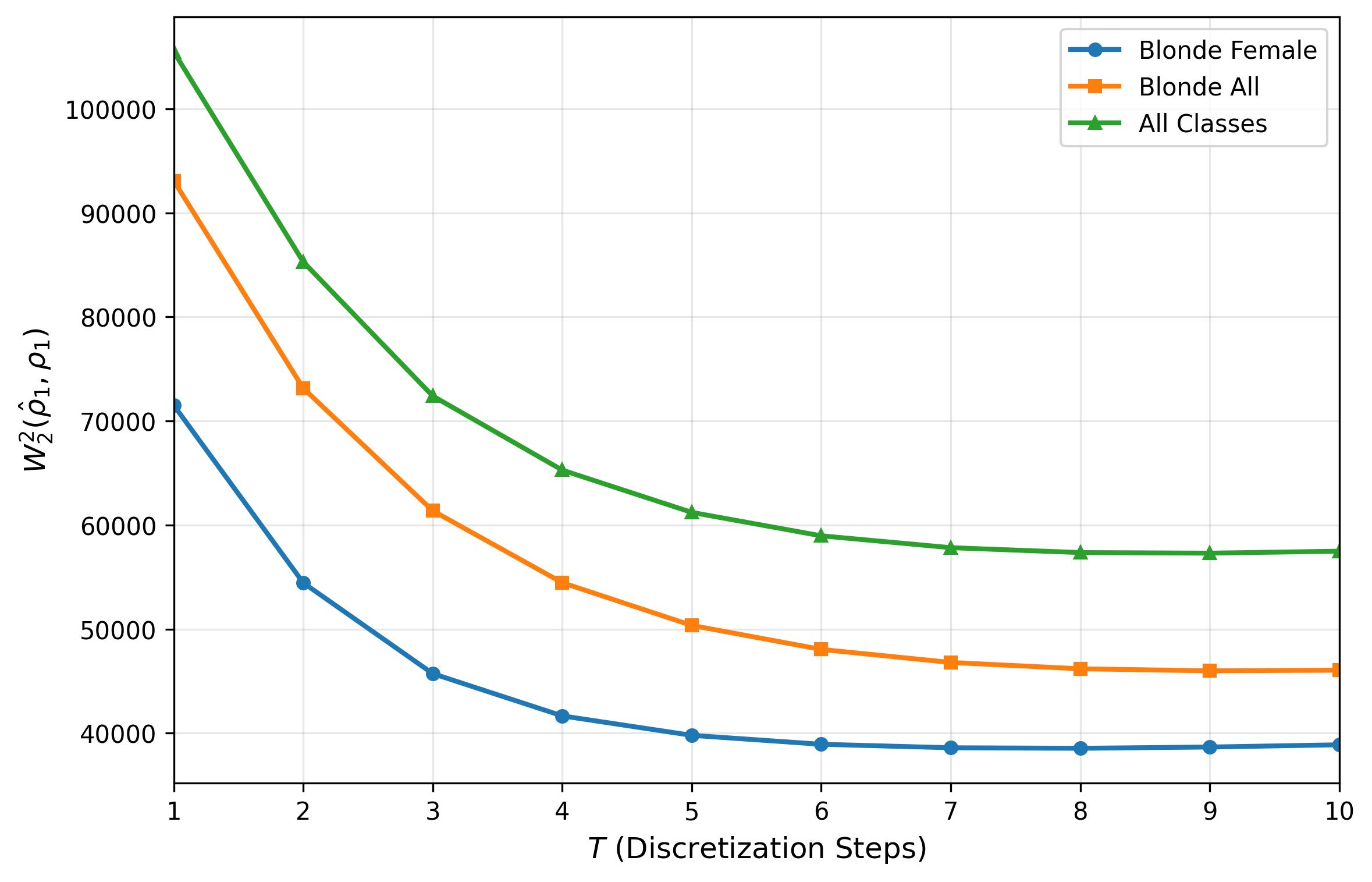}
    \caption{Squared Wasserstein error $W_2^2(\hat{\rho}_1,\rho_1)$ versus the number of discretization steps $T$ on three CelebA subsets with increasing modal complexity. The discretization error decreases rapidly with $T$ and stabilizes around $T\approx 8$. Moreover, for every fixed $T$, the error increases monotonically as the dataset becomes more multimodal, from \emph{Blonde Female} to \emph{Blonde All} to \emph{All Classes}.}
    \label{fig:celeba_w2}
\end{figure}

\paragraph{Observations.}
The CelebA experiment reveals two clear trends.

First, for all three subsets, the squared Wasserstein error decreases sharply as the number of discretization steps increases, and then stabilizes around $T \approx 8$. This is consistent with our theoretical analysis in Theorem~2, which predicts that discretization error should decrease as the flow is integrated more finely.

Second, for every fixed value of $T$, the error is smallest for \emph{Blonde Female}, larger for \emph{Blonde All}, and largest for \emph{All Classes}. Since these three settings are ordered by increasing diversity and effective mode complexity, this provides further evidence that more complex multimodal datasets induce less straight flows and therefore require more discretization steps to achieve the same sampling fidelity.

\paragraph{Implication.}
These results support the broader relevance of our theory beyond small benchmark datasets. Although training on full ImageNet is outside the scope of the present work, the CelebA results already demonstrate the same qualitative scaling behavior predicted by our straightness-based analysis: increasing mode complexity worsens discretization error, while increasing $T$ compensates for this by better resolving the flow trajectory.

Overall, this experiment strengthens the empirical evidence that the geometry of the learned flow---as captured by our straightness perspective---continues to govern sampling accuracy even in substantially larger and more realistic image datasets.
    
\subsection{Verifying Assumption \ref{assumption: J1 psd} empricially} \label{appendix: verifying-assumption empirically}

\begin{figure*}[htbp]
    \centering
    \subfloat[\centering  ]{{\includegraphics[width=0.25\textwidth]{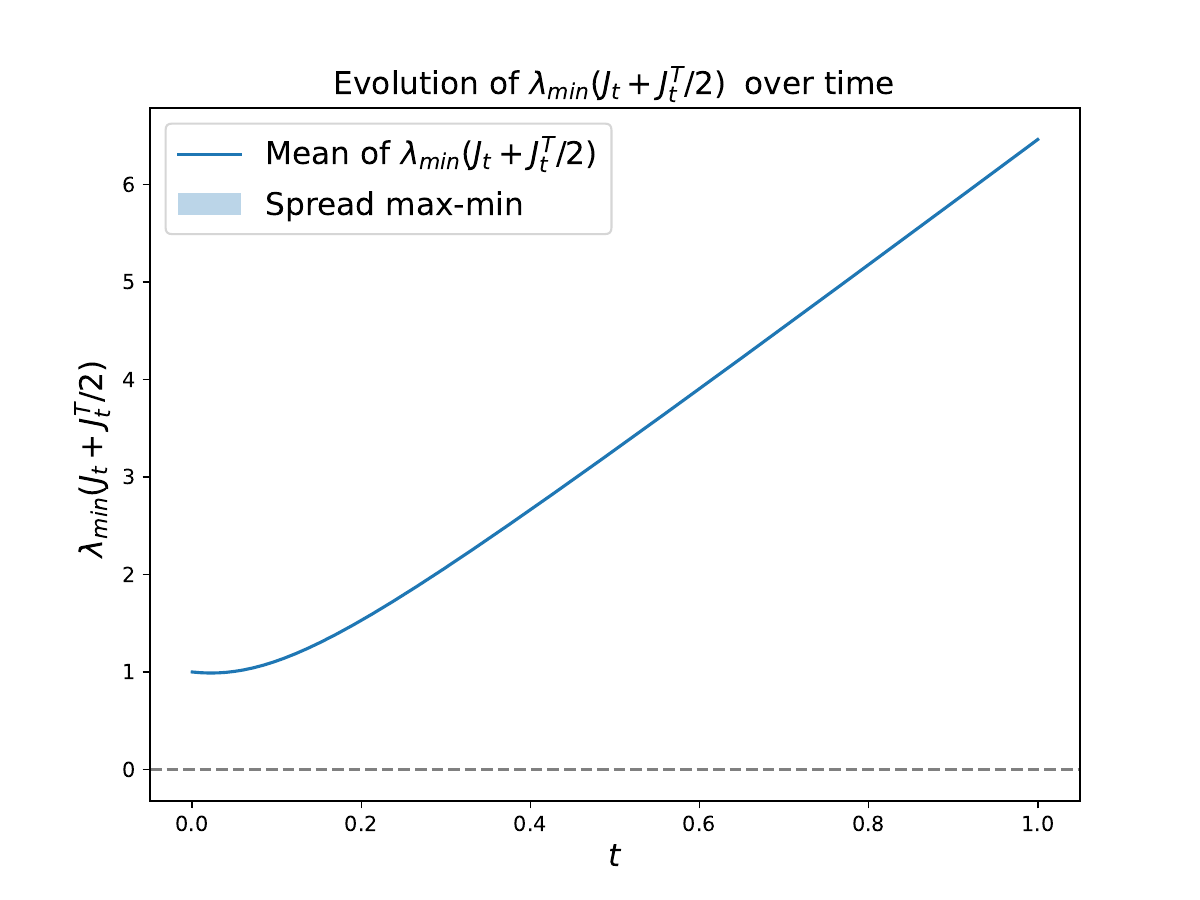} }}%
    \hfill
    \subfloat[\centering]{{\includegraphics[width=0.25\textwidth]{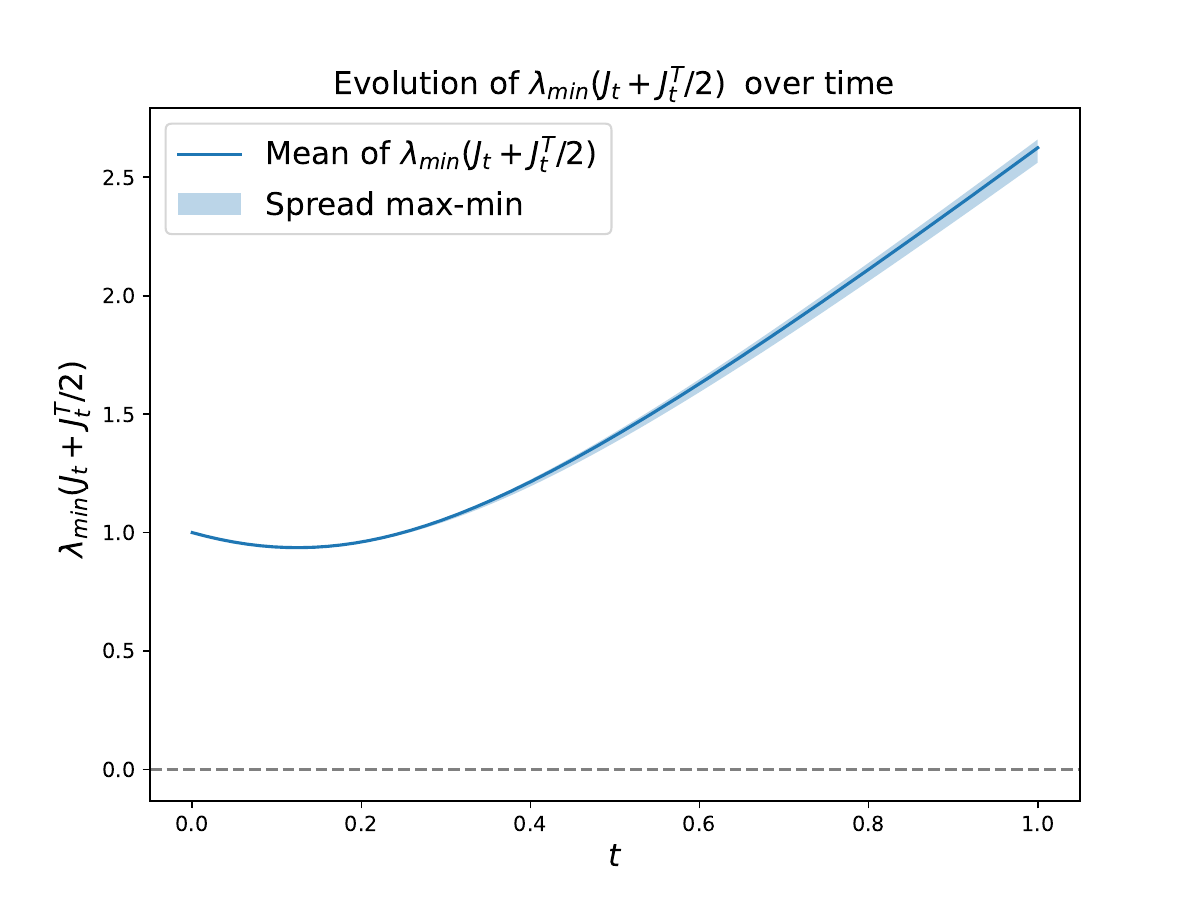} }}
    \hfill
    \subfloat[\centering]{{\includegraphics[width=0.25\textwidth]{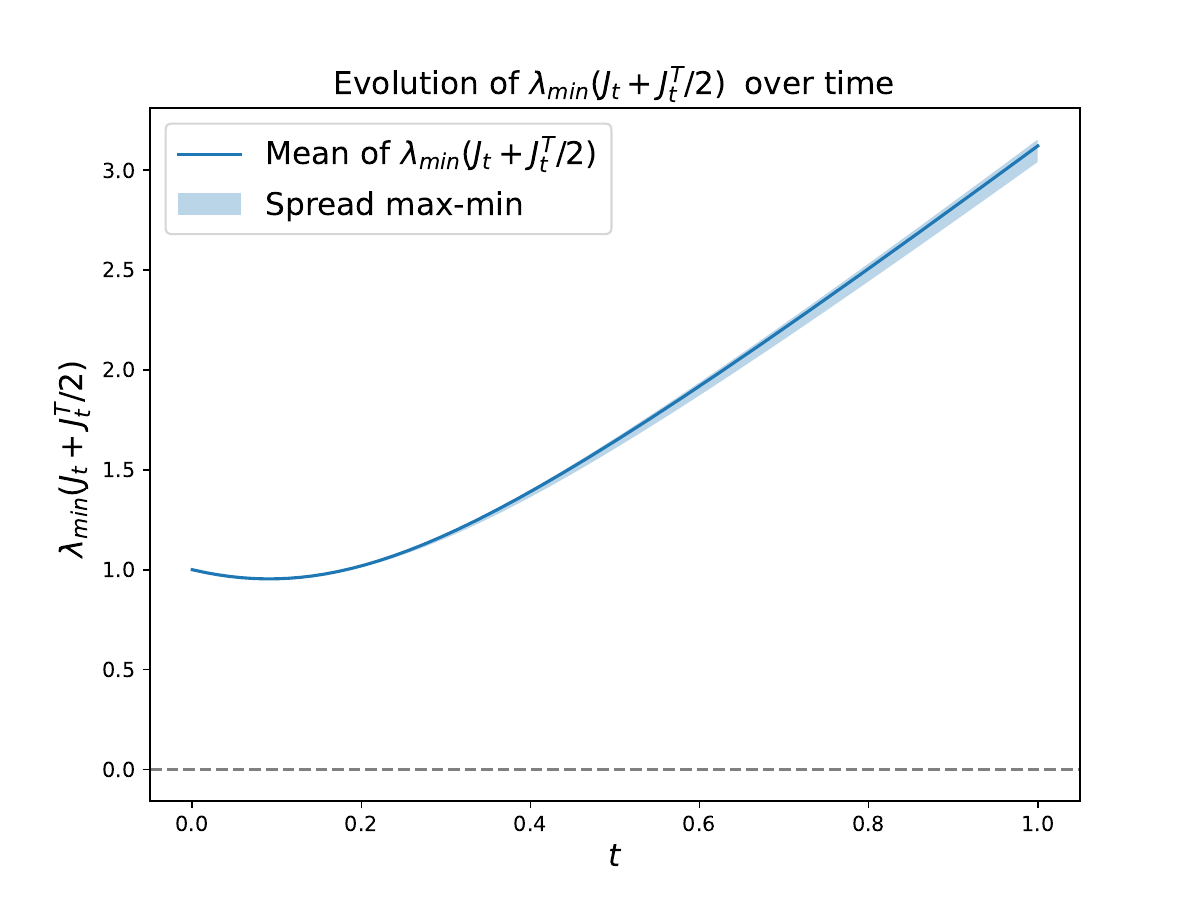} }}
    \caption{
Evolution of \(0.5 \times \lambda_{\min}(J_t^{z_0} + J_t^{z_0\top})\) over time for 100 samples \(z_0 \sim \mathcal{N}(0, I_2)\), under target distributions \(\rho_1 = \sum_{i=1}^K \pi_i\, \mathcal{N}(\mu_i, \sigma^2 I)\) with \(K \in \{2, 3, 4\}\), corresponding to subplots (a), (b), and (c), respectively.
}
 \label{fig: psd_assumption verification}
\end{figure*}


Figure \ref{fig: psd_assumption verification} shows an empirical verification of Assumption \ref{assumption: J1 psd}.
The parameter settings given as follows: 
\begin{itemize}
    \item[(a)] \(K=2\): \(\mu_1 = (5, 1),\ \mu_2 = (-7, -2),\ \sigma = 6.5,\ \pi_1 = 0.6,\ \pi_2 = 0.4\).
    \item[(b)] \(K=3\): \(\mu_1 = (1, 2),\ \mu_2 = (2, 0),\ \mu_3 = (-1, -2),\ \sigma = 2.5,\ \pi_1 = \pi_2 = 0.4,\ \pi_3 = 0.2\).
    \item[(c)] \(K=4\): \(\mu_1 = (1, 3),\ \mu_2 = (2, 0),\ \mu_3 = (-1, -2),\ \mu_4 = (0, -2),\ \sigma = 3,\ \pi_1 = 0.3,\ \pi_2 = 0.4,\ \pi_3 = 0.2,\ \pi_4 = 0.1\).
\end{itemize}

\section{{Existence and uniqueness of Rectified flow}}
\label{sec: general straightness}
Prior works \cite{liu2023flow, liu2022rectified} assume (also Section \ref{sec: convergence rate}) that the velocity is Lipschitz smooth, which is a sufficient condition for the existence of a unique solution to ODE \eqref{eq: ode-true}. However, such conditions could be restrictive in practice. In this section, we will work with somewhat more pragmatic conditions on the true velocity functions $v_t$. Recall ODE \eqref{eq: modified ode}, i.e., 
\begin{equation*}
    dZ_t = v_t(Z_t)\; dt, \quad Z_0 = z_0.
\end{equation*}



\begin{definition}
  For a positive integer $k$, a function $f: \bbR^d \to \bbR^d$ is said to be $\cC^{k}$  if it is $k$-times continuously differentiable. Additionally, $f$ is called a $\cC^{1,1}$ function if $f$ is a $\cC^1$ function and its Jacobian is  locally Lipschitz, i.e., for every $x\in \bbR^d$, there exists $\delta >0$ and $L_{loc}>0$ (which may depend on $x$) such that 
    \[
    \max\{\norm{x-x_1}_2, \norm{x-x_2}_2\} \le \delta \Rightarrow \norm{\nabla_x f(x_1) - \nabla_x f(x_2)}_{op} \le L_{loc} \norm{x_1-x_2}_2.
    \]
\end{definition}

\begin{assumption}
    \label{assumption: locally L-lipschitz}
    We assume that the velocity function $v_t(\cdot)$ is a $\cC^{1,1}$ function for all $t \in [0,1]$.
\end{assumption}
Note that, if $v_t(\cdot)$ is a $\cC^{2}$ function, then it automatically satisfies Assumption \ref{assumption: locally L-lipschitz}. This is satisfied as long as the target distribution has the second moment (See Theorem \ref{thm: gaussian straight flow}). Therefore, we argue that Assumption \ref{assumption: locally L-lipschitz} is less restrictive compared to the global Lipschitzness condition. 
Now we present a general existence and uniqueness result for Rectified Flow. 
\begin{proposition}[Existence and Uniqueness]
    \label{thm: general straight flow}
    Let $\bbE\norm{X_1}_2< \infty$ and the Assumption \ref{assumption: locally L-lipschitz} hold. Also, assume that the solution to the ODE \eqref{eq: modified ode} satisfies the non-explosive condition
    \begin{equation}
    \label{eq: non-explosive}
    \sup_{t \in  [0,1]} \norm{Z_t(z_0)}_2 < \infty \quad \text{for all initial values $z_0 \in \bbR^d$}.
    \end{equation}
    Then there exists a unique solution to ODE \eqref{eq: modified ode}.
\end{proposition}
The above proposition is a consequence of Theorem 5.2 of \cite{kunita1984stochastic}. The non-explosivity condition is particularly important for the existence of the ODE as it avoids any singularities in the solution path. However, this condition is hard to verify apriori. However, we provide a sufficient condition for non-explosivity that is easier to check so that Proposition \ref{thm: general straight flow} can be of practical use. 

\begin{assumption}[Osgood type criterion \citep{osgood1898beweis, groisman2007explosion}]
\label{assumption : osgood}
    Let $Z_t(z_0) \in \bbR^d$ be the solution of the ODE \eqref{eq: modified ode}, where $(z_0,t) \in \bbR^d \times [0,1]$. There exists a non-negative locally-Lipschitz (or strictly increasing) function $h: \bbR_+ \to \bbR_+$ such that 
     \begin{equation}
    \label{eq: osgood}
    \int_{u_0}^\infty \frac{1}{h(u)}\; du  > 1, \quad \text{for all $u_0 > 0$},
    \end{equation}
   and  $ \innerprod{Z_t(z_0), v_t(Z_t(z_0))} \le h(\norm{Z_t(z_0)}_2^2)$, for all $(z_0,t) \in \bbR^d \times [0,1]$.
\end{assumption}
\vspace{-2mm}
  One sufficient condition is that $\sup_{t \in [0,1]} \innerprod{x, v_t(x)} \le h(\norm{x}_2^2)$ for all $x \in \bbR^d$ and for a positive locally-Lipschitz (or strictly increasing) function $h$  satisfying \eqref{eq: osgood}. The above criterion ensures that $\norm{Z_t(z_0)}_2$ is always finite for all $t \in [0,1]$ \citep{groisman2007explosion}, i.e., the solutions does not explode. To be precise, the integral in \eqref{eq: osgood} quantizes the explosion time of $\norm{Z_t(z_0)}_2$, and it ensures that the explosion time falls outside $[0,1]$. 
  Moreover, as opposed to condition  \eqref{eq: non-explosive}, this can be easily checked for a large class of target distributions, e.g., a general mixture of Gaussians. For example, for $(X_0, X_1) \sim N(0, I_d) \times \rho_1$ with $\rho_1 = \sum_{j = 1}^J \pi_j N(\mu_j, \Sigma_j)$, it follows that $\sup_{t \in [0,1]} \innerprod{x, v_t(x)} \le A\norm{x}_2^2 + B \norm{x}_2$ for some $A, B>0$ (see  Appendix \ref{sec: app gaussian to general mixture}). Therefore, $h(u) = A u + B \sqrt{u}$ is a valid choice and it also satisfies Assumption \ref{assumption : osgood}, as $\int_{u_0}^\infty (A u + B \sqrt{u})^{-1} \; du = \infty$ for all $u_0>0$. 
  We now state the main result below.
\begin{theorem}
\label{thm: gaussian straight flow}

    Let $(X_0,X_1) \sim N(0, I_d) \times \rho_1$ such that $\bbE \norm{X_1}_2< \infty$. Then, the velocity $v_t(\cdot)$ satisfies Assumption \ref{assumption: locally L-lipschitz}. Moreover, under Assumption \ref{assumption : osgood}, there exists a unique solution to ODE \eqref{eq: modified ode}. 
    
\end{theorem}

The above theorem gives a fairly general existence and uniqueness (without Lipschitz smoothness) result for Rectified Flow starting from an independent coupling that covers a large class of target distributions. Essentially, the first moment ensures that Assumption \ref{assumption: locally L-lipschitz} is satisfied. Therefore, coupled with Assumption \ref{assumption : osgood}, the conditions of Proposition \ref{thm: general straight flow} are satisfied, and hence, the result follows.  The complete proof is deferred to Appendix \ref{sec: app gaussian straight flow}.

\section{{Proofs for Wasserstein convergence Bounds}}

\subsection{Proof of Theorem \ref{thm: pdata bound}}
\label{sec: proof W2 ode rate}
Let $\{\rho_t\}_{t \in [0,1]}$ and $\{\tilde{\rho}_t\}_{t\in [0,1]}$ be distribution of the solution of \eqref{eq: ode-true} and \eqref{eq: ode-approx} respectively. Let $\pi_t$ be the optimal coupling between $\rho_t$ and $\tilde{\rho}_t$. Therefore, using Corollary 5.25 of \cite{santambrogio2015optimal}, we have 
\begin{align*}
    \frac{1}{2} \frac{d W_2^2(\rho_t, \tilde \rho_t)}{dt} &= \int \innerprod{x - y, v_t(x) - \widehat{v}_t(y)} \; d\pi_t(x,y)\\
    & =  \int \innerprod{x - y, v_t(x) - \widehat{v}_t(x)} \; d\pi_t(x,y)  + \int \innerprod{x - y, \widehat{v}_t(x) - \widehat{v}_t(y)}  \; d\pi_t(x,y)\\  
    & \le \frac{1}{2}\int \norm{x - y}_2^2 \; d\pi_t(x,y) + \frac{1}{2}\int \norm{v_t(x) - \widehat{v}_t(x)}_2^2 \; d\pi_t(x,y) + \widehat{L}  \int \norm{x - y}_2^2 \; d\pi_t(x,y)\\
    & = (1/2 + \widehat{L}) W_2^2(\rho_t, \tilde \rho_t) + \frac{b(t)}{2}.
\end{align*}

Solving the above differential inequality leads to the following inequality
\[
W_2^2(\rho_\tau, \tilde \rho_\tau) \le  W_2^2(\rho_0, \tilde \rho_0) + e^{1 + 2\widehat{L}} \int_{0}^\tau b(t)\; dt.
\]
The result follows by noting that $W_2^2(\rho_0, \tilde \rho_0) = 0$ and setting $\tau = 1$.

\subsection{Comparison of straightness parameters}
\label{sec: proof of lemma straightness comparison}

\begin{lemma}
\label{lemma: straightness comparison}
    The AS and PWS parameters satisfy $ \gamma_{2,T}(\cZ)\ge \gamma_1(\cZ) \ge S(\cZ)$. Moreover, $S(\cZ) = 0$ if and only if $\gamma_1(\cZ) = \gamma_{2,T}(\cZ) = 0$. 
\end{lemma}
Recall that $S(\cZ) = \int_{0}^1 \bbE\norm{Z_1 - Z_0 - v_t(Z_t)}_2^2\; dt$. Also, note that $Z_1 - Z_0 = \int_{0}^1 v_u(Z_u)\; du$. Therefore, we have 
\begin{align*}
    S(\cZ) & = \int_{0}^1 \bbE \norm{ \int_{0}^1 [v_u(Z_u) - v_t(Z_t)] \; du }_2^2\; dt\\
    & = \int_{0}^1 \bbE \norm{ \int_{0}^1 \int_t^u \dot v_\tau(Z_\tau) \;d\tau \; du }_2^2\; dt \\
    & \le \int_{0}^1 \bbE \left[\int_0^1 \abs{t - u} \int_{t \wedge u}^{t \vee u} \norm{\dot v_\tau(Z_\tau)}_2^2 \; d\tau \; du \right] \; dt\\
    & \le \int_0^1 \bbE \int_0^1 \int_0^1 \norm{\dot v_\tau(Z_\tau)}_2^2 \; d\tau \; du\\
    & \le \int_0^1 \bbE \norm{\dot v_\tau(Z_\tau)}_2^2 \; d\tau = \gamma_1(\cZ).
\end{align*}

Moreover, note that 
\begin{equation}
\label{eq: gamma inequality}
\gamma_1(\cZ) = \sum_{i=1}^T (t_i - t_{i-1}) . \frac{1}{t_i - t_{i-1}}\int_{t_{i-1}}^{t_i} \bbE \norm{\dot v_\tau(Z_\tau)}_2^2 \; d\tau \le \gamma_{2, T}(\cZ). 
\end{equation}
This shows that $S(\cZ) \le \gamma_1(\cZ) \le \gamma_{2,T}(\cZ)$.

For the second part, first note that the $t_{i} - t_{i-1} = 1/T$. Therefore, 
\[
\gamma_1(\cZ) = \frac{1}{T} \sum_{i=1}^T \frac{1}{t_i - t_{i-1}}\int_{t_{i-1}}^{t_i} \bbE \norm{\dot v_\tau(Z_\tau)}_2^2 \; d\tau \ge \frac{\gamma_{2,T}(\cZ)}{T}.
\]
The above inequality along with \eqref{eq: gamma inequality} tells that $\gamma_1(\cZ) = 0$ iff $\gamma_{2,T}(\cZ) = 0$. 

Now, due to the inequality $S(\cZ) \le \gamma_1(\cZ)$, we have $S(\cZ) = 0$ if $\gamma_1(\cZ) = 0$. For the other direction, let us assume $S(\cZ) = 0$. This shows that $v_t(Z_t) = Z_1 - Z_0$ almost surely in $t$ and $(Z_0, Z_1)$. This shows that $\dot v_t(Z_t) = 0$ almost surely. Hence the result follows.

\subsection{Proof of Theorem \ref{thm: W2 ode disc}}
\label{sec: proof W2 ode disc}
Recall that for a given partition $0 = t_0 < t_1 < \ldots < t_T = 1$ of the interval $[0,1]$ of equidistant points $\{t_i\}_{0\le i \le T}$ with $h := T^{-1}$, we follow the Euler discretized version of the of ODE \eqref{eq: ode-approx} to obtain the sample estimates:
\[
\widehat{Y}_{t_i} = \widehat{Y}_{t_{i-1}} +  h \widehat{v}_{t_i}(\widehat{Y}_{t_i}), \quad \widehat{Y}_0 = Z_0. 
\]


Before analyzing the discretization error, we introduce the following interpolation process for $t \in [t_i, t_{i+1}]$ and each $i \in \{0, \ldots, T\}$:
\begin{equation}
    \label{eq: interpolation proc}
    \ddt\bar{Y}_t = \widehat{v}_{t_i}(\bar{Y}_{t_i}) ,\quad \bar{Y}_{t_i} = \widehat{Y}_{t_i}.
\end{equation}
The above ODE flow gives us a continuous interpolation between $\widehat{Y}_{t_i}$ and $\widehat{Y}_{t_{i+1}}$. Coupled with the above flow equation and the ODE flow \eqref{eq: emp-ode-disc}, we have the following \textit{almost sure} differential inequality for $ t \in [t_i, t_{i+1}]$:

\begin{equation}
\label{eq: ddt decomp}
    \begin{aligned}
        \ddt\Norm{Z_t - \bar{Y}_t}_2^2 & = 2 \innerprod{Z_t - \bar{Y}_t, \ddt Z_t - \ddt\bar{Y}_t}\\
        & = 2 \innerprod{Z_t - \bar{Y}_t, v_t(Z_t) - \widehat{v}_{t_i}(\bar{Y}_{t_i})}\\
        & \le \widehat{L}\Norm{Z_t - \bar{Y}_t}_2^2 + \Norm{v_t(Z_t) - \widehat{v}_{t_i}(\bar{Y}_{t_i})}_2^2/\widehat{L}
    \end{aligned}
\end{equation}
Multiplying $e^{- \widehat{L}(t-t_i)}$ on both sides of the above inequality and rearranging the terms leads to 

\begin{equation*}
\scriptstyle
    \begin{aligned}
        & e^{-\widehat{L}(t-t_i)}\ddt\Norm{Z_t - \bar{Y}_t}_2^2  -  e^{-\widehat{L}(t-t_i)} \widehat{L}\Norm{Z_t - \bar{Y}_t}_2^2  \le e^{-\widehat{L}(t-t_i)}\Norm{v_t(Z_t) - \widehat{v}_{t_i}(\bar{Y}_{t_i})}_2^2/\widehat{L}.\\ 
        & \Leftrightarrow \ddt \{ e^{-\widehat{L}(t-t_i)}\Norm{Z_t - \bar{Y}_t}_2^2\} \le e^{-\widehat{L}(t-t_i)}\Norm{v_t(Z_t) - \widehat{v}_{t_i}(\bar{Y}_{t_i})}_2^2/\widehat{L}\le \Norm{v_t(Z_t) - \widehat{v}_{t_i}(\bar{Y}_{t_i})}_2^2/\widehat{L}.\\
        & \Leftrightarrow \Norm{Z_{t_{i+1}} - \widehat{Y}_{t_{i+1}}}_2^2 \le e^{\widehat{L}(t_{i+1} - t_i)}\Norm{Z_{t_i} - \widehat{Y}_{t_i}}_2^2  + \frac{e^{\widehat{L}(t_{i+1} - t_i)}}{\widehat{L}} \int_{t_i}^{t_{i+1}} \Norm{v_t(Z_t) - \widehat{v}_{t_i}(\bar{Y}_{t_i})}_2^2\; dt.
    \end{aligned}
\end{equation*}
Define $\Delta_i := \bbE \Norm{Z_{t_{i}} - \widehat{Y}_{t_{i}}}_2^2$. Using the above inequality we have 
\begin{equation}
\label{eq: moment recursion}
\begin{aligned}
& \Delta_{i+1}\\
&\le e^{\widehat{L}h} \Delta_i + \frac{e^{\widehat{L}h}}{\widehat{L}} \int_{t_i}^{t_{i+1}}\bbE \Norm{v_t(Z_t) - \widehat{v}_{t_i}(\widehat{Y}_{t_i})}_2^2 \; dt\\
& \le e^{\widehat{L}h} \Delta_i\\
&  +  \frac{3 e^{\widehat{L}h}}{\widehat{L}} \left\{
\underbrace{\int_{t_i}^{t_{i+1}}\bbE \Norm{v_t(Z_t) - v_{t_i}(Z_{t_i})}_2^2 \; dt}_{T_1} + \underbrace{\int_{t_i}^{t_{i+1}}\bbE \Norm{v_{t_i}(Z_{t_i}) - \widehat{v}_{t_i}(Z_{t_i})}_2^2 \; dt}_{T_2}+ \underbrace{\int_{t_i}^{t_{i+1}}\bbE \Norm{\widehat{v}_{t_i}(Z_{t_i}) -\widehat{v}_{t_i}(\widehat{Y}_{t_i})}_2^2 \; dt}_{T_3} \right\}.
\end{aligned}
\end{equation}
Now we will bound each of the last three terms on the right-hand side of the above inequality. 

\textbf{Bounding $T_1$.}
For the first term, we have 

\begin{equation}
    \label{eq: straightness bound}
    \begin{aligned}
        \bbE \Norm{v_t(Z_t) - v_{t_i}(Z_{t_i})}_2^2 & = \bbE \norm{\int_{t_i}^{t}\ddtau v_\tau(Z_\tau) \; d\tau}_2^2\\
        & \le (t - t_i) \int_{t_i}^{t} \bbE \norm{\ddtau v_\tau(Z_\tau)}_2^2 \; d\tau \quad 
        \\
        & \le h^2 \gamma_i,
    \end{aligned}
\end{equation}
where $ \gamma_i = \frac{1}{t_{i+1} - t_i} \int_{t_i}^{t_{i+1}}  \bbE \Norm{\ddtau v_\tau(Z_\tau)}_2^2 \; d\tau$. 
This shows that $T_1 \le h^3 \gamma_i$.

\textbf{Bounding $T_2$.} The term $T_2$ is bounded by $h \sqevel$ as $\bbE \Norm{v_{t_i}(Z_{t_i}) - \widehat{v}_{t_i}(Z_{t_i})}_2^2 \le \sqevel$ (Assumption \ref{assmp: main assumption}\ref{assmp: estimation error}).

\textbf{Bounding $T_3$.} For the final term we will use that $\widehat{v}_{t_i}$ is $\widehat{L}$-Lipschitz. This entails that $T_3 \le \widehat{L}^2h \Delta_i$. Plugging these bounds in the recursion formula \eqref{eq: moment recursion}, we get 
\[
\Delta_{i+1} \le e^{\widehat{L}h}(1 + 3 \widehat{L} h) \Delta_i + 3e^{\widehat{L}h} (h^3 \gamma_i + h \sqevel)/\widehat{L}.
\]
Solving the recursion yields
\begin{align*}
\Delta_T & \le e^{T \widehat{L} h}(1 + 3 \widehat{L} h)^T \Delta_0 +  \frac{3 h^3}{\widehat{L}} \left\{\sum_{k = 1}^{T} e^{k \widehat{L} h}(1+3 \widehat{L} h)^{k-1} \gamma_{T-k}\right\} + \frac{3h}{\widehat{L}} \left\{\sum_{k = 1}^{T} e^{k \widehat{L} h}(1+3 \widehat{L} h)^{k-1}\right\}\sqevel.
\end{align*}
Recall that $\gamma_{2,T}(\cZ):= \max_{k} \gamma_k$. Note that $\Delta_0 = 0$ as $Z_0 = \widehat{Y}_0$. Therefore, we have 
\[
\Delta_T \le  \frac{ e^{4 \widehat{L}}}{\widehat{L}^2} \left(\frac{\gamma_{2,T}(\cZ)}{T^2} + \sqevel \right). 
\]


Here we used the fact that 
\[
\sum_{k = 1}^{T} e^{k \widehat{L} h}(1+3 \widehat{L} h)^{k-1} \le  \frac{e^{4 \widehat{L}}-1}{1 + 3\widehat{L} h - e^{- \widehat{L} h}} \le \frac{e^{4 \widehat{L}}}{3 \widehat{L}h }.
\]

Therefore, we have $$W_2^2(\widehat{\rho}_1, \rho_1) \le \Delta_T \le \frac{ e^{4 \widehat{L}} }{\widehat{L}^2} \left(\frac{\gamma_{2,T}(\cZ)}{T^2} + \sqevel \right).$$

However, the above upper bound explodes for $\widehat{L} \to 0$. Therefore, we handle the case $\widehat{L} <1$ in a slightly different manner.

\textbf{Separately handling $\widehat{L}<1$ case:} We recall the decomposition \eqref{eq: ddt decomp}. We will only change the last inequality in that decomposition, i.e., for $\alpha >0$ we get
\begin{equation}
\label{eq: ddt decomp2}
    \begin{aligned}
        \ddt\Norm{Z_t - \bar{Y}_t}_2^2 & = 2 \innerprod{Z_t - \bar{Y}_t, \ddt Z_t - \ddt\bar{Y}_t}\\
        & = 2 \innerprod{Z_t - \bar{Y}_t, v_t(Z_t) - \widehat{v}_{t_i}(\bar{Y}_{t_i})}\\
        & \le \alpha \Norm{Z_t - \bar{Y}_t}_2^2 + \Norm{v_t(Z_t) - \widehat{v}_{t_i}(\bar{Y}_{t_i})}_2^2/\alpha
    \end{aligned}
\end{equation}
Therefore, following exactly similar steps as before, we arrive at the following recursion:
\[
\Delta_{i+1} \le e^{\alpha h} \left(1 + \frac{3\widehat{L}^2 h}{\alpha}\right) \Delta_i + \frac{3 e^{\alpha h}}{\alpha} (h^3 \gamma_i + h \sqevel).
\]
Solving this yields
\[
\Delta_T \le \frac{e^{\alpha + 3\widehat{L}^2/\alpha} -1}{1 + 3 \widehat{L}^2 h/\alpha -e^{-\alpha h}} \left(\frac{3h^3}{\alpha}. \gamma_{2,T}(\cZ)  + \frac{3 h}{\alpha}. \sqevel\right)
\]
Note that $e^{\alpha + 3\widehat{L}^2/\alpha} -1 \le e^{\alpha + 3\widehat{L}/\alpha} -1$ as $\widehat{L} <1$. Additionally, 
\[
1 + 3 \widehat{L}^2 h/\alpha -e^{-\alpha h} \ge 1 - e^{-\alpha h} \ge \alpha h e^{-\alpha h}.
\]
Setting $\alpha = 1$, and using the above inequalities along with the fact that $h \le 1$, we get 
\[
 \frac{e^{\alpha + 3\widehat{L}^2/\alpha} -1}{1 + 3 \widehat{L}^2 h/\alpha -e^{-\alpha h}} \le \frac{e^{2+4\widehat{L}}}{h}. 
\]

Finally, using the above inequality we have 
\[
W_2^2(\widehat{\rho}_1, \rho_1) \le \Delta_T \le 27 e^{4 \widehat{L}}\left(\frac{\gamma_{2,T}(\cZ)}{T^2} + \sqevel \right).
\]
Combining this with previous upper bound we finally get the result.


\section{{Proofs for bounding $\gamma_{2,T}$}}

\paragraph{Definitions}:
Let us define the key terms used in this derivation:
\begin{itemize}
    \item \( z_t \): The state (e.g., data point) at time \( t \).
    \item \( p_t(z_t) \): The probability density function of $X_t = (1-t)X_0 + t X_1$.
    \item \( v_t(z_t) \): The velocity field that transports the probability density.
    \item \( s_t(z_t) = \nabla_{z_t} \log p_t(z_t) \): The score function, which is the gradient of the log-probability density.
    \item \( H_t(z_t) = \nabla_{z_t}^2 \log p_t(z_t) \): The Hessian of the log-probability density.
    \item \( Y = \nabla_{z_{t}}\text{div}(v_{t}(z_{t})) \): The gradient of the divergence of the velocity field.
\end{itemize}
The velocity field is defined as:
\begin{equation}
    v_t(z_t) = \frac{z_t}{t} + \left(\frac{1-t}{t}\right)s_t(z_t)
\end{equation}

\subsection{General expression for $\dot v_t(z_t)$(Lemma \ref{lem: acceleration-rf})}
\label{sec: general form v_dot_t}

\paragraph{Derivation of $\dot v_t(z_t)$}:

We begin with the material derivative (or total derivative) of the velocity field \( v_t(z_t) \), denoted by \( \dot{v}_t(z_t) \).
\begin{equation}
    \frac{d}{dt}v_{t}(z_{t}) = \nabla_{z_t} v_t(z_t) \cdot v_t(z_t) + \frac{\partial v_t(z_t)}{\partial t}
\end{equation}
Expanding the partial derivative term using the product rule:
\begin{align*}
    \frac{\partial v_t}{\partial t} &= \frac{\partial}{\partial t}\left(\frac{z_{t}}{t}+\left(\frac{1-t}{t}\right)s_{t}(z_{t})\right) \\
    &= -\frac{z_t}{t^2} + \frac{\partial}{\partial t}\left(\left(\frac{1}{t}-1\right)s_t(z_t)\right) \\
    &= -\frac{z_t}{t^2} - \frac{1}{t^2}s_t(z_t) + \left(\frac{1-t}{t}\right)\frac{\partial s_t(z_t)}{\partial t}
\end{align*}
Substituting this back gives the full expression for the material derivative:
\begin{equation}
    \dot{v}_{t}(z_{t}) = \nabla_{z_{t}}v_{t}(z_{t})\cdot v_{t}(z_{t}) - \frac{z_{t}}{t^{2}} - \frac{1}{t^{2}}s_{t}(z_{t}) + \left(\frac{1-t}{t}\right)\frac{\partial s_{t}(z_{t})}{\partial t}
\end{equation}

Next, we find an expression for \( \frac{\partial s_t}{\partial t} \). We use the continuity equation, \( \frac{\partial p_t}{\partial t} = -\text{div}(p_t v_t) \).
\begin{align*}
    \frac{\partial s_t}{\partial t} &= \frac{\partial}{\partial t} \nabla_{z_t} \log p_t = \nabla_{z_t} \left( \frac{\partial}{\partial t} \log p_t \right) = \nabla_{z_t} \left( \frac{1}{p_t} \frac{\partial p_t}{\partial t} \right) \\
    &= \nabla_{z_t} \left( -\frac{1}{p_t} \text{div}(p_t v_t) \right) \\
    &= \nabla_{z_t} \left( -\frac{1}{p_t} (\nabla_{z_t} p_t \cdot v_t + p_t \text{div}(v_t)) \right) \\
    &= \nabla_{z_t} \left( -(\frac{\nabla_{z_t} p_t}{p_t} \cdot v_t + \text{div}(v_t)) \right) \\
    &= \nabla_{z_t} [-s_t \cdot v_t - \text{div}(v_t)] \\
    &= -\{H_t \cdot v_t + \nabla_{z_t}v_t \cdot s_t + Y_t\},
\end{align*}
where $Y_t(z_t) =\nabla_{z_{t}}\text{div}(v_t(z_{t}))$.
Now, we substitute the expression for \( \frac{\partial s_t}{\partial t} \) into the equation for \( \dot{v}_t \). The derivation follows the specific intermediate steps and cancellations from your notes.
\begin{align*}
    \dot{v}_{t} = \nabla_{z_{t}}v_{t} \cdot v_{t} - \frac{z_{t}}{t^{2}} - \frac{s_{t}}{t^{2}} - \left(\frac{1-t}{t}\right)\{H_{t} \cdot v_{t} + \nabla_{z_{t}}v_{t} \cdot s_{t} + Y_t\}
\end{align*}
Now substitute $s_t(z_t) = \frac{t v_{t}(z_{t})-z_{t}}{1-t}$
\begin{align*}
    = \cancel{\nabla_{z_{t}}v_{t}(z_{t})\cdot v_{t}(z_{t})} &-\frac{z_{t}}{t^{2}}-\left(\frac{1-t}{t}\right)\left\{H_{t}(z_{t})\cdot v_{t}(z_{t}) + \nabla_{z_t}v_t(z_t) \left(\frac{\cancel{t v_{t}(z_{t})}-z_{t}}{1-t}\right)+Y_t(z_t)\right\} - \frac{1}{t^{2}}s_{t}(z_{t})
\end{align*}
Now substitute $\nabla_{z_{t}}v_{t}(z_{t}) = \frac{I+(1-t)H_{t}(z_{t})}{t}$.
\begin{align*}
    \dot{v}_{t}(z_t) &= \cancel{-\frac{z_{t}}{t^{2}}}-\left(\frac{1-t}{t}\right)\left\{H_{t}(z_{t})\left(\cancel{\frac{z_t}{t}}+\left(\frac{1-t}{t}\right)s_{t}(z_{t})\right) - \left(\frac{\cancel{I}+\cancel{(1-t)H_{t}(z_{t})}}{t}\right)\frac{z_t}{1-t} + Y_t(z_t)\right\} - \frac{s_t(z_t)}{t^2} \\
    &= -\left(\frac{1-t}{t}\right)^{2}H_{t}(z_{t})s_{t}(z_{t}) - \left(\frac{1-t}{t}\right)Y_t(z_t) - \frac{1}{t^{2}}s_{t}(z_{t})
\end{align*}
Now, we find an expression for \(Y_t(z_t)\) and substitute it.
\begin{align*}
    Y_t(z_t) &= \nabla_{z_{t}}\text{div}(v(z_{t},t)) = \nabla_{z_{t}}\left(\text{Tr}(\nabla_{z_{t}}v(z_{t},t))\right) \\
    &= \nabla_{z_{t}}\text{Tr}\left(\frac{I}{t}+\left(\frac{1-t}{t}\right)H_{t}(z_{t})\right) \\
    &= \left(\frac{1-t}{t}\right)\nabla_{z_{t}}\text{Tr}(H_{t}(z_{t})) \quad (\text{since } \nabla \text{Tr}(I) = 0)
\end{align*}
Substituting this into the simplified expression for \( \dot{v}_t \) and rearranging gives the final result:
\begin{equation}
    \dot{v}_{t}(z_{t})=-\frac{1}{t^{2}}\left\{(1-t)^{2}[H_{t}(z_{t})s_{t}(z_{t})+\nabla_{z_{t}}\text{Tr}(H_{t}(z_{t}))] + s_t(z_t)\right\}
\end{equation}

\subsection{Expression for $\dot v_t(z_t)$ for a target mixture of Gaussian distributions}
\label{sec: general form of v_dot}
Let $p_1(x) = \sum_{i=1}^{K} \pi_i\; p_{i, 1}(x)$, where $p_{i, 1}(x) = N(x \mid \mu_i, \sigma^2 I_d)$ be the target mixture of $K$ Gaussians. Note that the density of $X_t = tX_1 + (1-t)X_0$ is another $K$-Gaussian mixture given by $p_1(x) = \sum_{i=1}^{K} \pi_i\;  p_{i, t}(x)$, where $p_{i, t}(x) = N(x \mid t\mu_i, \sigma_t^2 I_d)$ and $\sigma_t^2=t^2\sigma^2 + (1-t)^2.$
\paragraph{Derivation of the Score \(s_t(z)\)}:
The score of the mixture is the gradient of its log-density.
\begin{align*}
    s_t(z) &= \nabla_z \log p_t(z) \\
    &= \frac{\nabla_z p_t(z)}{p_t(z)} && \text{(Chain rule for log)} \\
    &= \frac{\nabla_z \left( \sum_{i=1}^K \pi_i p_{i,t}(z)\right)}{p_t(z)} && \text{(Substitute definition of \(p_t\))} \\
    &= \frac{\sum_{i=1}^K \pi_i \nabla_z p_{i,t}(z)}{p_t(z)} && \text{(Linearity of gradient)}
\end{align*}
We use the identity \( \nabla_z p_{i,t}(z) = p_{i,t}(z) \nabla_z \log p_{i,t}(z) = p_{i,t}(z) s_{i,t}(z) \).
\begin{align*}
    s_t(z) &= \frac{ \sum_{i=1}^K \pi_i p_{i,t}(z) s_{i,t}(z)}{p_t(z)} \\
    &= \sum_{i=1}^K \frac{\pi_i p_{i,t}(z)}{ \cdot p_t(z)} s_{i,t}(z) && \text{(Rearrange terms)}
\end{align*}
Recognizing the definition of the weight, \( w_{i,t}(z) = \frac{\pi_i p_{i,t}(z)}{\cdot p_t(z)} \), we arrive at the final expression:
\begin{equation}
    s_t(z) = \sum_{i=1}^K w_{i,t}(z) s_{i,t}(z)
\end{equation}
This shows the score of the mixture is the weighted average of the scores of its components.

\paragraph{Derivation of the Hessian \(H_t(z)\)}:
The Hessian is the Jacobian of the score vector, \( H_t(z) = \nabla_z s_t(z)^T \). We differentiate the expression for \(s_t(z)\) using the product rule.
\begin{align*}
    H_t(z) &= \nabla_z \left( \sum_{i=1}^K w_{i,t}(z) s_{i,t}(z) \right)^T \\
    &= \sum_{i=1}^K \nabla_z (w_{i,t}(z) s_{i,t}(z))^T \\
    &= \sum_{i=1}^K \left( (\nabla_z w_{i,t}(z)) s_{i,t}(z)^T + w_{i,t}(z) (\nabla_z s_{i,t}(z)^T) \right)
\end{align*}
First, we find the gradient of the weight \(w_{i,t}\).
\[ \nabla_z w_{i,t} = \nabla_z \left( \frac{p_{i,t}}{\sum_j p_{j,t}} \right) = \frac{(\nabla_z p_{i,t})(\sum_j p_{j,t}) - p_{i,t}(\sum_j \nabla_z p_{j,t})}{(\sum_j p_{j,t})^2} \]
Substituting \( \nabla_z p_{i,t} = p_{i,t} s_{i,t} \) and dividing by \( (\sum_j p_{j,t})^2 \), we get:
\[ \nabla_z w_{i,t} = \frac{p_{i,t} s_{i,t}}{\sum_j p_{j,t}} - \frac{p_{i,t}}{(\sum_j p_{j,t})^2} \sum_j p_{j,t} s_{j,t} = w_{i,t} s_{i,t} - w_{i,t} s_t = w_{i,t}(s_{i,t} - s_t) \]
Now, substitute this back into the expression for the Hessian. Note that \( \nabla_z s_{i,t}(z)^T = H_{i,t}(z) = -I_d/\sigma_t^2 \).
\begin{align*}
    H_t(z) &= \sum_{i=1}^K \left( w_{i,t}(s_{i,t} - s_t)s_{i,t}^T + w_{i,t} H_{i,t} \right) \\
    &= \left( \sum_i w_{i,t} H_{i,t} \right) + \left( \sum_i w_{i,t} s_{i,t}s_{i,t}^T \right) - s_t \left( \sum_i w_{i,t} s_{i,t}^T \right) \\
    &= \left( \sum_i w_{i,t} (-\frac{I_d}{\sigma_t^2}) \right) + \mathbb{E}_{w_t}[s_{\cdot,t}s_{\cdot,t}^T] - s_t s_t^T \\
    &= -\frac{I_d}{\sigma_t^2} + \text{Cov}_{w_t}(s_{\cdot,t})(z)
\end{align*}
where $\text{Cov}_{w_t}(s_{\cdot,t})(z) := \mathbb{E}_{w_t}[s_{\cdot,t}s_{\cdot,t}^T] - s_t s_t^T$.

\paragraph{Derivation of \(\nabla_{z}\text{Tr}(H_t(z)) \)}:
We start with the trace of the Hessian.
\[ \text{Tr}(H_t(z)) = \text{Tr}\left(-\frac{I_d}{\sigma_t^2}\right) + \text{Tr}(\text{Cov}_{w_t}(s_{\cdot,t})) = -\frac{d}{\sigma_t^2} + \mathbb{E}_{w_t}[\|s_{\cdot,t}\|^2] - \|s_t\|^2 \]
Now we take the gradient of this expression with respect to \(z\).
\[ \nabla_{z}\text{Tr}(H_t(z)) = \nabla_z \left( -\frac{d}{\sigma_t^2} + \sum_{i=1}^K w_{i,t}(z)\|s_{i,t}(z)\|^2 - \|s_t(z)\|^2 \right) \]
The first term is constant and its gradient is zero. We apply the product rule to the second term and the chain rule to the third.
\begin{align*}
    \nabla_{z}\text{Tr}(H_t(z)) &= \sum_{i=1}^K \left( (\nabla_z w_{i,t})\|s_{i,t}\|^2 + w_{i,t}(\nabla_z\|s_{i,t}\|^2) \right) - \nabla_z\|s_t\|^2 \\
    &= \sum_{i=1}^K \left( w_{i,t}(s_{i,t}-s_t)\|s_{i,t}\|^2 + w_{i,t}\left(\frac{-2s_{i,t}}{\sigma_t^2}\right) \right) - 2H_t s_t \\
    &= \underbrace{\sum_i w_{i,t} s_{i,t}\|s_{i,t}\|^2}_{\mathbb{E}_{w_t}[s_{\cdot,t}\|s_{\cdot,t}\|^2]} - \underbrace{s_t\left(\sum_i w_{i,t} \|s_{i,t}\|^2\right)}_{\text{This is } s_t \mathbb{E}_{w_t}[\|s_{\cdot,t}\|^2]} - \underbrace{\frac{2}{\sigma_t^2}\sum_i w_{i,t} s_{i,t}}_{\frac{2}{\sigma_t^2}s_t} - 2H_t s_t
\end{align*}
Grouping the terms gives the final expression:
\begin{equation}
    \nabla_{z}\text{Tr}(H_t(z)) = \mathbb{E}_{w_t}[s_{\cdot,t} \|s_{\cdot,t}\|^2] - s_t\mathbb{E}_{w_t}[\|s_{\cdot,t}\|^2] - \frac{2s_t}{\sigma_t^2} - 2H_ts_t
\end{equation}

\paragraph{Final Simplified Expression for \( \dot{v}_t(z_t) \)}:
We now substitute the derived expression for \(Y(z)\) into the general formula for \( \dot{v}_t(z_t) \).
The general formula is:
\[ \dot{v}_{t}(z_{t})=-\frac{1}{t^{2}}\left\{(1-t)^{2}[H_{t}(z)s_{t}(z)+Y(z)] + s_t(z)\right\} \]
Let's first simplify the term inside the square brackets, \( H_t s_t + Y \).
\begin{align*}
    H_t s_t + Y &= H_t s_t + \left( \mathbb{E}_{w_t}[s_{\cdot,t} \|s_{\cdot,t}\|^2] - s_t\mathbb{E}_{w_t}[\|s_{\cdot,t}\|^2] - \frac{2s_t}{\sigma_t^2} - 2H_t s_t \right) \\
    &= \mathbb{E}_{w_t}[s_{\cdot,t} \|s_{\cdot,t}\|^2] - s_t\mathbb{E}_{w_t}[\|s_{\cdot,t}\|^2] - \frac{2s_t}{\sigma_t^2} - H_t s_t
\end{align*}
This expression can be made more compact by defining the covariance between the random vector \(s_{\cdot,t}\) and the random scalar \(\|s_{\cdot,t}\|^2\) under the discrete distribution \(w_t\):
\[ \text{Cov}_{w_t}(s_{\cdot,t}, \|s_{\cdot,t}\|^2) = \mathbb{E}_{w_t}[s_{\cdot,t} \|s_{\cdot,t}\|^2] - \mathbb{E}_{w_t}[s_{\cdot,t}]\mathbb{E}_{w_t}[\|s_{\cdot,t}\|^2] = \mathbb{E}_{w_t}[s_{\cdot,t} \|s_{\cdot,t}\|^2] - s_t \mathbb{E}_{w_t}[\|s_{\cdot,t}\|^2] \]
This term is a vector that captures the asymmetry (or skewness) in the distribution of component scores. So, the simplified term is:
\[ H_t s_t + Y = \text{Cov}_{w_t}(s_{\cdot,t}, \|s_{\cdot,t}\|^2) - \frac{2s_t}{\sigma_t^2} - H_t s_t \]
Substituting this back into the formula for \( \dot{v}_t \):
\begin{equation}
\label{eq: final form v_dot_t}
    \dot{v}_{t}(z_{t})=-\frac{\sigma^2}{\sigma_t^{2}}s_t -\frac{(1-t)^{2}}{t^2}\left[\text{Cov}_{w_t}(s_{\cdot,t}, \|s_{\cdot,t}\|^2) - \text{Cov}_{w_t}(s_{\cdot,t})s_t\right]
\end{equation}
This is the final, simplified analytical expression for the time evolution of the velocity field for a Gaussian Mixture Model. It reveals that the dynamics are driven by the score (mean field), the Hessian-score product (curvature), and a third-order term related to the asymmetry of the mixture components.

\subsection{Proof of Lemma \ref{lemma: bound on gamma_2}}
\label{sec: bound on v_dot_t for GMM}
For ease of notation, we begin by deriving the expected square norm of the two covariance terms in Equation \eqref{eq: final form v_dot_t} for general Gaussian mixtures and then adapt it to the interpolating structure of RF.



\begin{lemma}[Bound on the Expected Squared Norm of the Hessian]
\label{lemma: bound on Cov}
Let $p(x)$ be a mixture of $K$ Gaussians, given by $p(x) = \sum_{k=1}^K \pi_k\; p_k(x)$, with $p_k(x) = \mathcal{N}(x\mid \mu_k, \Sigma)$, $\Sigma = \sigma^2 I_d$ and $D:= \max_{i, j} \norm{\mu_i - \mu_j}_2$. Let $s_k(x) = \Sigma\inv (\mu_k - x)$ be the score and $\gamma_k(x) = \pi_k p_k(x)/p(x)$ be the posterior weight of the $k^{th}$ mixture component. Then:
$$
 \Ee{x}{\norm{\text{Cov}_{\gamma(x)}(\{s_k(x)\})}_2^2} \le  \frac{D^4}{\sigma^8}.
$$
\end{lemma}
where $\text{Cov}_{\gamma(x)}(\{s_k(x)\}) := \sum_{k=1}^K \gamma_k(x) s_k(x)s_k(x)^\top - s(x)s(x)^\top$.
\begin{proof}
By the property of covariance, we have $\text{Cov}_{\gamma(x)}(\{s_k\}) = \Sigma^{-1} \text{Cov}_{\gamma(x)}(\{\mu_k\}) \Sigma^{-1}$. We recall that $\max_{i \ne j} \norm{\mu_1 - \mu_j}_2 \le D$. Therefore, we have 
\begin{align*}
    \norm{\text{Cov}_{\gamma(x)}(\{\mu_k\})}_2 & \le \sum_{k} \gamma_k(x) \norm{\mu_k - m_{\gamma(x)}(\{\mu_k\})}_2^2 \quad \text{ where } m_{\gamma(x)}(\{\mu_k\}):= \sum_{k}\gamma_k(x) \mu_k\\
    & \le D^2.
\end{align*}
The above bound yields $\norm{\text{Cov}_{\gamma(x)}(\{s_k\})}_2 \le D^2 \norm{\Sigma^{-1}}_2^2 = \frac{D^2}{\sigma^4}$, and which gives:

$$
 \E{\norm{\text{Cov}(\{s_k\})}_2^2} \le  \frac{D^4}{\sigma^8}.
$$
\end{proof}

\begin{lemma}[$D$-based Bound on a Third-Order Covariance Vector]
\label{lemma: bound on C}
Under the same GMM assumptions as Lemma 1, let $\mathbf{C}(x) = \text{Cov}_{\gamma(x)}(\|s_.(x)\|^2, s_.(x))$, where $s_k(x) = \Sigma^{-1}(x-\mu_k)$. Let $D = \max_{i,j} \|\mu_i - \mu_j\|$. The expected squared norm of $\mathbf{C}(x)$ is bounded by:
$$
\mathbb{E}\left[\|\mathbf{C}(x)\|_2^2\right] \le \frac{D^2 d(d+2)}{\sigma^8}
$$
\end{lemma}

\begin{proof}
From its definition, the norm of $\mathbf{C}(x)$ can be bounded as:
$$
\|\mathbf{C}(x)\|_2 = \left\| \sum_{i,j}\gamma_{i}\gamma_{j}\|s_{i}\|^{2}(s_{i}-s_{j}) \right\|_2 \le \sum_{i,j}\gamma_{i}\gamma_{j}\|s_{i}\|^{2}\|s_{i}-s_{j}\|_2.
$$
We introduce a uniform bound on the pairwise distance term:
$$
\|s_i - s_j\|_2 = \|\Sigma^{-1}(\mu_j - \mu_i)\|_2 \le \|\Sigma^{-1}\|_2 \|\mu_j - \mu_i\|_2 \le \|\Sigma^{-1}\|_2 D
$$
Substituting this into the sum and simplifying, using $\sum_j \gamma_j = 1$:
$$
\|\mathbf{C}(x)\|_2 \le \sum_{i,j}\gamma_{i}\gamma_{j}\|s_{i}\|^{2} (\|\Sigma^{-1}\|_2 D) = (\|\Sigma^{-1}\|_2 D) \left( \sum_i \gamma_i \|s_i\|^2 \right)
$$
Squaring both sides and taking the expectation gives:
$$
\mathbb{E}[\|\mathbf{C}(x)\|_2^2] \le (\|\Sigma^{-1}\|_2 D)^2 \cdot \mathbb{E}\left[ \left( \sum_i \gamma_i(x) \|s_i(x)\|^2 \right)^2 \right]
$$
We bound the expectation term using Jensen's inequality:
$$
\mathbb{E}\left[ \left( \sum_i \gamma_i \|s_i\|^2 \right)^2 \right] \le \mathbb{E}\left[ \sum_i \gamma_i \|s_i\|^4 \right] = \sum_i \pi_i \mathbb{E}_{x \sim \mathcal{N}_i}[\|s_i(x)\|^4]
$$
For the isotropic case, this sum evaluates to $\frac{d(d+2)}{\sigma^4}$. We also have $\|\Sigma^{-1}\|_2 = 1/\sigma^2$. Substituting these into the main inequality yields the final result:
$$
\mathbb{E}[\|\mathbf{C}(x)\|_2^2] \le \left(\frac{1}{\sigma^2} D\right)^2 \cdot \left( \frac{d(d+2)}{\sigma^4} \right) = \frac{D^2 d(d+2)}{\sigma^8}
$$
\end{proof}

\paragraph{Final bound on $\gamma_{2,T}$}:
By \citet[Lemma F.3]{gupta2024improvedsamplecomplexitybounds}, we know that $\bbE \norm{s_t}_2^2 = O(d)$, as $p_t$ is mixture of Gaussian distribution (sub-Gaussian).  We note that the maximum inter-mean distance at time $t$ is $t D$, and the variance of each component at time $t$ is $\sigma_t^2 = t^2 \sigma^2 + (1-t)^2$. Recall that from the proof of Lemma \ref{lemma: bound on Cov}, we have $\norm{\text{Cov}_{\gamma(x)}(\{s_k\})}_2 \le \frac{D^2}{\sigma^4}$. If we apply Lemma \ref{lemma: bound on Cov} and Lemma \ref{lemma: bound on C} on \eqref{eq: final form v_dot_t}, then we get 
\[
\bbE \norm{\dot v_t(z_t)}_2^2   \le C_\sigma \left(d + D^2 d^2 + D^4 d \right),
\]
where $C_\sigma>0$ is positive constant depending on $\sigma$.

\paragraph{Explicit bound on Wasserstein convergence rate}

\begin{lemma}
\label{lemma: bound on L for GMM}
    Let the target distribution $\rho_1 = \sum_{k \in [K]} \pi_k N(\mu_k, \sigma^2 I_d)$, and let $ \max_{i \ne j} \Vert \mu_i - \mu_j\Vert_2 \le D$. Consider the ODE flow \eqref{eq: emp-ode-disc} with $T$ discretization steps and the true RF drift field given in Lemma \ref{lem: drift-gmm}. Then, the Lipschitz constant of the RF drift is $\frac{(1+\sigma^2)^2}{\sigma^2} \br{1 + \frac{D^2}{2 \sigma^2}}$ and hence, the squared-Wasserstein convergence error scales as $W_2^2(\widehat \rho_1, \rho_1) = \cO(d^2/T^2)$.
\end{lemma}

\begin{proof}
To provide bound on the Lipschitz constant, we will simply obtain upper bound on the operator norm of $A_t := \nabla v_{z_t}(z_t, t)$. 
Following the notations in Lemma \ref{lemma: GMM general covariance v}, we have each of the component variances to be $\Sigma_i =  \sigma^2 I$ for all $i \in [K]$. This means that $\Sigma_{i, t} = \sigma_t^2 I$, where $\sigma_t^2 = (1-t)^2 + t^2 \sigma^2$. 
Using \eqref{eq: nabla_v}, we have
  \[
A_t = \frac{1}{t}\{I - \frac{(1-t)}{\sigma_t^{2}} I\} +  \frac{t(1-t)}{2 \sigma_t^4} \underbrace{\sum_{i \ne j}w_{i,t} w_{j,t} (\mu_i - \mu_j) (\mu_i - \mu_j)^\top}_{B_t}, 
\]
Now, we have
\[
\|B_t\|_{op} \le \frac{1}{\sigma_t^4} \times \max_{i \ne j}\Vert\mu_i - \mu_j\Vert_2^2 \times (\sum_{i \ne j}w_{i,t} w_{j,t}) \le \frac{1}{\sigma_t^4} \times \max_{i \ne j}\Vert\mu_i - \mu_j\Vert_2^2.
\]
The last inequality follows from the fact that $\sum_{i \ne j}w_{i,t} w_{j,t} = 1 -( \sum_{k \in [K]} w_{k,t}^2) \le 1$.
Recall that $\Sigma = \sigma^2 I_d$, and this entails that $\Sigma_t = \sigma_t^2 I_d$, where $\sigma_t^2 := (1-t)^2 + t^2 \sigma^2$.
Therefore, using the above expressions we have 
\[
\Vert A_t\Vert_{op} \le \frac{\vert(1+\sigma^2)t -1\vert}{\sigma_t^2}  + \frac{D^2}{2 \sigma_t^4}.
\]

Next, we note that $\max_{t \in [0,1]} \vert (1 + \sigma^2)t -1\vert = \max\{\sigma^2 , 1\} \le 1+ \sigma^2$. This is due to the fact that $\vert (1 + \sigma^2)t -1\vert$ is strictly decreasing between $t \in [0, \frac{1}{1+\sigma^2}]$ and strictly increasing when $t \in [\frac{1}{1+\sigma^2}, 1]$. 

Also, we define $\sigma_*^2 :=\min_{t \in [0,1]} \sigma_t^2$. An elementary calculus shows that $\sigma_*^2 = \frac{\sigma^2}{1 + \sigma^2}$. Therefore, we have 

\[\Vert A_t\Vert_{op} \le \frac{\max\{\sigma^2, 1\}} {\sigma_*^2} + \frac{D^2}{2 \sigma_*^4} = \frac{(1+\sigma^2)^2}{\sigma^2} + \frac{D^2 (1+\sigma^2)^2}{2 \sigma^4} = \frac{(1+\sigma^2)^2}{\sigma^2} \br{1 + \frac{D^2}{2 \sigma^2}}\]
Hence the Lipschitz constant is $\frac{(1+\sigma^2)^2}{\sigma^2} \br{1 + \frac{D^2}{2 \sigma^2}}$. Now, the Wasserstein result is a direct application of Theorem \ref{thm: W2 ode disc} and Lemma \ref{lemma: bound on gamma_2} in conjunction, since the estimation error is zero.
\end{proof}

\section{{Proofs for straightness of 2-RF}}
\label{sec: proof of main straightness results}
\subsection{Proof of Theorem~\ref{thm:ItoSigma}}
\label{sec: proof gaussian RF}
Let $X_0\sim \cN(0, I)$ and $X_1\sim \cN(\mu, \Sigma)$. Let $\Sigma_t=t^2\Sigma+(1-t)^2I$.
Then we have that $X_t \sim \cN(t\mu, \Sigma_t),$. Let the density of $X_t$ be $\xi_t$ and the score $s_t(x) = \nabla_x \log \xi_t(x) = \Sigma_t^{-1}(t\mu-x)$. Therefore, by using (\ref{eq: drift-score}), the drift is given by:
\bas{
v(x,t)&=\frac{x}{t}+\frac{1-t}{t}\Sigma_t^{-1}(t\mu-x)\\
&=(1-t)\Sigma_t^{-1}\mu+\frac{1}{t}\br{I-(1-t)\Sigma_t^{-1}}x
}
Therefore, we have $\nabla_x v(x,t) = \frac{1}{t}\br{I-(1-t)\Sigma_t^{-1}}$. This shows that commutativity condition in Theorem \ref{prop: 1-rf Monge} is satisfied and the $z_0 \mapsto Z_1(z_0)$ is the Monge map. To obtain the exact form,
we want to solve the following ODE:
\begin{align}
\label{eq:covar-one-one}
\frac{dZ_t}{dt}-\frac{1}{t}\br{I-(1-t)\Sigma_t^{-1}}Z_t=(1-t)\Sigma_t^{-1}\mu; \quad Z_0 = z_0
\end{align}


Now we look at the structure of $I-(1-t)\Sigma_t^{-1}$. Let the eigendecomposition of $\Sigma=U\Lambda U^\top$. We will assume $\Sigma$ is full rank.
So, 
\bas{
I-(1-t)\Sigma_t^{-1}=U\Lambda_t U^\top
}
where
$\frac{\Lambda_t}{t}=\frac{1}{t}\{I-(1-t)(t^2\Lambda+(1-t)^2I)^{-1}\}$. This can also be written as:
\[
\lambda_{t,i}=\frac{1}{t}\left\{1-\frac{1-t}{t^2\lambda_i+(1-t)^2}\right\}=\frac{t(1+\lambda_i)-1}{t^2\lambda_i+(1-t)^2}.
\]
Substituting this into Equation~\eqref{eq:covar-one-one}, we have:
\[
\frac{dZ_t}{dt}-U \diag(\lambda_{t,1}, \ldots, \lambda_{t,d}) U^\top Z_t=(1-t)\Sigma_t^{-1}\mu.
\]

So, we first get the integrating factors of each eigenvalue.
\bas{
I_i(t)=\frac{1}{\sqrt{(1+\lambda_i)t^2-2t+1}}
}
Multiplying $U \diag(I_1(t), \ldots, I_d(t)) U^\top$ on both sides of the ODE and then solving we get:
\bas{
U\Lambda'_t U^\top Z_t= U\Lambda''_t U^\top\mu+\text{constant}
}
where $\lambda'_{t,i}=\frac{1}{\sqrt{(1+\lambda_i)t^2-2t+1}}$ and $\lambda''_{t,i}=\frac{t}{\sqrt{(1+\lambda_i)t^2-2t+1}}$

This yields,
\ba{
& \Sigma^{-1/2} Z_1(z_0)-z_0=\Sigma^{-1/2}\mu \nonumber\\
& 
\Rightarrow Z_1(z_0)=\Sigma^{1/2}z_0+\mu. \label{eq: gaus-gaus}
}
This finishes the proof.

\subsection{Proof of Theorem~\ref{thm:2gauss_rd}}
\label{sec:twogaussrd}

We point out that straightness of 1-RF can be obtained via an intuitive argument in this case. First note that straightness of RF is invariant under rotation. Under a proper rotation, the $d$ dimensional target distribution can be reduced to another where the means of the two components of the Gaussian mixture are sparse with two non-zero coefficients each, one of which is equal (lets say coordinate 1). Due to this it follows that ODE \eqref{eq: modified ode} gets decoupled and it can be analyzed coordinate-wise. So, essentially the $d$-dimensional problem gets reduced to one-dimensional case and the result follows from Proposition \ref{prop: 1-dim RF}. We elaborate more on this below.
\paragraph{Intuitive proof of straightness:}
\begin{proof}
Let \( \tilde{X}_0, \tilde{X}_1 \in \mathbb{R}^d \) for \( d \geq 2 \), where \( \tilde{X}_0 \sim \cN(0, I) \) and \( \tilde{X}_1 \sim \sum_{i=1}^2 \pi_i \, \cN(\tilde{\mu}_i, \sigma^2 I) \) with {\blue{$\sigma^2=1$ (for simplicity)}}. We start with the matrix \( \tilde{M} = \begin{bmatrix} \tilde{\mu}_1 & \tilde{\mu}_2 \end{bmatrix} \) and perform a QR decomposition: \( \tilde{M} = \tilde{Q} \tilde{R} \), where \( \tilde{Q} \in \mathbb{R}^{d \times 2} \) is an orthonormal matrix that spans the subspace of \( \tilde{\mu}_1 \) and \( \tilde{\mu}_2 \).

Next, we extend \( \tilde{Q} \) to a complete orthonormal basis for \( \mathbb{R}^d \) using \( \tilde{Q}' \in \mathbb{R}^{d \times (d-2)} \), which spans the orthogonal complement of the column space of $\tilde{Q}$. We define \( Q = \begin{bmatrix} \tilde{Q} & \tilde{Q}' \end{bmatrix}^\top \). This projection guarantees that:
\[
Q \tilde{\mu}_1 = (x_1, y_1, 0, \ldots, 0)^\top, \quad Q \tilde{\mu}_2 = (x_2, y_2, 0, \ldots, 0)^\top
\]
i.e., only the first two components are non-zero.

To equalize one of the components, we apply a rotation matrix \( R(\theta) \in \mathbb{R}^{d \times d} \), which rotates the first two components while leaving the others unchanged:
\[
R(\theta) = \begin{bmatrix} \cos\theta & -\sin\theta & 0 \\ \sin\theta & \cos\theta & 0 \\ 0 & 0 & I_{d-2} \end{bmatrix}
\]
We set \( \theta \) as:
\[
\theta = \tan^{-1}\left( \frac{y_2 - y_1}{x_1 - x_2} \right)
\]
This ensures that the second components of \( R(\theta) Q \tilde{\mu}_1 \) and \( R(\theta) Q \tilde{\mu}_2 \) are identical.

Finally, we define the overall transformation as \( P = R(\theta) Q \). This matrix \( P \in \mathbb{R}^{d \times d} \) is orthonormal (and hence, invertible) since it is the product of two orthonormal matrices. The transformation $P$, not only makes the last $d-1$ coordinates of the means identical but also reduces the effective dimension of the flow to two. 

Now, we rotate our space using the linear transformation $P$ and obtain the distributions $X_0 = P \tilde X_0 \sim \cN(0, I)$ and $X_1 = P \tilde X_1 \sim \sum_{i=1}^2 \pi_i\,  \cN(\mu_i, \Sigma)$, where $\mu_i = P\tilde \mu_i$, $ \Sigma = P \tilde \Sigma P^\top = I$. Also note that by the above construction of the transformation $P$, $\mu_{1, k} = \mu_{2, k} := c_k.$ for all $k \in [d]\backslash \bc{1}$. We first show that $(Z_0, Z_1) = \text{Rectify}(X_0, X_1)$ is straight and then argue that an invertible transformation does not hamper straightness.\\
To proceed, we apply the Rectify procedure on $( X_0,  X_1)$ and obtain the following ODE:  
\begin{align*}
    v_t( Z_t) = \frac{d Z_t}{dt}&=\frac{(2t-1) Z_t}{\sigma_t^2}+\frac{1-t}{\sigma_t^2}\sum_{i=1}^2w_i( Z_t) \mu_i
\end{align*}
For $k \in [d]\backslash \bc{1}$, we have that
\begin{align*}
    \frac{d Z_{t, k}}{dt}&=\frac{(2t-1) Z_{t, k}}{\sigma_t^2}+ c_k
\end{align*}
Hence, using (\ref{eq: gaus-gaus}) the final mapping is just a translation given by $Z_{1, k} = Z_{0, k} + c_k$.
However, for the first co-ordinate, for $g_t( Z_{t, 1}) = \log \br{\frac{\pi_2}{\pi_1}}-\frac{1}{2\sigma_t^2}\br{\br{ Z_{t, 1}-t \mu_{2, 1}}^2 - \br{ Z_{t, 1}-t \mu_{1, 1}}^2}$, we have 
\begin{align*}
    \frac{d Z_{t, 1}}{dt}&= \underbrace{\frac{(2t-1) Z_{t, 1}}{\sigma_t^2}+\frac{1-t}{\sigma_t^2} \br{\frac{ \mu_{1, 1} +  \mu_{2, 1}\exp\br{g_t( Z_{t, 1})}}{1+\exp\br{g_t( Z_{t, 1})}}}}_{v_1(Z_{t,1}, t)}
    \label{eq: ode-y-2-mix}
\end{align*}
Now using \eqref{eq: nabla_v}, it is easily verifiable that $\nabla v_1(Z_{t,1},t )$ is bounded, i.e., $v_1(Z_{t,1},t )$ is Lipschitz. Therefore, $Z_{t,1}(\cdot)$ is an increasing function. As a result $z_0 \mapsto Z_1(z_0)$ is co-ordinate wise increasing function and $H_t(z_0):= (1-t)z_0 + t Z_1(z_0)$ is an invertible map. Therefore, straightness of the resulting coupling follows.
\end{proof}

\subsection{Proof of Theorem~\ref{thm:mixtomix}}
\label{sec:proofgmmtogmm}

\begin{proof}
Consider $\bmu_{01} = (0,a)^\top, \bmu_{02} = (0,-a)^\top$ and $\bmu_{11} = (a,a)^\top, \bmu_{12} = (a,-a)^\top$ for some $a>0$.
Let $$X_0 \sim 0.5 \cN(\bmu_{01}, I) + 0.5 \cN(\bmu_{02}, I) ,\quad  X_1 \sim 0.5 \cN(\bmu_{11}, I) + 0.5 \cN(\bmu_{12},I).$$
In this case, the velocity functions in $x$ and $y$-direction for 1-rectification turns out to be
\begin{align*}
&u_t(x) = \frac{\left(2t-1\right)x}{\sigma_t^2} + \frac{(1-t)a}{\sigma_t^2},\\
& v_t(y) = \frac{\left(2t-1\right)y}{\sigma_t^2}\\
& 
+\frac{a}{\sigma_t^2}\cdot\frac{\exp\left(-\frac{\left(y-a\right)^{2}}{2\sigma_t^{2}}\right)\left(1-2t\right)-\ \exp\left(-\ \frac{\left(y+a\right)^{2}}{2\sigma_t^2}\right)\left(1-2t\right)+\exp\left(-\frac{\left(y\ -\ \left(2t-1\right)a\right)^{2}}{2\sigma_t^2}\right)-\exp\left(-\frac{\left(y\ +\ \left(2t-1\right)a\right)^{2}}{2\sigma_t^2}\right)}{\exp\left(-\frac{\left(y-a\right)^{2}}{2\sigma_t^2}\right)+\ \exp\left(-\ \frac{\left(y+a\right)^{2}}{2\sigma_t^2}\right)+\exp\left(-\frac{\left(y\ -\ \left(2t-1\right)a\right)^{2}}{2\sigma_t^2}\right)+\exp\left(-\frac{\left(y\ +\ \left(2t-1\right)a\right)^{2}}{2\sigma_t^2}\right)}.
\end{align*}

Next, we will take the derivative of $v_t(y)$ with respect to $y$. For notational brevity, let us define 
\begin{align*}
& e_1(y) = \exp\left(-\frac{\left(y-a\right)^{2}}{2\sigma_t^{2}}\right)(1-2t),\\
& e_2(y) = \exp\left(-\frac{\left(y + a\right)^{2}}{2\sigma_t^{2}}\right)(1-2t),\\
& e_3(y) = \exp\left(-\frac{\left(y-a (2t-1)\right)^{2}}{2\sigma_t^{2}}\right),\\
& e_4(y) = \exp\left(-\frac{\left(y + a (2t-1)\right)^{2}}{2\sigma_t^{2}}\right).
\end{align*}
Then we have 
\[
\abs{\frac{d v_t(y)}{dy}} \le \frac{2t-1}{\sigma_t^2} + \frac{a^2}{\sigma_t^4} \cdot \frac{4 \{e_1(y) e_2(y) + e_2(y) e_3(y) + e_3(y) e_4(y) + e_4(y) e_1(y)\}}{(\sum_{j=1}^4 e_j(y))^2} \le 2 + 4a^2.
\]
We used the basic inequalities $4(ab+bc+cd+da) \le (a+b+c+d)^2$ and $\sigma_t^2 \ge 1/2$ in the last step of the above display.

This shows that $v_t(y)$ is uniformly Lipschitz. This entails that the map $\cT: \bbR \to \bbR$ that sends $y_0$ to a point $y_1\in \bbR$, and defined through the ODE
\[
\ddt Y_t = v_t(Y_t); \;Y_0 = y_0,
\]
is an injective map due to the uniqueness of the solution of the above ODE. Also, we denote by $Y_t^{y_0}$ the solution of the above ODE.

To show the strict increasing property of $\cT$, let us consider the same ODE with $Y_0 = \tilde{y}_0 < y_0$. We also consider the solution $Y_t^{\tilde{y}_0}$. Consider the function $L_t := Y_t^{y_0} - Y_t^{\tilde{y}_0}$, which is also continuous in $t\in [0,1]$. To prove increasing property, it is enough to show that $L_1>0$. Let us assume that $L_1\le 0$. We already know $L_0>0$, and hence by Intermediate Value Property, we have there exists a $\tau\in(0,1]$ such that $L_\tau = 0$. This entails that there exists $y_\tau \in \bbR$ such that $Y_{\tau}^{y_0} = Y_{\tau}^{\tilde{y_0}} = y_\tau$. This shows that we have two different solutions of the ODE passing through $(\tau, y_\tau)$, which is a contradiction. This proves the coveted strict increasing property of $\cT$. Hence, \textit{we have a straight coupling} by similar argument as in previous section.
\end{proof}

\subsection{Proof of Theorem \ref{thm: straightness}}
\label{appendix: 1-rf straight}
Recall that we need to show that $\bbE (Z_1 - Z_0 \mid t Z_1 + (1-t)Z_0) = Z_1 - Z_0$ almost surely in $t, Z_0$. Therefore, it suffices to show that the function $H_t(z_0):= (1-t)z_0 + t Z_1(z_0)$ is an invertible map. We will equivalently show that $H_t(\cdot)$ is locally inveritible and a proper function \cite{plastock1974homeomorphisms}. 

\paragraph{$H_t$ is locally invertible}: Note that $\nabla_{z_0} H_t(z_0) = (1-t) I_d + t J_1^{z_0}$. Therefore, due to Assumption \ref{assumption: J1 psd} we can conclude that 
$$
u^\top \nabla_{z_0} H_t(z_0) u \ge (1-t) \quad \text{ for all $ t \in [0,1]$ and $u \in \{w \in \bbR^d \mid \norm{w}_2 = 1\}$.}$$
Specifically, $\nabla_{z_0} H_t(z_0)$ is invertible if $t <1$. This shows that $H_t$ is locally invertible for all $z_0$, as long as $t <1$.

\paragraph{$H_t$ is proper}:  We will show that $\norm{H_t(z_0)}_2 \to \infty$ as $\norm{z_0}_2 \to \infty$ for $t <1$. Let us fix $z_1 \in \bbR^d$ and define the function 
$h_t(\lambda) = \innerprod{H_t(\lambda z_0 + \bar{\lambda} z_1), z_0 - z_1}$, where $\lambda \in [0,1]$ and $\bar{\lambda} = 1- \lambda$. By the Taylor's formula we have the following for some $\tilde{\lambda} \in (0,1)$:

\begin{align*}
    & h_t(1) = h_t(0) + (z_0 - z_1)^\top \nabla H_t(\tilde z_\lambda) (z_0 - z_1) \quad \quad ; \tilde z_\lambda = \tilde \lambda z_0 + (1 - \tilde \lambda) z_1\\
    & \Rightarrow \innerprod{H_t(z_0) - H_t (z_1), z_0 - z_1} \ge (1-t) \norm{z_0 - z_1}_2^2\\
    & \Rightarrow \frac{\norm{H_t(z_0) - H_t(z_1)}_2^2}{2 (1-t)} \ge \frac{(1-t)}{2} \norm{z_0 - z_1}_2^2\quad \left(\text{Young's inequality: $ab \le \frac{a^2}{2 \eta}  + \frac{\eta b^2}{2}$}\right)
\end{align*}
The final inequality shows that $\lim_{\norm{z_0}_2 \to \infty}\norm{H_t(z_0)}_2 = \infty$.

\paragraph{$H_t$ is globally invertible}: $H_t$ is locally invertible and proper. Then $H_t$ is globally invertible due to Corollary 2.1 of \cite{plastock1974homeomorphisms}.

\paragraph{Straightness}: This shows that $H_t(Z_0)$ is invertible almost surely in $Z_0 \sim N(0, I_d)$ for all $ t \in [0,1)$. 
Then, we have 

\begin{equation}\label{eq: V_straightness}
\begin{aligned}
     V(Z_0, Z_1) &:= \bbE_{Z_0, t}\norm{Z_1 - Z_0 - \bbE \{Z_1 - Z_0 \mid H_t(Z_0)\}}_2 \\
     & = \int_{0}^1 \bbE_{Z_0}\Big[\norm{Z_1 - Z_0 - \bbE \{Z_1 - Z_0 \mid H_t(Z_0)\}}_2 \Big] \;dt\\
     & = 0.
\end{aligned}
\end{equation}
Therefore, $(Z_0, Z_1)$ is a straight coupling.

\subsection{Proof of Theorem \ref{prop: K-mixture rf}}
\label{appendix: proof of K-mixture rf}
We will show that Assumption \ref{assumption: J1 psd} is satisfied in this case. For notational convenience we drop the superscript $z_0$ in $J_t^{z_0}$ and denote it by $J_t$. We start with the ODE 
\[
\ddt J_t = \nabla v (Z_t(z_0), t) J_t, \quad J_0 = I_d
\]
Let $U_t = J_t^{-1}$. Then, elementary calculation shows that 
\[
\ddt U_t = - J_t^{-1} \dot J_t J_t^{-1} = -U_t \nabla v(Z_t(z_0),t ).  
\]
Using \eqref{eq: nabla_v}, we get 
\[
A_t := \nabla v(Z_t(z_0), z_0) = \frac{1}{t}\bs{I_d - (1-t) \Sigma_t^{-1}} +  \frac{t(1-t)}{2} \underbrace{\sum_{i \ne j}w_{i,t} w_{j,t} \Sigma_t^{-1}(\mu_i - \mu_j) (\mu_i - \mu_j)^\top \Sigma_t^{-1}}_{B_t}, 
\]
Recall that $\Sigma = \sigma^2 I_d$, and this entails that $\Sigma_t = \sigma_t^2 I_d$, where $\sigma_t^2 := (1-t)^2 + t^2 \sigma^2$. Therefore, 
\[
\norm{B_t}_{op} \le \frac{1}{\sigma_t^4} \times \underbrace{\max_{i \ne j}\norm{\mu_i - \mu_j}_2^2}_{=: D}.
\]

\paragraph{Bounding $\norm{U_t}_{op}$}:
Take any unit vector $u$. Since $A_t$ is symmetric,
\begin{align*}
    \ddt u^T U_t U_t^T u & = - 2u^T U_t A_t U_t^T u\\
     & = - \frac{2}{t} u^\top U_t \bs{I_d - (1-t) \Sigma_t^{-1}} U_t^\top u  - \underbrace{ t (1-t) u^\top U_t B_t U_t^\top u}_{\ge 0}\\
     & \le - 2 \frac{(1 + \sigma^2)t -1}{\sigma_t^2} \br{u^\top U_t U_t^\top u} \\
\end{align*}
Hence for any unit vector $u$,
\begin{align*}
    u^T U_t U_t^\top u\leq \frac{1}{\sigma_t^2 } \Rightarrow \Norm{U_t U_t^\top}_{op} \le \frac{1}{\sigma_t^2 }
\end{align*}
Hence 
\begin{align*}
    \|U_t\|^2_{op}\leq \|U_t U_t^T\|_{op} \le \frac{1}{\sigma_t^2 }.
\end{align*}

\paragraph{Lower bounding on $u^T U_t u$}:
Assuming $\norm{u}_2 = 1$, we consider the evolution:

\begin{equation*}
    \ddt \left( u^T U_t u \right) = -\frac{(1 +\sigma^2)t - 1}{\sigma_t^2} u^T U_t u - \frac{t(1 - t)}{2} u^T U_t B_t u.
\end{equation*}

Using the bound:

\begin{equation*}
    |u^T  U_t B_t u| \le \Norm{U_t}_{op} 
 \Norm{B_t}_{op} \leq \frac{D}{\sigma_t^5} ,
\end{equation*}

we get:

\begin{equation*}
    \ddt \left( u^T U_t u \right) \geq -\frac{(1+\sigma^2)t - 1}{\sigma_t^2} u^T U_t u - \frac{t(1 - t)}{2\sigma_t^5} D.
\end{equation*}


Define:

\begin{equation*}
    I(t) = \int_0^t \frac{(1+\sigma^2)s - 1}{\sigma_s^2} ds = \frac{1}{2} \log(\sigma_t^2).\Rightarrow e^{I(t)} = \sigma_t
\end{equation*}

Multiplying by the integrating factor, we obtain:

\begin{equation*}
    \ddt \left( e^{I(t)} u^T U_t u \right) \geq - e^{I(t)} \cdot \frac{t(1 - t)}{2\sigma_t^5} D = -\frac{t(1 - t)}{2\sigma_t^4} D .
\end{equation*}

Integrating both sides, we obtain:

\begin{equation*}
    \sigma u^T U_1 u \geq 1- \frac{D}{2}\int_0^1 \frac{s(1 - s)}{\sigma_s^{4}} ds = 1 - \frac{D}{4 \sigma^2}.  
\end{equation*}

Therefore, if $D \le 4 \sigma^2$ then $u^T U_1 u \ge 0$ for all unit vector $u$. This ensures that Assumption \ref{assumption: J1 psd} is satisfied as Lemma \ref{lemma: lower bund on matrix} yields

\[
\lambda_{min}(J_1 + J_1^T) = 2 \min_{u : \norm{u}_2 =1} u^T J_1 u \ge 2\br{\min_{u : \norm{u}_2 =1} u^\top U_1 u} \times \lambda_{min}(J_1^T J_1)\ge 0.
\]
Moreover, Assumption \ref{assumption : osgood} is automatically satisfied (see Section \ref{sec: app gaussian to general mixture}).
Therefore, the straightness of 1-RF follows from Theorem \ref{thm: straightness}.

\subsection{A general version of Theorem \ref{prop: K-mixture rf}}
\label{sec: proof of general version of K-mixture}
In this section, we present a slightly general version of Theorem \ref{prop: K-mixture rf} as follows.

\begin{theorem}
    \label{prop: K-mixture rf general}
    Let $(X_0, X_1) \sim N(0, I_d) \otimes \rho_1$ where $\rho_1 := \sum_{j=1}^K \pi_j N(\mu_j,  \Sigma )$ with mixture proportions $\{\pi_j\}_{j\in [K]}\in (0,1)^K$. Let $m, M$ be minimum and maximum eigenvalues of $\Sigma^{1/2}$ respectively, and $\kappa := M/ m$ be its  condition number. If $\max_{i \ne j} \norm{\mu_i - \mu_j}_2^2 \le 2m^2 (3 - \kappa^2)$, then 1-RF yields a straight coupling.
\end{theorem}

\begin{proof}
Our target distribution is $\pi_1 = \sum_{j = 1}^K \pi_j N(\mu_j, \Sigma)$. WLOG, we can assume that 
\[
\Sigma = \diag(\sigma_1^2, \ldots, \sigma_d^2).
\]
If $\Sigma$ is not diagonal, the let $\Sigma = P \Lambda P^\top$ be the spectral decomposition of $\Sigma$, where $\Lambda$ is diagonal matrix and $PP^\top = I_d$.
Now, recall that straightness is invariant under rotation. Therefore, we can always restrict ourselves to $P^\top_{\#} \rho_0 = N(0, I_d)$ and $P^\top_{\#} \rho_1 = \sum_{j= 1}^K \pi_j N(P^\top \mu_j, \Lambda)$. If RF leads to straight coupling in the rotated frame, the it also does the same in the un-rotated one. 

 Keeping this in mind, 
we start with the ODE 
\[
\ddt J_t = \nabla v (Z_t(z_0), t) J_t, \quad J_0 = I_d.
\]
Let $U_t = J_t^{-1}$. Then, elementary calculation shows that 
\[
\ddt U_t = - J_t^{-1} \dot J_t J_t^{-1} = -U_t \nabla v(Z_t(z_0),t ).  
\]
We know
\[
A_t := \nabla v(Z_t(z_0), z_0) = \frac{1}{t}\bs{I_d - (1-t) \Sigma_t^{-1}} +  \frac{t(1-t)}{2} \underbrace{\sum_{i \ne j}w_{i,t} w_{j,t} \Sigma_t^{-1}(\mu_i - \mu_j) (\mu_i - \mu_j)^\top \Sigma_t^{-1}}_{B_t}, 
\]
Recall that $M = \max_{i\in [d]} \sigma_i$ and $m = \min_{i\in [d]} \sigma_i$. We also define $M_t^2 := (1-t)^2 + t^2 M^2$ and $m_t^2 := (1-t)^2 + t^2 m^2$.
Therefore, 
\[
\norm{B_t}_{op} \le \frac{1}{m_t^4} \times \underbrace{\max_{i \ne j}\norm{\mu_i - \mu_j}_2^2}_{=: D}.
\]

\paragraph{Bounding $\norm{U_t}_{op}$}:
Take any unit vector $u$. Since $A_t$ is symmetric,
\begin{align*}
    \ddt u^T U_t U_t^T u & = - 2u^T U_t A_t U_t^T u\\
     & = - \frac{2}{t} u^\top U_t \bs{I_d - (1-t) \Sigma_t^{-1}} U_t^\top u  - \underbrace{ t (1-t) u^\top U_t B_t U_t^\top u}_{\ge 0}\\
     & =  u^\top U_t\;\diag\br{- 2 \frac{(1 + \sigma_1^2)t -1}{(1-t)^2 + t^2 \sigma_1^2}, \ldots, - 2 \frac{(1 + \sigma_d^2)t -1}{(1-t)^2 + t^2 \sigma_d^2}} \;U_t^\top u  - \underbrace{ t (1-t) u^\top U_t B_t U_t^\top u}_{\ge 0}\\
     & \le - 2 \frac{(1 + m^2)t -1}{m_t^2} \br{u^\top U_t U_t^\top u} \\
\end{align*}
The last inequality is true because the function $f_t(a) := -  \frac{(1 + a)t -1}{(1-t)^2 + t^2 a} $ is non-increasing over $a>0$ as $\frac{\partial f_t(a)}{\partial a} = - \frac{ t(1-t)}{((1-t)^2 + t^2 a)^2}<0$.
Hence, the following holds for any unit vector $u$:
\begin{align*}
    u^T U_t U_t^\top u\leq \frac{1}{m_t^2 } \Rightarrow \Norm{U_t U_t^\top}_{op} \le \frac{1}{m_t^2 }.
\end{align*}
Hence, 
\begin{equation}
    \label{eq: U_t bound}\|U_t\|^2_{op}\leq \|U_t U_t^T\|_{op} \le \frac{1}{m_t^2 }.
\end{equation}

\paragraph{Lower bounding on $u^T U_t u$}:
Assuming $\norm{u}_2 = 1$, we again consider the ODE:

\begin{equation*}
\begin{aligned}
    \ddt \left( u^T U_t u \right) & = -\frac{1}{t} u^\top U_t \bs{I_d - (1-t) \Sigma_t^{-1}}  u  - \frac{t(1-t)}{2} u^T U_t B_t u\\
    & = u^\top U_t [ f_t(m^2) I_d + \diag\left(  f_t(\sigma_1^2) -  f_t(m^2), \ldots,  f_t(\sigma_d^2) -  f_t(m^2)\right)  ] u -  \frac{t(1-t)}{2} u^T U_t B_t u\\
    & = -\frac{(1 + m^2)t -1}{m_t^2} u^\top U_t u  + u^\top U_t \diag\left(  f_t(\sigma_1^2) -  f_t(m^2), \ldots,  f_t(\sigma_d^2) -  f_t(m^2)\right) u - \frac{t(1-t)}{2} u^T U_t B_t u\\
    & = -\frac{(1 + m^2)t -1}{m_t^2} u^\top U_t u  - u^\top U_t \diag\left(  f_t(m^2) - f_t(\sigma_1^2)   , \ldots,  f_t(m^2) - f_t(\sigma_d^2) \right) u - \frac{t(1-t)}{2} u^T U_t B_t u
\end{aligned}
\end{equation*}
As $f_t(\cdot)$ is decreasing function, we have 
\begin{align*}
f_t(m^2) - f_t(\sigma_k^2) & = - \frac{t (1-t)}{((1-t)^2 + t^2 \xi_k^2)^2} .( m^2 - \sigma_k^2) \quad (\text{where $m \le \xi_k \le \sigma_k$})\\
& = \frac{t (1-t)}{((1-t)^2 + t^2 \xi_k^2)^2} .(\sigma_k^2 - m^2)\\
& \le \frac{t(1-t)}{m_t^4} . (M^2 - m^2) = \frac{t(1-t)}{m_t^4} . (\kappa^2 -1)m^2
\end{align*}

Next, Using \eqref{eq: U_t bound} we get:

\begin{equation*}
    |u^T  U_t B_t u| \le \Norm{U_t}_{op} 
 \Norm{B_t}_{op} \leq \frac{D}{m_t^5} ,
\end{equation*}
\[
|u^\top U_t \diag\left(  f_t(m^2) - f_t(\sigma_1^2)   , \ldots,  f_t(m^2) - f_t(\sigma_d^2) \right) u| \le \frac{t(1-t)}{m_t^5} (\kappa^2 -1) m^2.
\]
we get:

\begin{equation*}
    \ddt \left( u^T U_t u \right) \geq -  \frac{(1+m^2)t - 1}{m_t^2} u^T U_t u - \frac{t(1-t)}{m_t^5} (\kappa^2 -1) m^2 - \frac{t(1 - t)}{ 2 m_t^5} D.
\end{equation*}


Define:

\begin{equation*}
    I(t) =  \int_0^t \frac{(1+m^2)s - 1}{m_s^2} ds = \frac{1}{2}\log(m_t^2).\Rightarrow e^{I(t)} = m_t
\end{equation*}

Multiplying by the integrating factor, we obtain:

\begin{equation*}
    \ddt \left( e^{I(t)} u^T U_t u \right) \geq  - e^{I(t)}. \frac{t(1-t)}{m_t^5} (\kappa^2 -1)m^2 - e^{I(t)} \cdot \frac{t(1 - t)}{m_t^5} D = - \frac{t(1-t)}{m_t^4}(\kappa^2 -1)m^2 -\frac{ t(1 - t) }{2 m_t^4} D .
\end{equation*}

Integrating both sides, we obtain:

\begin{equation*}
\begin{aligned}
    m u^T U_1 u & \geq 1- \bc{\frac{D}{2} + (\kappa^2 -1)m^2}\int_0^1 \frac{s(1 - s)}{m_s^{4}} ds \\
    & =  1 - \bc{\frac{D}{2} + (\kappa^2 -1)m^2} . \frac{ 1}{2 m^2}.  
\end{aligned}
\end{equation*}
The last inequality follows from Lemma \ref{lemma: mindless integration}.
Therefore, if we have 
$$
D \le  2 m^2 (3 - \kappa^2),$$
then $u^T U_1 u \ge 0$ for all unit vector $u$. This ensures that Assumption \ref{assumption: J1 psd} is satisfied as Lemma \ref{lemma: lower bund on matrix} yields

\[
\lambda_{min}(J_1 + J_1^T) = 2 \min_{u : \norm{u}_2 =1} u^T J_1 u \ge 2\br{\min_{u : \norm{u}_2 =1} u^\top U_1 u} \times \lambda_{min}(J_1^T J_1)\ge 0.
\]
Moreover, Assumption \ref{assumption : osgood} is automatically satisfied (see Section \ref{sec: app gaussian to general mixture}).
Therefore, the straightness of 1-RF follows from Theorem \ref{thm: straightness}.
\end{proof}

\section{{Proofs for Monge optimality of 2-RF}}
\label{sec: monge map appendix}

\subsection{Proof of Theorem \ref{prop: 1-rf Monge}}
\label{appendix: 1-rf monge}
Recall the ODE \eqref{eq: jacobian ODE}
\begin{equation}
     \frac{dJ_t^{z_0}}{dt} = \nabla_{Z_t}v_t (Z_t(z_0)) J_t^{z_0}; \quad J_0^{z_0} = I_d. 
 \end{equation}
Due to the commutativity assumption, the unique solution to the above the ODE can be written in the following form 
\begin{equation}
J_t^{z_0} = \exp \left(\int_0^t \nabla_{Z_u}v_u (Z_u(z_0)) \; du\right),
\label{eq: J_t solution}
\end{equation}
where $\exp(A) := \sum_{k=0}^\infty A^k/k!$ for a $d\times d$ matrix $A$. We point the readers to \cite{ma2009note} and Section 5 of \cite{magnus1954exponential} for discussions related to \eqref{eq: J_t solution}.  Also, note that $\nabla_{Z_u}v_u (Z_u(z_0))$ is a symmetric matrix which ensures that $J_t^{z_0}$ is a symmetric positive definite matrix for all $t \in [0,1]$. In particular, $J_1^{z_0} = \nabla_{z_0} Z_1(z_0)$ is also a symmetric positive definite matrix. Now, we will show that there exists a convex function $\varphi$ such that $Z_1 = \nabla \varphi$. To show this, we will essentially use the symmetry of $\nabla_{z_0} Z_1(z_0)$. Let us define 
\[
\varphi(z) = \int_{0}^1 \innerprod{Z_1(t z), z}\; dt,
\]
and note the following algebraic identity \begin{align*}
\frac{\partial Z_{1j}(tz)}{\partial t} & = \lim_{h \to 0} \frac{Z_{1j}(t z +  h z) - Z_{1j}(tz)}{h} \\
& = \sum_{k=1}^d z_k . \frac{\partial Z_{1j}(u)}{\partial u_k}\Big \vert_{u = tz}\\
& = \frac{1}{t}\sum_{k=1}^d z_k . \frac{\partial Z_{1j}(tz)}{\partial z_k}
\end{align*}
Therefore, we have 
\begin{align*}
    \frac{\partial \varphi(z)}{\partial z_j} &= \int_0^1 \sum_{k = 1}^d \frac{\partial (z_k Z_{1k}(tz))}{\partial z_j}\; dt\\
    & = \int_0^1 Z_{1j}(tz)\; dt + \int_0^1 \sum_{k = 1}^d z_k \blue{\frac{\partial Z_{1k}(tz)}{\partial z_j}}\; dt\\
    & = \int_0^1 Z_{1j}(tz)\; dt + \int_0^1 \sum_{k = 1}^d z_k \blue{\frac{\partial Z_{1j}(tz)}{\partial z_k}}\; dt \quad \text{(Due to Symmetry)}\\
    & = \int_0^1 Z_{1j}(tz)\; dt + \int_0^1 t . \frac{\partial Z_{1j}(tz)}{\partial t}\; dt\\
    & = \int_0^1 Z_{1j}(tz) + t Z_{1j}(tz) \big\vert_{t = 0}^{t=1} - \int_{0}^1 Z_{1j}(tz)\; dt\\
    & = Z_{1j}(z).
    \end{align*}
This shows that $Z_1 = \nabla \varphi$, and in fact $\nabla^2 \varphi = \nabla Z_1 \succ 0$. Therefore, $\varphi$ is also a convex function. Then the optimality of the coupling $(Z_0,Z_1)$ follows from Theorem 1.48 of \cite{santambrogio2015optimal}.

\bk


\subsection{RF in one-dimension yields Monge coupling} 
\label{sec: RF in 1-D}
1-RF always yields a straight coupling as long as the solution to ODE \eqref{eq: modified ode} exists and is unique. This is because, in the one-dimensional setting, the commutativity condition in Theorem \ref{prop: 1-rf Monge} is trivially met. Thus, we derive the following conclusion.

\begin{proposition}
    \label{prop: 1-dim RF}
    Let $(X_0,X_1) \sim N(0,1) \times \rho_1$ be a bivariate random vector in $\bbR^2$, and let the conditions in Theorem \ref{thm: gaussian straight flow} hold. Then 1-RF yields the Monge transport map between $N(0,1)$ and $\rho_1$, and hence it also produces a straight coupling.
\end{proposition}
The above theorem shows that 1-RF yields the Monge map for any target distribution (with a second moment) in one dimension. For example, the target distribution $\rho_1$ can be any log-concave or a general $K$-mixture of Gaussian distribution, and 1-RF will yield the Monge map between $N(0,1)$ and $\rho_1$. 

However, it is instructive to point out that the straightness of the 1-RF can be understood through a much more intuitive and fundamental argument \citep[Theorem D.10]{liu2023flow}. In the one-dimensional case, uniqueness of the solution of ODE \eqref{eq: modified ode} implies that the map $z_0\mapsto Z_t(z_0)$ is a monotonically increasing function for all $t \in (0,1]$. Then the straightness follows immediately from Lemma D.9 of \cite{liu2023flow}. 
In addition, the monotonicity property also ensure that all the quantiles are preserved:
\begin{lemma}
\label{lem:quantile_invariance}
   Let $z_0\in \bbR$ and write $z_t:= Z_t(z_0)$. If the solution of ODE \eqref{eq: ode-true} is unique, then $\bbP(Z_t \le z_t)$ is a constant depending on $z_0$ for all $t$.
\end{lemma} 

\textit{Proof.}
    We recall the ODE $\dot Z_t = v_t(Z_t)$ with $Z_0  = z_0$. As $x \mapsto v_t(x)$ is uniformly Lipschitz, there exists a unique solution $\{Z_t\}_{t \in [0,1]}$ such that $Z_0 = z_0$.  Moreover, the map $Z_t : z_0 \mapsto z_t$ is monotonically increasing. To see this, let us assume  $z_0 > \tilde z_0$, but $z_t < \tilde z_t$. Note that $G(\tau) : = Z_\tau(z_0) - Z_{\tau}(\tilde z_0)$ is continuous in $\tau$. Also, $G(0)>0$ and $G(t)<0$. By the intermediate value property, there exists a $t_0 \in [0,1]$ such that $G(t_0) = 0$, i.e., $z_{t_0} = \tilde z_{t_0}$. This violates the uniqueness condition of the ODE solution. Hence, $Z_t$ is monotonically increasing.
By monotonicity, it follows that   
\[
\pr(Z_t \le z_t) 
= \pr(Z_0 \le z_0).
\]

In addition, this monotonicity property also ensures that $Z_1 = \nabla \varphi$ for some convex function $\varphi$. This immediately shows that $Z_1(\cdot)$ is the Monge map \citep[Theorem 1.48]{santambrogio2015optimal}. However, such arguments can not be easily generalized in higher dimensions and require deeper theoretical treatments as in Theorem \ref{prop: 1-rf Monge}.
In the next sections, we move to examples in higher dimensions.

\subsection{Proof of Proposition \ref{prop: monge map gaussian & 2-mix}}

\label{sec: proof of monge map gaussian & 2-mix}

\paragraph{Gaussian to Gaussian case}:
As shown in Section \ref{sec: proof gaussian RF}, we have 
\[\nabla_x v(x,t) = \frac{1}{t}\br{I-(1-t)\Sigma_t^{-1}}.\]
Therefore, It is clear that $\nabla v(Z_t(z_0),t)$ and $\nabla v(Z_s(z_0),s)$ are commutative. Hence, the result follows from Theorem \ref{prop: 1-rf Monge}.

\paragraph{Gaussian to 2-mixture of Gaussian case}:
First, note that Assumption \ref{assumption : osgood} is satisfied by the discussion in Section \ref{sec: app gaussian to general mixture}. Therefore, 
it suffices to prove the commutativity of $\nabla v(Z_t(z_0), t)$ and $\nabla v(Z_s(z_0), s)$ for some $t < s$. Using \eqref{eq: nabla_v}, we have 
\[
A_t := \nabla v(Z_t(z_0), z_0) = \frac{(1+\sigma^2)t - 1}{\sigma_t^2}I_d +  \frac{t(1-t)}{\sigma_t^4} w_{1,t} w_{2,t} (\mu_1 - \mu_2) (\mu_1 - \mu_2)^\top, 
\]
where $\sigma_t^2 : = t^2 \sigma^2 + (1-t)^2$. Now, it is evident that $A_t A_s = A_s A_t$. Hence, the result follows from Theorem \ref{prop: 1-rf Monge}.

\section{{Proofs for Existence of RF}}
In this section, we collect the proofs of the main results of Appendix \ref{sec: general straightness}.

\subsection{Proof of Theorem \ref{thm: gaussian straight flow}
}
\label{sec: app gaussian straight flow}
 We start by analyzing the velocity function. Recall that 
    \[
    v_t(x) = \begin{cases}
        \frac{x}{t} + \left(\frac{1-t}{t}\right) s_t(x) & , 0 < t < 1\\
        \bbE (X_1) - x  & ,\quad t =0\\
        x &, \quad t =1.
    \end{cases}
    \]
    where $s_t(x)$ is the (data) score function of $(1-t)X_0 + tX_1$. Let $\phi$ denote the standard gaussian density function in $\bbR^d$. 
    
    \textbf{Verifying Assumption \ref{assumption: locally L-lipschitz}:} For $t\in [0,1)$ we have 
    \begin{align*}
        s_t(x) &= \nabla_x \log \left(\int_{- \infty}^\infty (1-t)^{-d/2} \phi\left(\frac{ x - ty}{1-t}\right) \; \rho_1(dy) \right)\\
        & = \frac{ \frac{1}{1-t}\int_{-\infty}^{\infty} \left(\frac{ty - x}{1-t}\right)\phi\left(\frac{ x - ty}{1-t}\right)  \; \rho_1(dy)}{\int_{-\infty}^\infty \phi\left(\frac{ x - ty}{1-t}\right)  \; \rho_1(dy) }\\
        & = \frac{t}{(1-t)^2}. \frac{\int_{-\infty}^{\infty} y\phi\left(\frac{ x - ty}{1-t}\right)   \; \rho_1(dy)}{\int_{-\infty}^\infty \phi\left(\frac{ x - ty}{1-t}\right)   \;\rho_1(dy) } - \frac{x}{(1-t)^2}.
    \end{align*}
    Therefore, $v_t(x) = \frac{\int_{-\infty}^{\infty} \left(\frac{y-x}{1-t}\right)\phi\left(\frac{ x - ty}{1-t}\right) \; \rho_1(dy)}{\int_{-\infty}^\infty \phi\left(\frac{ x - ty}{1-t}\right) \; \rho_1(dy)}$ for $t \in [0,1)$.

It is quite clear that $v_0(x)$ and $v_1(x)$ are $\cC^2$ functions. Moreover, one can show that $v_t(x)$ is also $\cC^2$ function for every $t \in (0,1)$ ($\nabla_x$ and $\int$ are interchangeable due to moment condition). It suffices to show that $\Psi_1(x)  := \int_{-\infty}^{\infty} y \phi\left(\frac{ x - ty}{1-t}\right) \; \rho_1(dy)$ and $\Psi_2(x):= \int_{-\infty}^{\infty} \phi\left(\frac{ x - ty}{1-t}\right) \; \rho_1(dy)$ are $\cC^2$ functions and $\Psi_2 >0$. Note that, $\Psi_2(x) = \bbE_{X_1 \sim \rho_1} \phi\left(\frac{x - t X_1}{1-t}\right)>0$. Now, we will show that $\Psi_1(x) $ is $\cC^1$. One can similarly show that it is also $\cC^2$ by following a similar argument.

We define
\[
D(x, y):= \nabla_x \left[y \phi \left(\frac{x - ty}{1-t}\right)\right] = \frac{1}{(1-t)^2} y \left(ty - x\right)^\top \exp \left( - \frac{\norm{x - ty}_2^2}{2 (1-t)^2}\right).
\]

Note that if $\norm{y}_2^2 \ge 4 \norm{x}_2^2/t^2$, we have $\innerprod{ u, D(x,y) u} \le \frac{t \norm{y}_2^2 + \norm{y}_2 \norm{x}_2}{(1-t)^2} \exp( - t\norm{y}_2^2/4)$ for all $u \in \bbS^{d-1}$, as $\norm{ty - x}_2^2 \ge (t^2/2)  \norm{y}_2^2 - \norm{x}_2^2 \ge (t^2/4)  \norm{y}_2^2$. In addition, the upper bound is integrable w.r.t $\rho_1(dy)$. For $\norm{y}_2^2 \le 4 \norm{x}_2^2/t^2$, we have $\innerprod{u,D(x,y)u } \le  \frac{t \norm{y}_2^2 + \norm{y}_2 \norm{x}_2}{(1-t)^2} \le \frac{ 6 \norm{x}_2^2}{t (1-t)^2}$, and the upper bound is obviously integrabel w.r.t $\rho_1(dy)$. Therefore, we have 
\[
\nabla \Psi_1(x) = \int_{-\infty}^\infty D(x,y) \; \rho_1(dy).
\]
The continuity also follows from generalized DCT. One can take a further derivative to show that $\Psi_1$ is $\cC^2$ function, and follow the similar argument for $\Psi_2(x)$.
    

\textbf{Non-explosive:}
For notational brevity, we write $X_t$ instead of $X_t(z_0)$. Note that 
\[
\ddt \norm{X_t}_2^2 = \innerprod{X_t, v_t(X_t)} \le h(\norm{X_t}_2^2).
\]
Write $U_t := \norm{X_t}_2^2$. Let $V_t$ be a sequence of maps such that 
\[
\ddt V_t = h(V_t); \quad V_0 = U_0.
\]
Due to Condition \eqref{eq: osgood}, we have $V_t < \infty$. 
Next, we claim that $U_t \le V_t$ for all $t \in [0,1]$. 

\underline{\textit{Under local-lipschitz property:}} If not, then there exist times $t_0, t_1$ such that 
\[
U_{t_0} = V_{t_0}, \quad \text{and }\quad U_t > V_t\quad \text{for all $t_0 < t \le t_1$.}
\]
Define $\Delta(t):= U_t - V_t$. Therefore, we have $\Delta(t_0)=0$ and $\Delta(t)>0$ for all $t \in (t_0, t_1]$. Let $w = U_{t_0} = V_{t_0}$. Due to local-Lipschitz property of $h$, there exists $\delta_w>0$ and $L_{w}>0$ such that 
\[
\abs{w_1 - w} \vee \abs{w_2-w} <\delta_w \Rightarrow \abs{h(w_1) - h(w_2)} \le L_w \abs{w_1 - w_2}.
\]
Due to continuity of $U_t$ and $V_t$ at $t = t_0$, there exists $\eta>0$ such that $t+\eta < t_1$ and for all $\eta^\prime \le \eta$ we have $\abs{U_{t_0+\eta^\prime} - w} \vee \abs{V_{t_0 + \eta^\prime} - w} <\delta_w$. For , $t \in [t_0, t_0 + \eta]$, we consider the ODE
\begin{align*}
\dot \Delta (t) & = \dot U_t - \dot V_t\\
& = h(U_t) - h(V_t)\\
& \le L_w \abs{U_t - V_t} \quad (\text{local-Lipshcitzness})\\
& = L_w \Delta(t) \quad \quad\quad (\text{as $\Delta(t)>0$}).
\end{align*}
Therefore, by Gronwall's lemma we have $\Delta (t) \le \Delta(t_0) \exp(L_w t)$. This implies that $\Delta(t) \le 0$ for $t \in (t_0 , t_0 +\eta]$, which is a contradiction to the fact that $\Delta(t)>0$ for all $t \in (t_0, t_1]$. Hence, we have $U_t \le V_t <\infty$ for all $t \in [0,1]$.  This establishes the non-explosive property (Condition \eqref{eq: non-explosive}) of the ODE.

\underline{\textit{Under strictly increasing property:}}
In this case, we will show a stronger result, i.e., $U_t< V_t$ for all $t \in (0,1]$.
If not, let $\tau := \inf\{t>0: U_t \ge V_t \}$. By definition, we have $\tau>0$ and $U_\tau \ge V_\tau$. This implies that 
\[
\int_0^\tau h(U_t) - h(V_t) \; dt\ge 0 \Rightarrow \text{$\exists s \in (0, \tau)$ such that $h(U_s) \ge h(V_s)$.}
\]
Therefore, we have $U_s \ge V_s$, which contradicts the definition of $\tau$. Hence, we have $U_t <V_t$ for all $t \in (0,1]$.

Now the result follows by applying Proposition \ref{thm: general straight flow}.

\subsection{Non-explosivity: Gaussian to a general mixture of Gaussian}
\label{sec: app gaussian to general mixture}
First, for notational brevity, we write $\norm{u}_{\Sigma} = \sqrt{u^\top \Sigma^{-1} u}$ for a positive-definite matrix $\Sigma$.
Let $X_0 \sim N\br{0, I_d}$ and $X_1 \sim \sum_{i=1}^K \pi_i N\br{\mu_i, \Sigma_i}$. Let $X_t = tX_1 + (1-t)X_0$, then we have
\begin{align} 
    v_t(x) = \frac{x}{t} + \frac{1-t}{t}s_t(x)
\end{align}
where, $s_t(x) = \nabla_x \log p_t(x)$ is given by
\[s_t(x) = \sum_i w_{i, t}(x) \Sigma_{i,t}^{-1}\br{t\mu_i-x},\]
$\Sigma_{i,t} = (1-t)^2 I_d + t^2\Sigma_{i}$ and 
\[w_{i, t}(x) = \frac{\pi_i \exp\br{\frac{-\norm{x-t\mu_i}^2_{\Sigma_i}}{2}}}{\sum_j\pi_j\exp\br{\frac{-\norm{x-t\mu_j}^2_{\Sigma_i}}{2}}}.\]

Therefore, we have

\[
v_t(x) = \sum_i w_{i,t}(x) \left(I_d - (1-t) \Sigma_{i,t}^{-1}\right)\frac{x}{t} + (1-t) \sum_{i} w_{i,t}(x) \Sigma_{i,t}^{-1} \mu_i
\]

Note that, if $\lambda$ is an eigenvalue of $\Sigma_{i}$, then the corresponding eigenvalue of $\frac{1}{t}(I_d - (1-t)\Sigma_{i,t}^{-1})$ is $\frac{t^2(1+\lambda) -1}{(1-t)^2 + t\lambda^2} \le (1+ \lambda^{-1})$. Therefore, $\norm{\frac{1}{t}(I_d - (1-t)\Sigma_{i,t}^{-1})}_{op} \le 1 + \norm{\Sigma_i^{-1}}_{op}=: A_i$. Similar argument shows that $\norm{\Sigma_{i,t}^{-1}}_{op} \le A_i$. Therefore, we have
\[
\innerprod{x, v_t(x)} \le \underbrace{(\max_i A_i)}_{A} \norm{x}_2^2 +  \underbrace{(\max_i A_i \norm{\mu_i}_2)}_{B} \norm{x}_2.
\]
Therefore, Assumption \ref{assumption : osgood} is satisfied with $h(u) = A u + B \sqrt{u}$ which is strictly monotonic function and $\int_{u_0}^\infty (Au + B \sqrt{u})^{-1} \; du = \infty$ for all $u_0>0$. Moreover, we have $\bbE \norm{X_1}_2 < \infty$. Therefore, by Theorem \ref{thm: gaussian straight flow} we conclude that the solution to the ODE \eqref{eq: modified ode} is unique.

\section{{Auxiliary results}}
    
\begin{lemma}
    \label{lemma: lower bund on matrix}
    Let $A \in \bbR^{d\times d}$ be an invertible matrix. Define $q(A):= \min_{u : \norm{u}_2 = 1} u^\top A u$. Then the following inequality is true provided $q(A^{-1})\ge 0$:
    \[
    q(A) \ge q(A^{-1}) q(A^\top A).
    \]
    
\end{lemma}
\begin{proof}
    Let $u$ be a unit vector such that $u^\top A u  = q (A)$. As $A$ is invertible, there exists a $v \in \bbR^d$ such that $u =  A^{-1}v$. Note that we have $\norm{v}_2^2 = u^\top A^\top A u \ge q (A^\top A)\ge 0$.  Then, we have
    \begin{align*}
        q(A) & = u^\top A u \\
        & = v^\top (A^{-1})^\top v\\
        & = v^\top (A^{-1}) v\\
        & \ge q(A^{-1}) \norm{v}_2^2 \ge q(A^{-1}) q (A^\top A).
    \end{align*}
\end{proof}

\begin{lemma}
\label{lemma: drift-score}
Let $(X_0, X_1) \sim N(0, I_d) \otimes \rho_1$. 
Let the density of $X_t = tX + (1-t)Z$ be $p_t$, and the score to be $s_t(x) = \nabla \log p_t(x)$. Then, we have
\[
v_t(x) = \frac{x}{t} + \br{\frac{1-t}{t}} s_t(x).
\]
\end{lemma}

\begin{proof}
First, note that due to Tweedie's formula \citep{robbins1992empirical} we have $\bbE{tX \mid X_t = x} = x + (1-t)^2 s_t(x)$.
 Using this, we have
\begin{align}
    v_t(x) &= \bbE[X-Z \mid X_t = x] \nonumber\\
    & = \bbE[\frac{X - X_t}{1-t} \mid X_t = x] \nonumber\\\
    & = \frac{x+(1-t)^2s_t(x)}{t(1-t)} - \frac{x}{(1-t)} \quad (\text{applying Tweedie's formula} ) \nonumber\\\ 
    & = \frac{x}{t} + \br{\frac{1-t}{t}} s_t(x) .
\end{align}
\end{proof}

\begin{lemma}
\label{lemma: drift_GMM_to_GMM}
Let $X_0 \sim \frac{1}{K_0}\sum_{i=1}^{K_0}  N(\mu_{0i}, \sigma^2I)$, and $X_1 \sim \frac{1}{K_1}\sum_{i=1}^{K_1}  N(\mu_{1i}, \sigma^2I)$ be independent, and define $X_t = tX_1 + (1-t)X_0$. Then, we have 
\[
v_t(x) = \frac{x}{t} + \frac{(1-t)\sigma^2}{t}\br{\frac{1}{K_0}\sum_{i=1}^{K_0} \frac{p_{t}^{(i)}(x)}{ p_t(x)} \br{s_t^{(i)}(x) - \frac{\mu_{0i}}{1-t}}},
\]
where $p_t^{(i)}(x)  = \frac{1}{K_1} \sum_{j=1}^{K_1} N(\underbrace{t\mu_{1j} + (1-t)\mu_{0i}}_{\mu_{tj}^{(i)}}, \sigma_t^2)$, $\sigma_t^2 = (t^2+(1-t)^2)\sigma^2$.
$$s_t^{(i)}(x) = \nabla_x \log p_t^{(i)}(x) = \frac{1}{\sigma^2_t}\br{\sum_{j=1}^{K_1} w_j^{(i)}(x) \mu^{(i)}_{tj} - x},$$
and 
\[w_j^{(i)}(x) = \frac{\exp\br{\frac{-\norm{x-\mu^{(i)}_{tj}}^2}{2\sigma_t^2}}}{\sum_j\exp\br{\frac{-\norm{x-\mu^{(i)}_{tj}}^2}{2\sigma_t^2}}}\]

\end{lemma}

\begin{proof}
\begin{align*}
    v_t(x) & = \E{X_1 - X_0 \mid X_t = x}\\
    & = \E{\frac{X_1 - X_t}{1-t} \mid X_t = x}\\
    & = \frac{1}{t(1-t)}\br{\E{tX_1 \mid X_t = x} - tx}\\
    & = \frac{1}{t(1-t)}\br{\frac{1}{K_0}\sum_{i=1}^{K_0} \frac{p_{t}^{(i)}(x)}{ p_t(x)} \E{tX_1 \mid X_t^{(i)} = x}- tx}\\
    & = \frac{1}{t(1-t)}\br{\frac{1}{K_0}\sum_{i=1}^{K_0} \frac{p_{t}^{(i)}(x)}{ p_t(x)} \br{x - (1-t)\mu_{0i}+ \tilde{\sigma}_t^2 s_t^{(i)}(x)}- tx}, \quad \where \tilde\sigma_t^2 = (1-t)^2\sigma^2\\
    & = \frac{x}{t} + \frac{(1-t)\sigma^2}{t}\br{\frac{1}{K_0}\sum_{i=1}^{K_0} \frac{p_{t}^{(i)}(x)}{ p_t(x)} \br{s_t^{(i)}(x) - \frac{\mu_{0i}}{1-t}}} 
\end{align*}
\end{proof}

\begin{lemma}\label{lem: drift-gmm}
Let $X_0 \sim \cN\br{0, I_d}$ and $X_1 \sim \sum_i \pi_i N\br{\mu_i, \sigma^2_iI}$ be independent. Let $X_t = tX_1 + (1-t)X_0$, with density $p_t$. Then, using Lemma \ref{lemma: drift-score}, we have
\begin{align} \label{eq: drift-gmm-true}
    v_t(x) = \frac{x}{t} + \frac{1-t}{t}s_t(x)
\end{align}
where, $s_t(x) = \nabla_x \log p_t(x)$ is given by

\[s_t(x) = \sum_i w_{i, t}(x) \br{\frac{t\mu_i-x}{\sigma^2_{i,t}}},\]
$\sigma^2_{i,t} = (1-t)^2 + t^2\sigma^2_i$ and 

\[w_{i, t}(x) = \frac{\pi_i \exp\br{\frac{-\norm{x-t\mu_i}^2}{2\sigma_{i, t}^2}}}{\sum_j\pi_j\exp\br{\frac{-\norm{x-t\mu_j}^2}{2\sigma_{i, t}^2}}}\]
\end{lemma}

\begin{proof}
    The result directly follows from Lemma \ref{lemma: drift_GMM_to_GMM} with $K_0 = 1$.
\end{proof}

\begin{lemma}
\label{lemma: GMM general covariance v}
Let $X_0 \sim N\br{0, I}$ and $X_1 \sim \sum_i \pi_i N\br{\mu_i, \Sigma_i}$ be independent random variables. Let $X_t = tX_1 + (1-t)X_0$ with density $p_t$. Then, we have
\begin{align} 
    v_t(x) = \frac{x}{t} + \frac{1-t}{t}s_t(x),
\end{align}
where, $s_t(x) = \nabla_x \log p_t(x)$ is given by
\[s_t(x) = \sum_i w_{i, t}(x) \Sigma_{i, t}^{-1}\br{t\mu_i-x},\]
$\Sigma_{i,t} = (1-t)^2 I_d + t^2\Sigma_i$ and 
\[w_{i, t}(x) = \frac{\frac{\pi_i}{\sqrt{\det(\Sigma_{i, t})}} \exp\br{\frac{-(x-t\mu_i)^\top {\Sigma_{i,t}^{-1}}(x - t\mu_i)}{2}}}{\sum_j \frac{\pi_j}{\sqrt{\det(\Sigma_{j, t})}}\exp\br{\frac{-(x-t\mu_j)^\top {\Sigma_{j,t}^{-1}}(x - t\mu_j)}{2}}}.\]
\end{lemma}

\paragraph{Note:}
One can also evaluate the exact derivative the drift $v_t$ in the above case.
For notational brevity, we define 
$\delta_{i,t} := \Sigma_{i,t}^{-1}(t\mu_i-z_t)$. Then, we have

\begin{equation}
\label{eq: nabla_v}
\begin{aligned}
\nabla_{z_t} v(z_t,t) &= \left\{\sum_i \frac{1}{t}\left\{ I_d - (1-t) \Sigma_{i,t}^{-1}\right\}. w_{i,t}(z_t)\right\} \\
&\quad + {\left(\frac{1-t}{t}\right). \left\{
\sum_{i <j} w_{i,t}(z_t) w_{j,t}(z_t) (\delta_{i,t} - \delta_{j,t}) (\delta_{i,t} - \delta_{j,t})^\top
 \right\}}
\end{aligned}
\end{equation}

\begin{lemma}
\label{lemma: mindless integration}
For $a,b>0$, define
\[
I(a,b)=\int_{0}^{1} s(1-s)\,
\frac{\sqrt{(1-s)^2+a^{2}s^{2}}}{\bigl((1-s)^2+b^{2}s^{2}\bigr)^{5/2}}\,ds.
\]
Then
\[
I(a,b)=\frac{a^{2}+ab+b^{2}}{3\,b^{3}(a+b)}.
\]
\end{lemma}

\begin{proof}
Begin with the substitution
\[
u=\frac{s}{1-s},\qquad s=\frac{u}{1+u},\qquad ds=\frac{du}{(1+u)^{2}}.
\]
Direct algebra shows that the integral becomes
\[
I(a,b)=\int_{0}^{\infty}
u\,\frac{\sqrt{1+a^{2}u^{2}}}{(1+b^{2}u^{2})^{5/2}}\,du.
\]
Next let $v=u^{2}$, so $u\,du=\frac{1}{2}dv$.  Then
\[
I(a,b)=\frac12 \int_{0}^{\infty}
\frac{\sqrt{1+a^{2}v}}{(1+b^{2}v)^{5/2}}\,dv.
\]
Now scale by $y=b^{2}v$, so $dv=\tfrac{1}{b^{2}}dy$ and
\[
I(a,b)=\frac{1}{2b^{2}} \int_{0}^{\infty}
\frac{\sqrt{\,1+\frac{a^{2}}{b^{2}}y\,}}{(1+y)^{5/2}}\,dy.
\]
Set
\[
r=\frac{a^{2}}{b^{2}}.
\]
To convert the improper integral to a bounded interval, use
\[
y=\frac{t}{1-t},\qquad dy=\frac{dt}{(1-t)^{2}},\qquad t\in[0,1].
\]
A straightforward simplification yields
\[
\int_{0}^{\infty} \frac{\sqrt{1+r y}}{(1+y)^{5/2}}\,dy
= \int_{0}^{1} \sqrt{\,1+(r-1)t\,}\,dt .
\]

If $r\ne 1$, the elementary antiderivative gives
\[
\int_{0}^{1}\sqrt{1+(r-1)t}\,dt
= \frac{2}{3}\,\frac{r^{3/2}-1}{r-1}.
\]
Therefore,
\[
I(a,b)
= \frac{1}{2b^{2}}\cdot \frac{2}{3}\,\frac{r^{3/2}-1}{r-1}
= \frac{1}{3b^{2}} \frac{r^{3/2}-1}{r-1}.
\]

Substitute back $r=\frac{a^{2}}{b^{2}}$:
\[
r^{3/2}=\frac{a^{3}}{b^{3}},
\qquad
r-1=\frac{a^{2}-b^{2}}{b^{2}},
\]
so
\[
I(a,b)
= \frac{1}{3b^{2}}
\cdot
\frac{\frac{a^{3}}{b^{3}} - 1}{\frac{a^{2}-b^{2}}{b^{2}}}
=
\frac{a^{3}-b^{3}}{3b^{3}\,(a^{2}-b^{2})}.
\]
Factor numerator and denominator:
\[
a^{3}-b^{3}=(a-b)(a^{2}+ab+b^{2}),\qquad
a^{2}-b^{2}=(a-b)(a+b).
\]
Thus
\[
I(a,b)=
\frac{a^{2}+ab+b^{2}}{3\,b^{3}(a+b)}.
\]

\end{proof}

\end{document}